\documentclass[accepted]{uai2023} 

\usepackage[american]{babel}

\usepackage{natbib} 
\usepackage{mathtools} 
\usepackage{booktabs} 
\usepackage{tikz} 

\usepackage{verbatim}

\usepackage{amsthm}
\usepackage{amssymb}

\usepackage{cleveref}
\usepackage{thm-restate}
\usepackage{mathtools}
\usepackage{bbm}

\usepackage{csquotes}
\usepackage{nicefrac}

\usepackage{comment}

\DeclareMathOperator*{\argmax}{arg\,max}
\DeclareMathOperator*{\argmin}{arg\,min}

\newtheorem{lemma}{Lemma}
\newtheorem{theorem}{Theorem}
\newtheorem{proposition}{Proposition}
\newtheorem{corollary}{Corollary}

\usepackage{float}

\theoremstyle{definition}
\newtheorem{definition}{Definition}
\newtheorem{example}{Example}

\newcommand{\Rbar}{\overline{\mathbb{R}}}
\newcommand{\E}{\mathbb{E}}

\newcommand{\Score}{S}
\newcommand{\Id}{I}
\newcommand{\p}{{\bm{p}}}
\newcommand{\q}{\bm{q}}
\newcommand{\Pvar}{\bm{P}}
\newcommand{\Y}{Y}
\newcommand{\y}{\bm{y}}
\newcommand{\Regret}{\mathrm{Regret}}
\newcommand{\x}{\bm{x}}

\newcommand{\Tv}{\bm{v}}
\newcommand{\Tw}{\bm{w}}

\newcommand{\op}{\mathrm{op}}

\newcommand{\D}{\mathcal{D}}
\newcommand{\N}{\mathcal{N}}
\newcommand{\Pset}{\Delta(\N)}
\newcommand{\TPset}{\mathcal{T}}
\newcommand{\interior}[1]{\mathrm{int}(#1)}

\usepackage{xcolor}

\usepackage{ifthen}
\newboolean{commentsactivated}
\setboolean{commentsactivated}{false}
\newcommand{\vc}[1]{\ifthenelse{\boolean{commentsactivated}}{{\color{blue} {\em VC: #1 }}}{}}
\newcommand{\co}[1]{\ifthenelse{\boolean{commentsactivated}}{{\color{red} {\em CO: #1 }}}{}}
\newcommand{\jt}[1]{\ifthenelse{\boolean{commentsactivated}}{{\color{olive} {\em JT: #1 }}}{}}
\newcommand{\ec}[1]{\ifthenelse{\boolean{commentsactivated}}{{\color{teal} {\em EC: #1 }}}{}}
\newcommand{\RH}[1]
{\ifthenelse{\boolean{commentsactivated}}{{\color{violet} {\em RH: #1 }}}{}}

\usepackage{bm}
\newcommand{\norm}[1]{\left\Vert#1\right\Vert}
\renewcommand{\vec}[1]{\bm{#1}} 
\usepackage{tikz}
\usepackage{tkz-euclide}
\usetikzlibrary{decorations.markings, calc}
\newcommand{\defeq}{\vcentcolon=}

\usepackage{enumitem}

\usepackage{subcaption}

\title{Incentivizing honest performative predictions with proper scoring rules}

\author[1]{\href{mailto:<oesterheld@cmu.edu>?Subject=Incentivizing honest performative predictions with proper scoring rules}{Caspar~Oesterheld}{}\thanks{Equal contribution}}
\author[2]{\href{mailto:<jtreutlein@berkeley.edu>?Subject=Incentivizing honest performative predictions with proper scoring rules}{Johannes~Treutlein}{}\footnote[1]{}}
\author[3]{Emery~Cooper}
\author[4]{Rubi~Hudson}
\affil[1]{%
Carnegie Mellon University
}
\affil[2]{%
University of California, Berkeley
}
  \affil[3]{%
Center on Long-Term Risk
  }
\affil[4]{%
University of Toronto
  }

\begin{document}
\maketitle

\begin{abstract}
Proper scoring rules incentivize experts to accurately report beliefs, assuming predictions cannot influence outcomes. We relax this assumption and investigate incentives when predictions are \emph{performative}, i.e., when they can influence the outcome of the prediction, such as when making public predictions about the stock market. We say a prediction is a \emph{fixed point} if it accurately reflects the expert’s beliefs after that prediction has been made. We show that in this setting, reports maximizing expected score generally do not reflect an expert’s beliefs, and we give bounds on the inaccuracy of such reports. We show that, for binary predictions, if the influence of the expert's prediction on outcomes is bounded, it is possible to define scoring rules under which optimal reports are arbitrarily close to fixed points. However, this is impossible for predictions over more than two outcomes.
We also perform numerical simulations in a toy setting, showing that our bounds are tight in some situations and that prediction error is often substantial (greater than 5-10\%). Lastly, we discuss alternative notions of optimality, including performative stability, and show that they incentivize reporting fixed points.
\end{abstract}



\section{Introduction}

As AI capabilities increase, this raises concern for safety, including how to scalably control AI systems with superhuman capabilities \citep{
Russell2019,ngo2022alignment}. 
One proposed design for safety is \emph{oracle AI} [\citealp{armstrong2012thinking}; \citealp{armstrong2013risks}; \citealp{bostrom2014superintelligence}, Ch.\ 10]. An oracle AI 
makes predictions or forecasts about the world, but does not autonomously pursue goals. 
It could thus be safer while still being useful for many applications.

A proper scoring rule assigns scores to forecasts in a way that incentivizes honest reporting of beliefs [\citealp{Brier1950}; \citealp{Good1952}, Section 8; \citealp{McCarthy1956}; \citealp{Savage1971}; \citealp{gneiting2007strictly}]. Proper scoring rules have been used to incentivize honest reports from experts \citep{carvalho2016overview}. They could thus be used as an objective 
for oracle AIs. However, prior work assumes that predictions themselves do not influence the events they are trying to predict. 
In reality, predictions may be \emph{performative} \citep{perdomo2020performative,armstrong2017good}, meaning that they can influence the distribution of outcomes. For example, an AI predicting stock market prices might be able to influence whether people buy or sell stocks, and thus influence whether its predictions come true or not. This makes it important to investigate incentives and honesty of predictions when predictions are performative.

In this paper, we analyze the case of an AI model or human, henceforth called expert, making a probabilistic forecast over a finite set of possibilities to maximize a proper scoring rule. We say that a prediction is performatively optimal if it maximizes expected score, and we define a prediction as a \textit{fixed point} or self-fulfilling if it is equal to the expert's beliefs, conditional on the expert having made that prediction. We investigate to what extent honest predictions, i.e., fixed points, are incentivized in this setting.\footnote{We assume that the AI model can be ascribed explicit beliefs, so that its reports can be characterized as honest if they reflect the model's beliefs.} 
All else equal, honest predictions are preferable since, assuming a sufficiently capable expert, they provide us with more accurate information. However, if an expert has incentives other than to predict honestly---e.g., to bring about fixed points with lower entropy---this is undesirable even if the expert otherwise makes approximately accurate predictions.


The setting in which a model's predictions can influence the predicted distribution has been discussed as \emph{performative prediction} \citep{perdomo2020performative} in the machine learning literature. However, performative prediction focuses on classification or regression tasks with arbitrary model classes and loss functions rather than probabilistic predictions incentivized by proper scoring rules. The literature is motivated by minimization of a given loss function, whereas we take a mechanism design perspective, asking which scoring rules incentivize honest predictions. Focusing on a special case and taking a different perspective will lead to original results that are unique to our setting.

\textbf{Contributions. }
In \Cref{problem-setting}, we adapt the performative prediction formalism to probabilistic predictions or forecasts. We allow for an arbitrary function \(f\) describing the relationship between the expert's predictions and distributions over predicted outcomes caused by these predictions. 

In \Cref{incentives-non-fixed-points}, we show that for any strictly proper scoring rule, \textbf{there exist functions \(f\) from predictions to beliefs such that performatively optimal reports are not fixed points}, even if one exists and is unique. Moreover, we show that under reasonable distributions over such functions, \textbf{optimal reports are almost never fixed points}. This strengthens analogous results from the performative prediction literature.

In \Cref{bounds-on-deviation}, we then \textbf{provide upper bounds} for the inaccuracy of reported beliefs, and for the distance of predictions from fixed points.

In \Cref{sec:approximate-fixed-point-prediction}, we use the bounds to develop \textbf{scoring rules that make the bounds arbitrarily small for binary predictions}. \co{Currently this doesn't mention the result about how these scoring rules must be exponential. That's because this result seems pretty subtle and a bit hard to explain here.}
We also show that \textbf{when reporting a prediction over more than two outcomes, the bounds cannot be made arbitrarily small}.


In \Cref{numerical-simulations}, we perform \textbf{numerical simulations using the quadratic scoring rule}, to show how the inaccuracy of predictions and the distance of predictions from fixed points depend on the expert's influence on the world via its prediction.
The results show that our bounds are tight in some cases. They also show that substantially inaccurate reports (i.e., with errors greater than \(5-10\)\%) are common in our toy setting.

In \Cref{stop-gradients}, we discuss alternatives to performative optimality that do not set incentives other than honest predictions. We show that \emph{performatively stable} \citep{perdomo2020performative} predictions are fixed points. We then consider repeated risk minimization, repeated gradient descent, no-regret learning and prediction markets, and show that all of these settings lead to predictions that are fixed points or close to fixed points.

Finally, in \Cref{related-work}, we elaborate on related work, and in \Cref{conclusion}, we conclude and outline avenues for future work.

Proofs are in corresponding sections in \Cref{appendix:proofs}.

\section{Background}

\textbf{Proper scoring rules. }
Proper scoring rules are used to incentivize an expert to report probabilistic beliefs honestly. Consider a prediction given by a probability distribution \(\p\in\Delta(\N)\) over a set \(\N:=\{1,\dotsc,n\}\) of \(n\in\mathbb{N}\) disjoint and exhaustive outcomes. We identify each distribution \(\p\in\Delta(\N)\) with a vector \(\p\in[0,1]^n\) and write \(p_i\) for the probability of event \(i\in\N\) under distribution \(\p\). A \emph{scoring rule} is a function \(S\colon \Delta(\N)\times \N \rightarrow \Rbar\), where \(\Rbar:=[-\infty,\infty]\) is the extended real line. Given prediction \(\p\in\Delta(\N)\) and outcome \(i\in\N\), the expert receives the score \(S(\p,i).\)
We write $\Score(\p,\q):=\E_{i\sim \q}[S(\p,i)]$
for the expert's expected score, given that outcome \(i\) follows distribution \(\q\in\Delta(\N).\)
\begin{definition}A scoring rule \(S\) is called \emph{proper} if $\Score(\q,\q)\geq \Score (\p,\q)$
for all \(\p,\q\in\Delta(\N).\) It is called \emph{strictly proper} if this inequality is strict whenever \(\p\neq \q.\)
\end{definition}

\begin{example}[Logarithmic scoring rule]The logarithmic scoring rule is defined as
\(S(\p,i):=\log p_i\)
and \(\Score(\p,\q)=\sum_{i=1}^nq_i\log p_i\). This is also the negative of the 
cross-entropy loss employed in training, for example, current large language models \citep{brown2020language}. It is strictly proper.
\end{example}

\begin{example}[Quadratic scoring rule]
Another strictly proper scoring rule is the quadratic score, defined as
\(S(\p,i):=2p_i-\Vert \p \Vert_2^2\)
with \(\Score(\p,\q)=2 \p^\top \q-\Vert \p\Vert_2^2\). This is an affine transformation of the Brier score, making them equivalent scoring rules.
\end{example}

\citet[Theorem~1]{gneiting2007strictly} provide a characterization of proper scoring rules, which will be helpful for stating and proving many of our results.

First, given a convex function \(G\colon\Pset\rightarrow \Rbar\), a subgradient is a function \(g\colon \Pset \rightarrow \Rbar^n\) such that
for any \(\p,\q\in \Pset,\) we have \(G(\q)\geq G(\p) + g(\p)^\top(\q-\p)\). In general, this function may not be unique. Throughout this paper we assume that whenever the subgradients are finite, they are normalized to lie in the \emph{tangent space} of \(\Pset\), i.e., $g(\p)\in \TPset \defeq \{\x\in\mathbb{R}^n\mid \sum_ix_i=0\}$. This can be assumed since if $g(\p)$ is a subgradient of $G$ at point $\p$, so is $(g_i(\p)-\frac{1}{n}\sum_j g_j(\p))_i$.

\begin{theorem}[\citealp{gneiting2007strictly}]
\label{theorem:gneiting-raftery}
A scoring rule \(S\) is (strictly) proper, if and only if there exists a (strictly) convex function \(G\colon \Pset\rightarrow\Rbar\) with a \emph{subgradient} \(g\colon \Pset\rightarrow \Rbar^n\) such that $\Score(\p,\q)=G(\p)+g(\p)^\top(\q-\p)$ for all \(\p,\q\in \Pset\).
\end{theorem}


\textbf{Differentiable scoring functions. }
If \(G\) is differentiable at some point \(\p\),  then the subgradient \(g(\p)\) is just the gradient of \(G\), \(g(\p)=\nabla  G(\p)\). As before we let \(\nabla G(\p)\) be an element of the tangent space $\TPset$.
For any \(\Tv\in\TPset\), \(g(\p)^\top\Tv\) then gives the directional derivative of \(G\) at point \(\p\) in the direction \(\Tv\). Note that since $G$ is only defined on the simplex $\Pset$, the partial derivatives are not well-defined.\footnote{For example, in the case of three outcomes, the partial derivative at \((0.3,0.4,0.4)\) w.r.t.\ the first entry is the limit $\lim_{\epsilon\rightarrow 0}(G(0.3+\epsilon,0.4,0.4)-G(0.3,0.4,0.4))/\epsilon$. But $G(0.3+\epsilon,0.4,0.4)$ is not (necessarily)  defined for positive (or negative) $\epsilon$.}

Given \(G,g\) as in the Gneiting and Raftery characterization, we write \(Dg(\p)\in\mathbb{R}^{n,n}\) for the Jacobian matrix of \(g\), if it exists (i.e., this is the Hessian of \(G\)). Note that because $g$ is only defined on $\Pset$, the matrix representation of \(Dg(\p)\) in \(\mathbb{R}^{n,n}\) is not unique. Generally it does not matter which representation of $Dg(\p)$ we use. Importantly, for all $\vec{v}\in\TPset$, $Dg(\p)\vec{v}$ will always be unique and (because we assume that $g(\p)$ is in the tangent space) in $\mathcal{T}$.

\begin{figure}
\includegraphics[width=1\columnwidth]{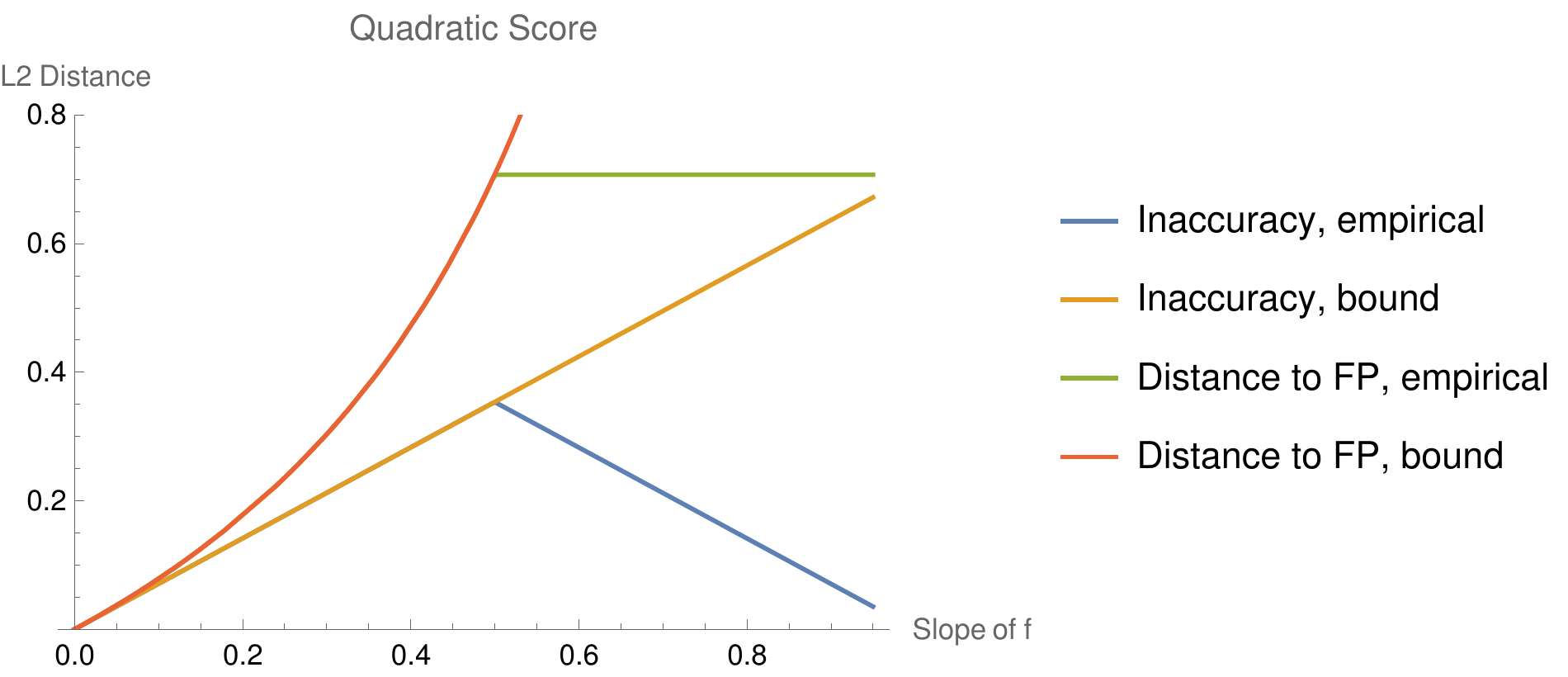}
\caption{Maximal inaccuracy and maximal distance to fixed point (FP) of optimal predictions, depending on the slope of \(f\), according to our simulation and our theoretical bound.
}
\label{fig:max-l2-distance-brier-two-outcome}
\end{figure}

\section{Problem setting}
\label{problem-setting}


In this paper, we take the stance of a principal trying to elicit honest predictions from an expert (human or AI system). We assume that the expert reports a prediction \(\p\) to maximize the expected score given by a proper scoring rule, \(\Score(\p,\q)\). 

Importantly, we assume that the expert's beliefs over outcomes, \(\q\), can themselves vary given different predictions \(\p\), because the expert may believe that its predictions affect the probability of outcomes. 
To model this, we assume that there is a function \({f\colon \Pset\rightarrow \Pset}\) such that beliefs are given by \(\q=f(\p)\).\footnote{Note that any other factor influencing the expert's belief \(\q\) can be incorporated into $f$ by marginalizing. For example, assume $\q$ is a function $\q=g(\p; X)$ where $\p$ is the expert’s prediction and $X$ is some environmental factor the expert is uncertain about. Then we can let $f(\p) := \mathbb{E}_X \left[ g(\p,X) \right]$.
}
We assume \(f\) is known to the expert, but not to the principal.

In the case of an AI system, \(f(\p)\) could also be seen as a ground distribution from which we sample to \emph{train} a model (see \Cref{appendix:online-learning}). In that case, the objective is to design a training procedure that sets the right incentives. However, in most of the following, we assume \(f(\p)\) are the subjective beliefs of a highly capable and knowledgeable expert. 

We say that a prediction \(\p\) is \emph{performatively optimal} \citep{perdomo2020performative} if \(\p\in\argmax_{\Pset} \Score(\p,f(\p))\). In the following, we will not assume convexity of this objective. Our bounds will depend on differentiability of \(S\) and \(f\).

A point \(\p\) is a \emph{fixed point} of \(f\) if \(f(\p)=\p\). By Brouwer's fixed point theorem, if \(f\) is continuous, a fixed point \(\p\in\Pset\) always exists. Moreover, if \(f\) is Lipschitz continuous with constant \(L_f<1\), then by Banach's fixed point theorem, the fixed point is unique.


\begin{example}[Bank Run]
A newspaper's AI predicts whether a certain bank will suffer a bank run or not. Readers use this information when deciding whether to withdraw their money. Specifically, imagine that the probability of a bank run as a function of the AI expert's prediction $\mathbf{p}=(p_1,p_2)\in \Delta(\{1,2 \})$ is given by the (monotonic) function $f\colon \Delta(\{1,2 \}) \rightarrow \Delta(\{1,2 \})$ whose entries are defined by $f_1(\mathbf{p}) = p_1-3(p_1-\nicefrac{1}{10})(p_1-\nicefrac{3}{5})(p_1-\nicefrac{9}{10})/2$ and $f_2(\mathbf{p})=1-f_1(\mathbf{p})$ for all $\mathbf{p}$. Then $f$ has fixed points at $\mathbf{p}=(\nicefrac{1}{10},\nicefrac{9}{10})$, $\mathbf{p}=(\nicefrac{3}{5},\nicefrac{2}{5})$, and $\mathbf{p}=(\nicefrac{9}{10},\nicefrac{1}{10})$.
\end{example}


We focus on fixed points (or approximate fixed points) as a standard of honesty. To see why one may prefer reports that are fixed points, consider a case in which there are no strong guarantees (upper bounds) on \(\Vert \p-f(\p)\Vert \). Then the actual probability of an event, \(f_i(\p)\), could be much higher or lower than the reported probability \(p_i\). This would prevent one from drawing any useful conclusions from the report. However, if \(\p=f(\p)\) or \(\Vert \p-f(\p)\Vert\) is small, then one can rely on the prediction \(\p\) to guide decisions.

That being said, fixed points are not all one might care about, especially when it comes to potential superhuman oracle AIs. Ideally, we would want such systems to not think about how to influence the world at all \citep{armstrong2017good}. Alternatively, they should choose good fixed points over bad ones, hoping that such fixed points exist (we discuss preferences between different fixed points in \Cref{preferences-between-fps}). Regardless, it is still important to understand whether and when fixed points are incentivized. For instance, if a model reports fixed points, one could try to use it only in situations in which a unique desirable fixed point exists.

\textbf{Relation to performative prediction. } As noted in the introduction, our setting is a special case of performative prediction \citep{perdomo2020performative}. In performative prediction, the goal is to find a model parameter that minimizes empirical risk for a classification or regression task, assuming that the choice of parameter can influence the data distribution. The loss-minimizing parameter when taking into account this influence is called performatively optimal. The analogue to fixed points in performative prediction are \emph{performatively stable} predictions. 

We indicate below when our results are analogous to results in the performative prediction setting. However, most of our results are unique to our setting. We take the perspective of a mechanism designer instead of taking a loss function as given. Moreover, we focus on fixed points instead of performative optima. In particular, we bound the quantity \(\Vert \p-f(\p)\Vert\) corresponding to the inaccuracy of predictions, which does not have a direct analogue in performative prediction. We give a more detailed comparison in \Cref{related-work}. 


\textbf{Additional notation. } 
We use \(\mathbf{1}\) to denote the vector \((1,\dotsc,1)^\top\in\mathbb{R}^n\) and \(\Id\) to denote the identity matrix. We define \(\interior{\Pset}:=\{\p\in\Pset\mid \forall i\colon 0<p_i<1\}\) and use \(\Vert\vec{x}\Vert:=\sqrt{\Vec{x}^\top \x}\) to denote the Euclidean norm on \(\mathbb{R}^n\).

\section{Incentives to predict non-fixed-points}
\label{incentives-non-fixed-points}

We begin by investigating whether an expert makes honest predictions, even in the presence of performativity. In performative prediction, it has been shown that performative optimality comes apart from performative stability (the analogous concept to a fixed point in our setting) \citep{perdomo2020performative,izzo2021learn}. However, one may ask whether this is always the case or whether, e.g., some scoring function would prevent this. 


We show that this is not the case: fixed points are in general not optimal. First, we show that for any strictly proper scoring rule there exist cases where a fixed point exists but the optimal prediction is not a fixed point. 
Afterwards, we show that when assuming differentiability and some reasonable distribution over \(f\), optimal predictions are almost surely not fixed points.

\begin{restatable}{proposition}{propone} \label{prop:non-fixed-point-optimal}
Let 
\(S\)
 be any strictly proper scoring rule. For any interior fixed point \(\p^*\in \interior{\Pset}\)  there exists a function 
\(f\)
 with Lipschitz constant 
\(L_f<1\)
 and a unique fixed point at 
\(\p^*\), such that there exists 
\(\p'\neq \p^*\) with 
\(\Score(\p', f(\p'))>\Score(\p^*,f(\p^*))\). That is, the unique fixed point of 
\(f\) is not performatively optimal.
\end{restatable}

\co{[LOW PRIORITY:] There are some interesting variants of this result. For example, I think for all but one $\p^*$, you can get the conclusion for arbitrarily small $L_f$. Whereas the current proof hinges on  $f\approx \mathrm{id}$.


\begin{proposition}
    Let $S$ be a strictly proper scoring rule and $L_f>0$. Then for almost all (measure $1$ of) points $\p^*\in \interior{\Pset}$ 
    there exists a function $f$ with Lipschitz constant $L_f$ s.t.\ $\p^*$ is the unique fixed point of $f$ and there exists $\p'\neq \p^*$ with $S(\p',f(\p'))>S(\p^*,f(\p^*))$. That is, the unique fixed point of $f$ is not performatively optimal.
\end{proposition}

\begin{proof}
    Let $S(\p,\q)=g(\p)(\q-\p)+G(\p)$.
    Consider specifically $p^*$ s.t.\
    \begin{itemize}
        \item $g(p^*) \neq 0$; and
        \item $g$ is locally Lipschitz continuous at $p^*$.
    \end{itemize}
    Since $G$ is convex, there is only one point on $\Pset$ s.t.\ doesn't satisfy the first point.
    TODO: why is the second true?
    I think we actually don't need the second thing necessarily. I think we could use Alexandrov's theorem, which states that $g$ is differentiable almost everywhere. 

    Now consider linear $f$ of the form $f(\p) = \p^* + \alpha (\p -\p^*)$ for $\alpha<1$. Note that $f$'s unique fixed point is $p^*$ and that $f$ has Lipschitz constant $\alpha$. 
    Now note that for any $p\neq p^*$
    \begin{eqnarray*}
        && S(p,f(p))-S(p^*,f(p^*))\\
        &=& g(p)(f(p)-p) + G(p) - G(p^*)\\
        &> & g(p)(f(p)-p) + g(p^*)(p-p^*)\\
        &=& g(p)(\alpha (p-p^*)) + g(p^*)(1-\alpha)(p^*-p),
    \end{eqnarray*}
    where the last line applies the subgradient inequality $G(p^*)+g(p^*)(p-p^*)<G(p)$.

    Now consider $p$ s.t.\ $p-p^*$ is in the direction of $g(p^*)$. Then $g(p)(\alpha (p-p^*))=\alpha \lVert g(p) \rVert \lVert p-p^* \rVert$. Further, because $G$ is continuously differentiable at $p^*$, there is a $K$ s.t.\ for small enough $\lVert p-p^* \rVert$ we have that $\lVert g(p)-g(p^*) \rVert\leq K \lVert p-p^* \rVert$. [TODO: actually this assumes that $g$ is locally Lipschitz continuous. Right?]

%
%
%
\end{proof}

}

Note that since the function 
\(f\) has Lipschitz constant strictly smaller than \(1\), it represents a world that \enquote{dampens} the influence of the prediction, leading to a unique fixed point by Banach's fixed point theorem. It is interesting that the expert still prefers to make a prediction that is not a fixed point.

The above result raises the question whether a situation where fixed points are suboptimal is a niche counterexample or whether it is common. We show that under some relatively mild assumptions, the optimal prediction is almost surely not a fixed point.
 The intuition behind this result is that if a prediction \(\p\) is an interior point and optimal, then \(\nabla_{\p}(\Score(\p,f(\p)))=0\). Using the Gneiting and Raftery characterization, we can show that this is a knife-edge case in which \(g(\p)^\top Df(\p)=0\). Given sufficiently continuous distributions, this happens with probability \(0\). The conditions on the stochastic field \(\{F(\p)\}_{\p\in\interior{\Pset}}\) ensure this continuity, i.e., that the distributions over \(f\) as well as \(Df(\p)\) do not assign positive probability to any single point or subspace, hence almost never sampling the knife edge case. The condition would hold, e.g., for a Gaussian process with smooth kernel and mean functions (see \Cref{ex:gaussian-process} in \Cref{appendix-proof-of-theorem-8}).

\begin{restatable}[]{theorem}{fprare}\label{prop:fixed-points-optimal-reports-are-rare}
Let \(\Score\) be a twice differentiable strictly proper scoring rule. Let \(\mathcal{F}:=\{F(\p)\}_{\p\in\interior{\Pset}}\) be a stochastic field with values in \(\Pset\) and let \(Y(\p,\Tv):=(\Pi_{n-1} F(\p), \Pi_{n-1}\partial_{\Tv}F(\p))\) for \(\p\in \interior{\Pset}\) and \(\Tv\in \TPset\cap S^{n-1}\). Assume that
    \begin{itemize}[nolistsep]
    \item the sample paths \(\p\rightsquigarrow F(\p)\) are twice continuously differentiable
    \item for each \(\p\in \interior{\Pset}\) and \(\Tv\in \TPset\cap S^{n-1}\), the random vector \(Y(\p,\Tv)\) has a joint density \(h_{Y(\p,\Tv)}\) and there exists a constant \(C\) such that $h_{Y(\p,\Tv)}\leq C$ for all \(\p\in\Pset,\Tv\in S^{n-1}\cap \TPset\). 
    \end{itemize}
    Then, almost surely, there is no point \(\p\in \interior{\Pset}\) such that \(\p\in\argmax_{\p'}\Score(\p',F(\p'))\) and \(F(\p)=\p\).
\end{restatable}

\begin{figure*}[t]
\centering
\begin{subfigure}[t]{.45\textwidth}
    \centering
    \includegraphics[width=\linewidth]{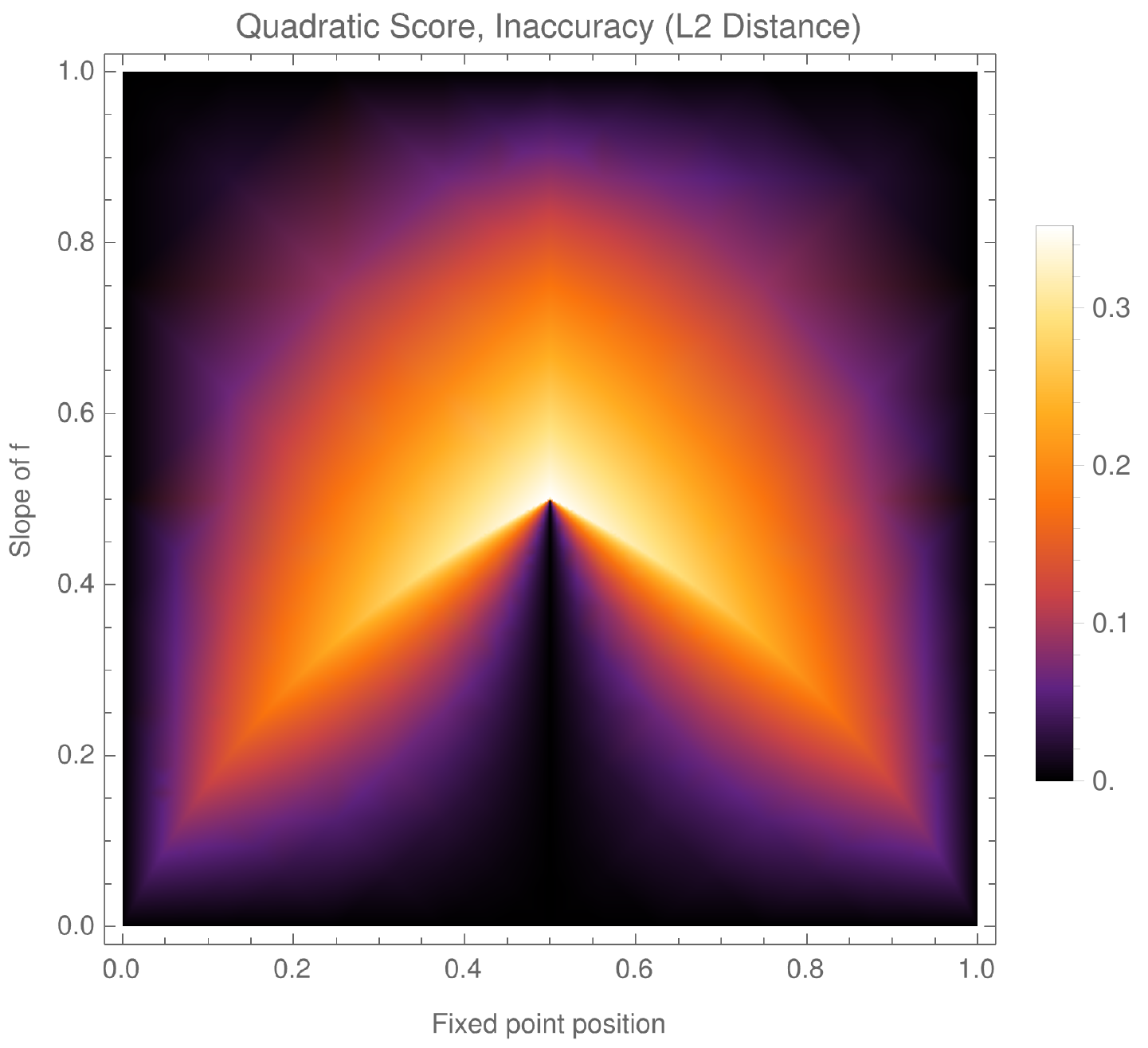}

\end{subfigure}%
\begin{subfigure}[t]{.45\textwidth}
    \centering
    \includegraphics[width=\linewidth]{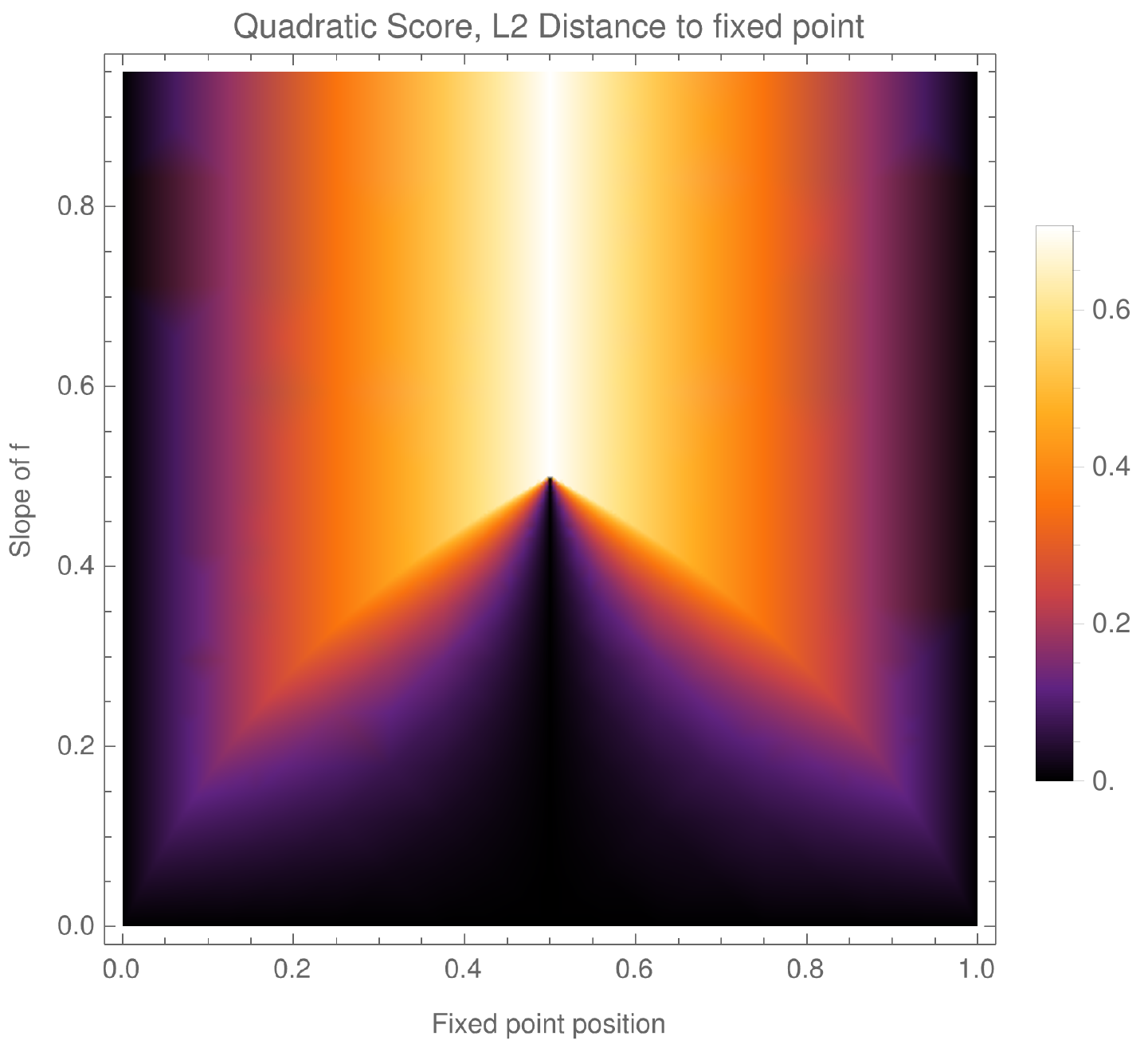}
\end{subfigure}
\caption{Heatmap of L2 distance of optimal prediction $\p$ to true probability distribution $f(\p)$ (left) and to the fixed point $\p^*$ (right), depending on fixed point position $p_1^*$ and $\alpha$ (slope of \(f\)), for the quadratic scoring rule.
}
\label{fig:density-plot-brier-l2-inaccuracy}
\label{fig:density-plot-brier-l2-disttofp}
\end{figure*}

\section{Bounds on the deviation from fixed points}
\label{bounds-on-deviation}

\co{LOW PRIORITY: Can we also get lower bounds with the same ideas? (Consider the derivative in the direction that moves $f$ in the direction of $g(p)$)?) (How would that relate to \Cref{thm:need-exponential-new}?) I guess they'd be less interesting, because they're restricted to differentiable scoring rules and probably they can be ``hacked'' (made very small without actually guaranteeing much).}

In the previous section, we have shown that performatively optimal predictions are generally not fixed points, i.e., they inaccurately represent the expert's beliefs. But \textit{how} inaccurate should we expect predictions to be, and what properties of $S$ and $f$ determine this inaccuracy?
Assuming differentiability of \(f\) and \(\Score\), this section provides upper bounds for the inaccuracy of optimal predictions \(\p\) (i.e., $\Vert \p-f(\p) \Vert$) and their distance from fixed points \(\p^*\) (i.e., \(\Vert \p-\p^*\Vert \)). Note that, while the latter has a direct analogue in the performative prediction literature \citep[][Theorem 4.3]{perdomo2020performative}, evaluating the inaccuracy of predictions only makes sense in our context where parameters are probability distributions.


For our bounds we will use the following notation. We use \(\Vert A\Vert_{\mathrm{op}}=\max_{\Tv\in \TPset}\frac{\Vert A\Tv\Vert}{\Vert \Tv\Vert}\) for the operator norm of $A$ on the tangent space. It is equal to $A$'s largest singular value when seen as an automorphism on the tangent space. We use $A\vert_{\TPset}\succeq \gamma$ to denote that \(\Tv^\top (A- \gamma\Id)\Tv\geq0\) for all \(\Tv\in\TPset\). If \(A\) is symmetric, this is equivalent to saying that the smallest eigenvalue of $A$ on the tangent space is at least $\gamma$. 
Further, note that if $g$ is a subderivative of $G$ and $\Vert g(\p)\Vert <L_G$ for all $\p\in\Pset$, then $L_G$ is a Lipschitz constant of $G$. Similarly, if $\Vert Df(\p)\Vert_{\mathrm{op}} \leq L_f$ for all $\p\in\Pset$, then $L_f$ is a Lipschitz constant of $f$.

\begin{restatable}{theorem}{inaccuracybound} \label{theorem:Caspar-approx-fix-point}
    Let \(S\) be a strictly proper scoring rule, and let \(G,g\) as in the Gneiting and Raftery characterization (\Cref{theorem:gneiting-raftery}). 
    Let \(\p\in\Pset\) and assume \(f,G,g\) are differentiable at \(\p\). Assume \(Dg(\p)|_{\TPset}\succeq\gamma_{\p}\) for some \(\gamma_\p>0\).
    Then whenever $\p$ is a performatively optimal report,
    \begin{equation*}
    \Vert \p - f(\p) \Vert \leq\frac{ \Vert  Df(\p)\Vert_{\mathrm{op}}\Vert g(\p)\Vert}{\gamma_{\p}}.\end{equation*}
    In particular, if $f$ has Lipschitz constant $L_f$, \(G\) has Lipschitz constant \(L_G\), and \(G\) is \(\gamma\)-strongly convex, then we have $\Vert \p - f(\p) \Vert \leq \frac{L_f L_G}{\gamma}$.
\end{restatable}

\co{LOW PRIORITY: Maybe say something about whether/when the bound is tight. I think the main bound is tight in the two-outcome case but not otherwise.}
%
%
%

In the case where \(f\) has Lipschitz constant \(L_f<1\), we can use the above results to derive a bound on how far the optimal report is from the (by Banach's fixed point theorem unique) fixed point.

\begin{restatable}{theorem}{bounddisttofp}\label{thm:distance-to-fp}
    Same assumptions as \Cref{theorem:Caspar-approx-fix-point}. Assume further that $f$ has Lipschitz constant $L_f<1$. Let $\p^*$ be the unique fixed point of $f$. Then for the performatively optimal report $\p$,
    \begin{equation*}
        \Vert \p-\p^*\Vert \leq \frac{\Vert g(\p)\Vert\Vert Df(\p) \Vert_{\mathrm{op}}}{(1-L_f)\gamma_\p} \leq \frac{L_fL_G}{(1-L_f)\gamma_\p}.
    \end{equation*}
\end{restatable}


Note that the assumption that $L_f<1$ ensures that $f$’s fixed point is unique by Banach’s fixed point theorem. Without $L_f<1$, no trivial bound holds, as we show in \Cref{prop:Lf1-no-non-trivial-bound} in \Cref{appendix:Lf1-no-non-trivial-bound}.

\co{[LOW PRIORITY:] I think it would be good to show in the appendix that even for the binary case we don't get any non-trivial guarantees for $L_f=1$, i.e., worst case distance to fixed point might be arbitrarily close to $1$.}

This bound is analogous to a bound in \citep[][Theorem~4.3]{perdomo2020performative}. Our bound differs in that we use Euclidean distance instead of Wasserstein distance to measure the sensitivity of \(f\) to the choice of report. Moreover, assuming a \(L_\ell\)-Lipschitz and \(\gamma\)-strictly convex loss function \(\ell\), their bound depends on the ratio \(\frac{L_\ell}{\gamma}\). We instead bound this distance against the ratio \(\frac{\Vert g(\p)\Vert}{\gamma_\p}\), 
which will allow us to minimize the bound in the two-outcome case by using exponential functions (\Cref{theorem:two-outcomes-arbitrarily-good-bounds}). This would not be possible when assuming \(\gamma\)-strict convexity, since there exist no functions that globally make the ratio \(\frac{L_\ell}{\gamma}\) arbitrarily small. \cite{perdomo2020performative} show that their bound can be made small by regularizing the loss function, but this would be undesirable in our setting, since regularized scoring rules would be improper and thus cease to incentivize honest reports even for constant \(f\).

\begin{example}[Bound for the quadratic scoring rule]
\label{example:bounds-for-brier}
    Consider the quadratic scoring rule $S(\p,i)=2p_i-\Vert\p\Vert_2$. Note that we can represent this in Gneiting and Raftery's characterization with \(G(\p)=\Vert \p\Vert^2\) and $g(\p)=2\p-\frac{2}{n}\mathbf{1}$. Thus, $Dg(\p)=2I$, where $I$ is the identity matrix. Hence \(Dg(\p)\succ 2\). Further, $\Vert g(\p)\Vert _2=2\Vert \p-\frac{1}{n}\mathbf{1}\Vert $.
    Thus, for $f$ with Lipschitz constant $L_f$, \Cref{theorem:Caspar-approx-fix-point} implies that for the optimal report $\p$ we have that $\Vert f(\p)-\p\Vert \leq L_f\Vert \p-\frac{1}{n}\mathbf{1}\Vert \leq L_f \sqrt{(n-1)/n}$. If $L_f<1$, then by \Cref{thm:distance-to-fp} we further have $\Vert \p -\p^*\Vert \leq \frac{L_f}{1-L_f}\Vert \p-\frac{1}{n}\mathbf{1}\Vert\leq \frac{L_f}{1-L_f}\sqrt{(n-1)/n}$. 
\end{example}

\section{Approximate fixed-point prediction with the right proper scoring rules?}
\label{sec:approximate-fixed-point-prediction}

The above results show that depending on the scoring rule we can obtain bounds on the accuracy of performatively optimal predictions. Can we make these bounds arbitrarily small by choosing an appropriate scoring rule, e.g., one that makes $\Vert g(\p) \Vert / \gamma_{\p}$ very small at each point? In this section, we show that the answer is yes in the two-outcome case and no in the general case.


\begin{restatable}{theorem}{exponentialthm}\label{theorem:two-outcomes-arbitrarily-good-bounds}
Consider the case of two outcomes, i.e., let $\mathcal{N}=\{1,2 \}$. Let $L_f\in \mathbb{R}$ and $\epsilon>0$. Then there exists a scoring rule $S$ s.t.\ under any $f$ with Lipschitz constant $L_f$, any optimal report $\p$ satisfies $\Vert \p-f(\p)\Vert \leq \epsilon$. If $L_f<1$, then there also exists a scoring rule that additionally ensures that under any $f$ with Lipschitz constant $L_f$, any optimal report satisfies $\Vert \p-\p^*\Vert \leq \epsilon$, where $\p^*$ is the (unique) fixed point of $f$.
\end{restatable}

Note that if there are multiple fixed points, then $S$ still induces preferences between---approximately---predicting these fixed points. In particular, because $S(\p,\p)$ is convex, the performatively optimal fixed point will either be the one that maximizes or the one that minimizes $p_1$ among the fixed points. This may be undesirable as the expert still has a strong incentive other than (though compatible with) honest prediction. We discuss this in more detail in \Cref{preferences-between-fps}.

Can arbitrarily good bounds be achieved with \textit{practical} proper scoring rules? Our proof of \Cref{theorem:two-outcomes-arbitrarily-good-bounds} uses exponential scoring rules with $g(\p)=(e^{L_fp_1/(\sqrt{2}\epsilon)},-e^{L_fp_1/(\sqrt{2}\epsilon})^\top$. For high $K$, this scoring rule seems impractical, because the stakes vary greatly over the interval. For example, $S((\nicefrac{2}{3}+\epsilon,\nicefrac{1}{3}-\epsilon),(\nicefrac{2}{3},\nicefrac{1}{3}))/S((\nicefrac{1}{2}+\epsilon,\nicefrac{1}{2}-\epsilon),(\nicefrac{1}{2},\nicefrac{1}{2}))=e^{L_f/(6\sqrt{2}\epsilon)}$. Hence, as we increase $L_f/\epsilon$, it becomes exponentially more important for the expert to predict accurately near $\nicefrac{2}{3}$ than to predict accurately near $\nicefrac{1}{2}$. In particular, an AI model trained with this scoring rule may be much worse at predicting probabilities near $\nicefrac{1}{2}$ than near $\nicefrac{2}{3}$. Similarly, it is unrealistic to reward a human expert with, say, millions of dollars near $\nicefrac{2}{3}$ and with just a few cents near $\nicefrac{1}{2}$. Unfortunately, it turns out that all possible scoring rules that achieve bound $\epsilon$ under Lipschitz constant $L_f$ have this undesirable property, though the exact bound turns out somewhat complicated.


\co{to do: give this in some kind of big O or similar notation, instead of giving the actual term}
\begin{restatable}{theorem}{needexponential}
\label{thm:need-exponential-new}
Suppose $S$ is a proper scoring rule s.t.\ for some $\epsilon, L_f > 0$ we have that whenever $f$ is $L_f$-Lipschitz, the optimal report $\p$ satisfies $\Vert f(\p) - \p\Vert < \epsilon$. Let $3\epsilon \leq p_l \leq p_h \leq 1-4\epsilon$ and $\delta=\epsilon / (L_f+1)$. Then the ratio of the supremum and infimum over $p_1\in [p_l,p_h]$ of $S((p_1+4\delta,1-p_1-4\delta),(p_1,1-p_1)) - S((p_1,1-p_1),(p_1,1-p_1))$ is at least
\begin{equation*}
    \frac{L_f}{2L_f+6}\left(3\frac{L_f+1}{L_f+3}\right)^{(L_f+1)(p_h-p_l)/(8\epsilon) -5/2}.
\end{equation*}
In particular, for fixed positive $L_f$, this term is exponential in $1/\epsilon$ and for fixed positive $\epsilon$ it is exponential in $L_f$.
\end{restatable}

Intuitively, the assumption on $S$ is that it ensures small accuracy bounds of $\epsilon$ for functions with Lipschitz constant $L_f$. Now note that $|S((p_1+4\delta,1-p_1-4\delta),(p_1,1-p_1)) - S((p_1,1-p_1),(p_1,1-p_1))|$ is the cost to the expert of misreporting by $4\delta$ when the true distribution is $(p_1,1-p_1)$. If this term is large, then the expert cares a lot about not misreporting by $4\delta$, and if the term is small, the expert does not mind misreporting much. Our result shows that the value of this term is much larger for some $p_1$ than it is for others, i.e., that for some probabilities $p_1$ the expert cares a lot more about accurately reporting $p_1$ than it does for other values of $p_1$. Our theorem puts a lower bound on the ratio between the lowest and largest possible values of that term. In particular, this does not hinge on probabilities $p_1$ near $0$ or $1$ and holds even if we restrict attention to probabilities between, say, $1/4$ and $3/4$.

\Cref{theorem:two-outcomes-arbitrarily-good-bounds} shows that in the binary prediction case, given a Lipschitz constant $L_f$ for the environment, we can achieve arbitrarily good bounds $\epsilon$ on the inaccuracy of the performatively optimal report. Unfortunately, this ceases to be possible in the many-outcome case. In that case, if all we know about $f$ is that it has Lipschitz constant $L_f$, there is some error $\epsilon$, linear in $L_f$ as $L_f\rightarrow 0$, that we must allow regardless of what strictly proper scoring rule we use.

\begin{restatable}{theorem}{impossibility}\label{thm:no_higher_dim_bound}
For any Lipschitz constant $L_f$, for $\epsilon>0$ sufficiently small, there is no proper scoring rule $S$ for the three-outcome case that achieves the following property: Whenever $f$ is $L_f$-Lipschitz, there is some performatively optimal report $\p$ with $\Vert f(\p)-\p\Vert  \leq \epsilon$. In particular, there exists some function $\epsilon(L_f)$ with $\epsilon(L_f) \sim c L_f$ as $L_f \rightarrow 0$ for some fixed constant $c$, s.t.\ the above property cannot be achieved with $\epsilon = \epsilon(L_f)$.  Thus, the best achievable bound is in $\Omega(L_f)$ as $L_f \rightarrow 0$, i.e. scales at least linearly with $L_f$ in the limit.
\end{restatable}

\co{LOW PRIORITY: I think it would be good to have an argument in the appendix that the true bound is nothing super simple. I think scoring rules that reward only for guessing the right outcome probably give non-trivial bounds for high $L_f$, for instance...}

\co{LOW PRIORITY: Maybe it'd be good to show that something similar holds for the distance to fixed points as well.}


\begin{figure*}[t]
\centering
\begin{subfigure}[t]{.45\textwidth}
    \centering
    \includegraphics[width=\textwidth]{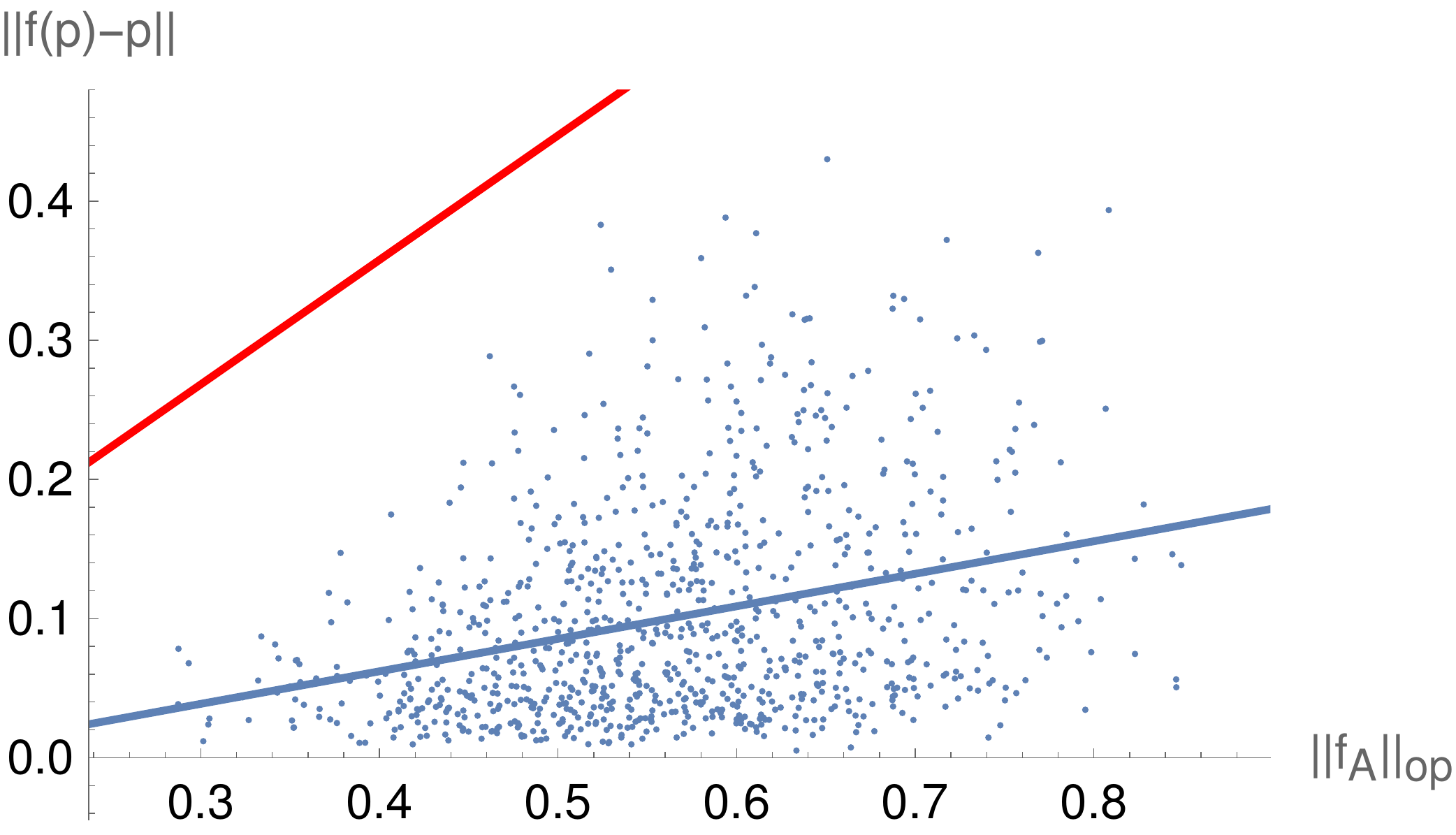}
\end{subfigure}%
\begin{subfigure}[t]{.45\textwidth}
    \centering
    \includegraphics[width=\textwidth]{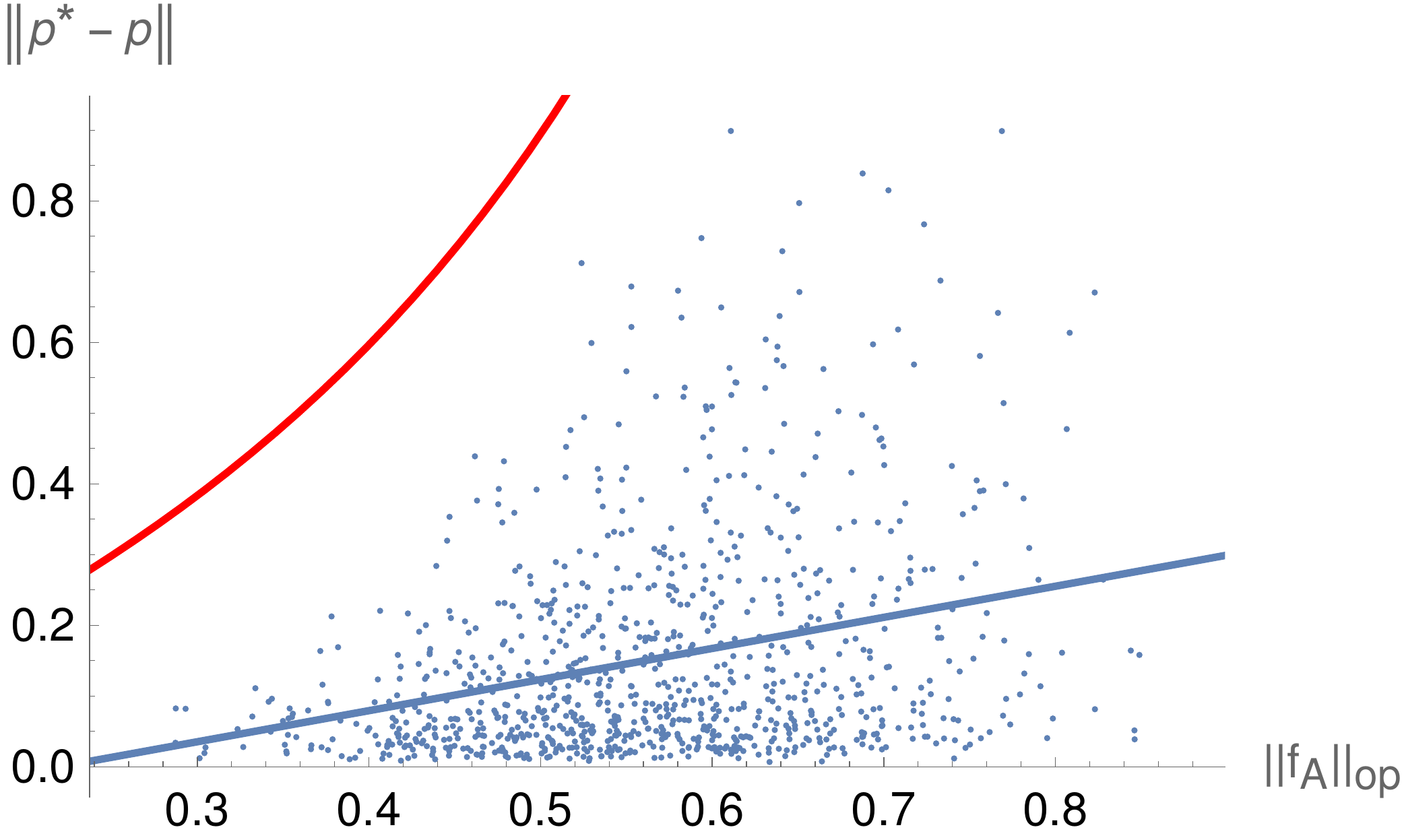}
\end{subfigure}
\caption{Scatter plots showing the L2 inaccuracy (left) and the distance to a fixed point (right) of the performatively optimal reports against the operator norm of $A$ in our experiments. In both plots, each point corresponds to a run of the experiments. The blue lines are found by linear regression on the points. The red lines are the bound given in \Cref{example:bounds-for-brier} as a function of the Lipschitz constant $L_f$.
}
\label{fig:scatterplot-opnormA-l2-inaccuracy-brier}
\label{fig:scatterplot-opnormA-l2-disttofp-brier}

\end{figure*}
\section{Numerical simulations}
\label{numerical-simulations}

\co{LOW PRIORITY: Currently we don't have anything for the log scoring rule anywhere.}

In this section, we provide some numerical simulations for the Brier score, to see how inaccurate performatively optimal predictions might be in practice. Throughout, we consider only affine-linear functions \(f\). This in particular means that all functions $f$ have operator norms between \(0\) and \(1\) and aside from degenerate cases a unique fixed point. The Mathematica notebook for our experiments (including some interactive widgets) is available at \url{https://github.com/johannestreutlein/scoring-rules-performative}. Although our experiments are set in toy models with linear $f$ and small sets of outcomes, they provide an initial estimate of the degree to which predictions can be off, depending on how much influence the expert can exert using their prediction. 

\subsection{Binary prediction}

\co{Make sure to mention somewhere how the L2 norm and the distance between the $p_1$s relate to each other, and justify the use of the L2 norm.}

\textbf{Experimental setup. } We begin with the binary prediction case, i.e., $\mathcal{N}=\{1,2\}$. 
We consider \(f\) to be affine linear with slope \(\alpha\) and fixed point \(\p^*\in \Delta(\mathcal{N})\), thus yielding the functional form $f(\p):=\p^*+\alpha(\p-\p^*)$ for all $\p\in \Delta(\{1,2\})$. Note that for all $\alpha\in [0,1]$ and all $\p^*\in \Pset$, a function thus defined is indeed a function $\Pset\rightarrow \Pset$. For $\alpha<0$, whether $f$ is a function $\Pset\rightarrow \Pset$ depends on $\p^*$. We restrict attention to $\alpha\in [0,1]$ for simplicity. 

\textbf{Graphing inaccuracy and distance to fixed points. }
In \Cref{fig:density-plot-brier-l2-inaccuracy} (left), we plot the inaccuracy $\Vert \p - f(\p) \Vert$ of the performatively optimal report \(\p\) against $\alpha,p^*_1$.
In \Cref{fig:density-plot-brier-l2-disttofp} (right), we plot the L2 distance \(\Vert \p^*-\p\Vert\) of the performatively optimal report \(\p\) to the fixed point \(\p^*\). For that plot we limit $\alpha$ to the range $[0,0.95]$, because of instability at $\alpha\approx 1$. Note that relatively high inaccuracies can be found at various qualitatively different points in the graphs, even when the slope of \(f\) is small, i.e., when the oracle has little influence on the environment.

\textbf{Assessing our bounds. }
To evaluate our bounds, we maximize distances across possible choices of fixed points \(\p^*\in \Delta(\{1,2\})\), and plot the maximal inaccuracy of the optimal prediction as well as the maximal distance from a fixed point in \Cref{fig:max-l2-distance-brier-two-outcome}. We compare to both theoretical bounds from \Cref{example:bounds-for-brier}, i.e., $\Vert  p - f(p) \Vert _2\leq \alpha/\sqrt{2}$ and $\Vert  p - p^* \Vert _2 \leq \alpha/((1-\alpha)\sqrt{2})$.

For both quadratic and log scoring rule (results in \Cref{appendix:experimental-results}), our theoretical bounds are tight for slopes \(\alpha\leq 0.5\). For higher slopes, inaccuracy goes down, as the function \(f(p)\) becomes closer to the identity function, and optimal predictions are bounded in \([0,1]\).

\subsection{Higher-dimensional prediction}


\textbf{Experimental setup. } Next, we turn to higher-dimensional predictions. We consider a model with five possible outcomes and linear $f\colon p\mapsto Ap$ for $A\in\mathbb{R}^{n\times n}$. $f$ is an automorphism on the simplex if and only if all of its columns are in the simplex. We hence randomly generate the matrix $A$ by sampling each column uniformly from the simplex. Note that $A$ is the Jacobian of $f$ at every point.

%

For each $f_A$ thus created, we first find the performatively optimal report $\p$ and the fixed point $\p^*$. We then record the following quantities: the operator norm of $f_A$; 
the distance of the fixed point distribution to the uniform distribution $ \Vert \p^*-\frac{1}{n}\boldsymbol{1}\Vert$; 
 the distance of the optimal report to the uniform distribution $\Vert \p - \frac{1}{n}\boldsymbol{1}\Vert $;
 the distance of the performatively optimal report to the fixed point $\Vert \p^*-\p\Vert$; 
the inaccuracy of the performatively optimal report $\Vert f(\p)-\p\Vert $. 
%
We are interested in how the second two items depend on the first two. We are also interested in how tight our bounds (from \Cref{example:bounds-for-brier}) are.

We collected 1000 random functions \(f_A\), but aborted 52 runs because they didn't terminate within 120 seconds, leaving us with 948 data points.


\co{[LOW PRIORITY:] It would be cool to have a density scatter plot somewhere.}

\textbf{Inaccuracy. }
\Cref{fig:scatterplot-opnormA-l2-inaccuracy-brier} (left) plots the L2 inaccuracy (i.e., the distances $\Vert \p-f(\p) \Vert$). The blue line shows the best linear fit to the data points, which is given by $-0.0314 + 0.234 x$, whereas our bound is $2L_f/\sqrt{5}\approx 0.8944 L_f$. The average L2 inaccuracy is $0.100$ with a standard deviation of $0.0770$. The quartiles are $0.0419, 0.0759, 0.138$. The correlation between the operator norm of $f_A$ and $\Vert \p - f(\p)\Vert $ is $0.312$.


\textbf{Distance to fixed points. }
\Cref{fig:scatterplot-opnormA-l2-disttofp-brier} (right) plots the L2 distance to the fixed point against the operator norm of $f_A$. The linear best fit (blue line) is given by  $-0.0966 + 0.440 x$, whereas our bound is $2L_f/(\sqrt{2}(1-L_f))\approx 0.8944 L_f / (1-L_f).$ The average L2 distance to the fixed point is $0.152$ with a standard deviation of $0.154$ and quartiles $0.0442, 0.0915, 0.210$. The correlation between the operator norm of $f_A$ and $\Vert\p^* - \p\Vert $ is $0.294$.


\textbf{The role of the location of the fixed point. } The graphs for the binary prediction case show that the location of the fixed point matters a lot for the accuracy of optimal reports (though the direction of the effect depends on the slope of $f$). A similar effect can be observed in the many outcome case. In fact, the effect of the location of the fixed point is actually stronger (though less reliable) than the effect of the operator norm of $f_A$. We provide more detail in \Cref{appendix:experiments-many-outcomes-effect-of-fp-loc}.


\textbf{Loose bounds, tight bounds. } \Cref{fig:scatterplot-opnormA-l2-inaccuracy-brier,fig:scatterplot-opnormA-l2-disttofp-brier} show that (in contrast to the binary prediction case), our bounds in terms of the operator norm of $f$ are typically quite loose. For example, the average slack of the inaccuracy bound is $0.404$ with a standard deviation of $0.0998$ and quartiles $0.337, 0.400, 0.471$. Recall from \Cref{example:bounds-for-brier} that in addition to bounds in terms of $L_f$ alone we have bounds in terms of $L_f$ and $\Vert \p-\frac{1}{n}\mathbf{1}\Vert $. These bounds are much tighter with an average slack of $0.0644$ with a standard deviation of $0.0597$ and quartiles $0.0274, 0.0487, 0.0830$. 

\co{
LOW PRIORITY: It might also be interesting to plot things against the location of $\p^*$ or so. Unfortunately, we don't have enormous amounts of space, so maybe we will only be able to do that in the appendix.

LOW PRIORITY: Could give analogous numbers for the bounds on distance to fixed point.}
%

\textbf{Discussion. } Based on our simulations, misprediction in the five outcome case seems similarly problematic as in the binary prediction case. In contrast to the binary case, the bounds in terms of $L_f$ are quite loose. The bounds in terms of $\Vert \p-\frac{1}{n}\mathbf{1}\Vert $ are much tighter. Note that because these bounds depend on the performatively optimal report, they can only be derived a posteriori once a report has been submitted. As in the two-outcome case, both the location of the fixed point and the operator norm/slope of $f$ matter a lot for accuracy and distance to fixed point of the performatively optimal report.

\section{Fixed points via alternative notions of optimality}
\label{stop-gradients}


Here, we focus on alternative settings that lead to accurate predictions and do not induce preferences over fixed points. The idea behind all of them is that, instead of optimizing \(\p\) and \(f(\p)\) jointly, we keep \(\q:=f(\p)\) fixed while choosing a prediction \(\p\) to maximize \(S(\p,\q)\). Repeating this procedure leads to honest predictions, where the choice of fixed point depends on contingent facts such as initialization, instead of being chosen to maximize $S(\p,\p)$. An AI model using this procedure could be safer, because its predictions are honest, and because it does not optimize its choice of fixed point for any goal. In this section we give a summary of a more detailed treatment with formal results in \Cref{appendix:alternative-notions-rationality}.

\textbf{Performative stability. }
Alternatives to performative optimality have been discussed in the performative prediction literature. Translated into our setting, a prediction \(\p^*\) is called \emph{performatively stable} if $\p^*\in \argmax_{\p}\Score(\p,f(\p^*))$.
This implies \(\p^*=f(\p^*)\) whenever \(\Score\) is strictly proper, so performative stability is equivalent to being a fixed point. 


\textbf{Repeated risk minimization and gradient descent. }
\citet{perdomo2020performative} consider learning algorithms that converge to performatively stable points, including repeated risk minimization and repeated gradient descent. In repeated risk minimization, we repeatedly update predictions via \(\p_{t+1}:=\argmax_\p S(\p,f(\p_t))\). Repeated gradient descent instead updates predictions via gradient descent on this objective. There also exist stochastic gradient descent versions of these algorithms \citep{mendler2020stochastic}. 
All of these schemes lead to stable points under appropriate conditions. We include a convergence proof for repeated gradient descent in our setting in \Cref{rrm-and-rgd}. 

\textbf{No-regret learning and prediction markets. } We also provide results for no-regret learning (\Cref{appendix:no-regret}) and prediction markets (\Cref{appendix:prediction-markets}). We introduce a no-regret learning setting and show that policies have sublinear regret if and only if they have sublinear prediction error. This differs from the setting considered by
\citet{pmlr-v162-jagadeesan22a}, in which no-regret policies converge to performatively optimal predictions.
Next, we provide a prediction market model and show that, if the weight of each trader in the market is small, equilibrium predictions by the market are close to fixed points. This is analogous to a result by \citet{hardt2022performative} bounding the distance of a market equilibrium from performatively stable points.

\section{Related work}
\label{related-work}

\textbf{Performative prediction. } In performative prediction, the goal is to find a model parameter \(\theta\in\mathbb{R}^d\) that minimizes an expected loss \(\E[\ell(Z;\theta)]\) where \(Z\) is a stochastic sample, usually a pair of input and target, \(Z=(X,Y)\). Unlike in the vanilla supervised learning setting, \(Z\sim\mathcal{D}(\theta)\) is sampled from a distribution \(\mathcal{D}(\theta)\) that itself depends on the chosen model parameter. Performatively optimal parameters are defined via \(\theta_{\mathrm{PO}}\in \argmin_\theta \E_{Z\sim \mathcal{D}(\theta)}[\ell(Z; \theta)]\), and the definition of performatively stable parameters is \(\theta_{\mathrm{PS}}\in \argmin_\theta \E_{Z\sim \mathcal{D}(\theta_{\mathrm{PS}})}\ell(Z;\theta)\). In general, performatively stable and optimal parameters can differ \cite[][Ex.~2.2]{perdomo2020performative}.

Our setting could be seen as a special case in which \(\theta\) is a single distribution \(\p\), data points are discrete outcomes \(y\), and the distribution \(\mathcal{D}(\theta)\) is given by \(f(\p)\). Unlike in the general performative prediction setting, we can determine the accuracy of a prediction \(\p\) as the distance from the distribution \(f(\p)\) (see \Cref{theorem:Caspar-approx-fix-point}), we can characterize predictions as honest if they are fixed points, and loss functions can be characterized as proper if they incentivize honest reports. As mentioned in \Cref{stop-gradients}, performatively stable points are fixed points and are thus a more desirable solution concept in our setting. 
There are some performative prediction settings in which performative optima can also be seen as manipulative and undesirable, such as in recommendation algorithms \citep{hardt2022performative}. However, as far as we are aware, we are the first to link performative stability to honesty in prediction.




\textbf{Scoring rules. } While the literature on scoring rules generally assumes that predictions are not performative, a few authors in this literature have studied agents manipulating the world \textit{after} making a prediction \cite{shi2009prediction,oka2014predicting}. To our knowledge, the cases discussed do not involve agents influencing the world directly through their predictions. \citet{chan2022scoring} introduce performative probabilistic predictions using scoring rules. However, they focus on particular functional forms of \(f\) and binary predictions and do not provide a more general analysis. Another related setting in which it has been shown that no proper scoring rules exist is that of second-order prediction, in which experts report distributions over first-order distributions to express epistemic uncertainty \citep{bengs2023second}.


\textbf{AI oracles. }
Issues with performativity have been mentioned in the literature on AI predictors or oracles \cite{armstrong2017good}.
Most prior work has focused on alleviating performativity altogether, e.g., by making the oracle predict \emph{counterfactual worlds} it cannot influence. We are not aware of any prior work on specifically the question of whether AI oracles would be incentivized to output fixed points at all.


\textbf{Decision scoring rules and decision markets. }
The literature on decision scoring rules and decision markets considers a setting in which experts make predictions about what would happen if a decision maker were to pursue one course of action or another. The decision maker then chooses based on these predictions, making the predictions performative. As shown by \citet{Othman2010}, the expert may thus be incentivized to mispredict when subject to a proper scoring rule. However, this literature typically takes the perspective of the decision maker and thus assumes some knowledge of $f$. For example, \citet{Othman2010} and \citet{oesterheld2020decision} show that the scoring rule $S$ must be chosen to align in some sense with the decision maker's utility function (and thus $f$). \citet{Chen2014} propose that the decision maker could randomize to set good incentives, which in our setting would entail manipulating $f$.

\textbf{Epistemic decision theory. }
A related topic in philosophy is \emph{epistemic decision theory}. In particular, \citet{greaves2013epistemic} introduces several cases in which outcomes depend on the agent's credences and compares the verdicts of different epistemic decision theories (such as an evidential and a causal version). While some of Greaves' examples involve agents knowably adopting incorrect beliefs, they require joint beliefs over several propositions, and Greaves only considers individual examples. We instead consider only a single binary prediction and prove results for arbitrary scoring rules and relationships between predictions and beliefs.

\textbf{Honest and truthful AI. }
Another related topic is honest and truthful AI \citep{evans2021truthful}. In our setting, an AI that reports an inaccurate prediction to achieve a higher score would be dishonest. \citet{evans2021truthful} discuss issues around training AIs to be truthful and honest, such as difficulties in judging truth. However, they do not explore performativity or proper scoring rules. We simplify our analysis by assuming that a ground truth exists and can be judged objectively. \citet{burns2022discovering} discuss extracting latent knowledge from AIs without relying on incentivizing honest reporting, but also do not address performativity.


\section{Conclusion and future work}
\label{conclusion}
If predictions cannot influence which outcome occurs, then strictly proper scoring rules incentivize experts (humans or AI systems) to report honest predictions. This fails if predictions are performative. We showed that, in general, 
strictly proper scoring rules do not incentivize accurate predictions in a performative prediction setting. We analyzed this inaccuracy quantitatively and gave upper bounds on inaccuracy. We showed that in the case of binary prediction, there exist scoring rules that incentivize arbitrarily accurate predictions. In contrast, for more than two outcomes, it is not possible to achieve arbitrarily strong bounds on accuracy. Our numerical simulations in a toy setting confirm that our bounds are tight in some situations and that inaccurate performative predictions are common. Finally, we showed that by using other types of objectives, such as minimizing regret, we can build AI models that predict fixed points.

We hope that future work will shed further light on practical and safe uses of AI systems as predictors, i.e., oracle AIs. First, some of our bounds could probably be improved or generalized (to non-differentiable $f,G$). Second, it would be valuable to have more specific models of $f$. Precise models of \(f\) may allow for stronger results \citep[cf.][]{Othman2010,oesterheld2020decision}. Third, we take a simplistic view of safety: we take it that incentives to predict honestly are good and that other incentives are problematic. We hope that future work will augment our analysis with more fine-grained models of safety. For example, a common safety concern is power-seeking behavior \citep{omohundro2008basic,turner2021optimal}. One could similarly ask to what extent performative oracle AI will spend compute to improve its ability to influence the world (cf.\ discussions of information acquisition, e.g.\ \citealp{Osband1989}; \citealp{neyman2021binary};
\citealp{li2022optimization}; \citealp{OesConAcquisition}).
Lastly, we are interested in theoretical and experimental evaluations of the practicality of different safe oracle AI designs and training setups.
\begin{acknowledgements}

CO acknowledges funding from the Cooperative AI Foundation, Polaris Ventures (formerly Center for Emerging Risk Research) and Jaan Tallinn's donor-advised fund at Founders Pledge.
JT and RH carried out most of this work as part of the SERI MATS program under the mentorship of Evan Hubinger (JT, RH) and Leo Gao (RH). JT is grateful for support by an Open Phil AI Fellowship and an FLI PhD Fellowship.
We sincerely thank four anonymous reviewers whose insightful comments helped us improve our paper. We are also indebted to Meena Jagadeesan, Erik Jenner, Adam Jermyn, and Marius Hobbhahn for their valuable discussions and feedback, and to Alexander Pan and Bastian Stern for pointing us to the relevant related literature.  
\end{acknowledgements}

\clearpage

\bibliography{refs}

\begin{thebibliography}{54}
\providecommand{\natexlab}[1]{#1}
\providecommand{\url}[1]{\texttt{#1}}
\expandafter\ifx\csname urlstyle\endcsname\relax
  \providecommand{\doi}[1]{doi: #1}\else
  \providecommand{\doi}{doi: \begingroup \urlstyle{rm}\Url}\fi

\bibitem[Agrawal et~al.(2009)Agrawal, Delage, Peters, Wang, and
  Ye]{Agrawal2009}
S.~Agrawal, E.~Delage, M.~Peters, Z.~Wang, and Y.~Ye.
\newblock A unified framework for dynamic pari-mutuel information market
  design.
\newblock In \emph{EC '09 Proceedings of the 10th ACM conference on Electronic
  commerce}, pages 255--264. 2009.

\bibitem[Armstrong(2013)]{armstrong2013risks}
S.~Armstrong.
\newblock Risks and mitigation strategies for oracle ai.
\newblock In \emph{Philosophy and Theory of Artificial Intelligence}, pages
  335--347. Springer, 2013.

\bibitem[Armstrong(2018)]{armstrong2018standard}
S.~Armstrong.
\newblock Standard {ML} {O}racles vs counterfactual ones.
\newblock AI Alignment Forum, 2018.
\newblock URL
  \url{https://www.alignmentforum.org/posts/hJaJw6LK39zpyCKW6/standard-ml-oracles-vs-counterfactual-ones}.

\bibitem[Armstrong and O'Rorke(2017)]{armstrong2017good}
S.~Armstrong and X.~O'Rorke.
\newblock Good and safe uses of {AI} oracles.
\newblock \emph{arXiv preprint arXiv:1711.05541}, 2017.

\bibitem[Armstrong et~al.(2012)Armstrong, Sandberg, and
  Bostrom]{armstrong2012thinking}
S.~Armstrong, A.~Sandberg, and N.~Bostrom.
\newblock Thinking inside the box: Controlling and using an oracle ai.
\newblock \emph{Minds and Machines}, 22\penalty0 (4):\penalty0 299--324, 2012.

\bibitem[Aza{\"\i}s and Wschebor(2009)]{azais2009level}
J.-M. Aza{\"\i}s and M.~Wschebor.
\newblock \emph{Level sets and extrema of random processes and fields}.
\newblock John Wiley \& Sons, 2009.

\bibitem[Bell et~al.(2021)Bell, Linsefors, Oesterheld, and
  Skalse]{bell2021reinforcement}
J.~Bell, L.~Linsefors, C.~Oesterheld, and J.~Skalse.
\newblock Reinforcement learning in newcomblike environments.
\newblock \emph{NeurIPS}, 34:\penalty0 22146--22157, 2021.

\bibitem[Bengs et~al.(2023)Bengs, H{\"u}llermeier, and
  Waegeman]{bengs2023second}
V.~Bengs, E.~H{\"u}llermeier, and W.~Waegeman.
\newblock On second-order scoring rules for epistemic uncertainty
  quantification.
\newblock \emph{arXiv preprint arXiv:2301.12736}, 2023.

\bibitem[Bostrom(2014)]{bostrom2014superintelligence}
N.~Bostrom.
\newblock \emph{Superintelligence}.
\newblock Oxford University Press, 2014.

\bibitem[Brier(1950)]{Brier1950}
G.~W. Brier.
\newblock Verification of forecasts expressed in terms of probability.
\newblock \emph{Monthly Weather Review}, 78\penalty0 (1), 1 1950.

\bibitem[Brown et~al.(2020)Brown, Mann, Ryder, Subbiah, Kaplan, Dhariwal,
  Neelakantan, Shyam, Sastry, Askell, et~al.]{brown2020language}
T.~Brown, B.~Mann, N.~Ryder, M.~Subbiah, J.~D. Kaplan, P.~Dhariwal,
  A.~Neelakantan, P.~Shyam, G.~Sastry, A.~Askell, et~al.
\newblock Language models are few-shot learners.
\newblock \emph{NeurIPS}, 33:\penalty0 1877--1901, 2020.

\bibitem[Burns et~al.(2022)Burns, Ye, Klein, and
  Steinhardt]{burns2022discovering}
C.~Burns, H.~Ye, D.~Klein, and J.~Steinhardt.
\newblock Discovering latent knowledge in language models without supervision.
\newblock \emph{arXiv preprint arXiv:2212.03827}, 2022.

\bibitem[Carvalho(2016)]{carvalho2016overview}
A.~Carvalho.
\newblock An overview of applications of proper scoring rules.
\newblock \emph{Decision Analysis}, 13\penalty0 (4):\penalty0 223--242, 2016.

\bibitem[Chan(2022)]{chan2022scoring}
A.~Chan.
\newblock Scoring rules for performative binary prediction.
\newblock \emph{arXiv preprint arXiv:2207.02847}, 2022.

\bibitem[Chen and Pennock(2007)]{Chen2007}
Y.~Chen and D.~M. Pennock.
\newblock A utility framework for bounded-loss market makers.
\newblock In \emph{UAI'07 Proceedings of the Twenty-Third Conference on
  Uncertainty in Artificial Intelligence}, pages 49--56. 2007.

\bibitem[Chen and Waggoner(2016)]{chen2016informational}
Y.~Chen and B.~Waggoner.
\newblock Informational substitutes.
\newblock In \emph{2016 IEEE 57th Annual Symposium on Foundations of Computer
  Science (FOCS)}, pages 239--247, 2016.
\newblock \doi{10.1109/FOCS.2016.33}.

\bibitem[Chen et~al.(2014)Chen, Kash, Ruberry, and Shnayder]{Chen2014}
Y.~Chen, I.~A. Kash, M.~Ruberry, and V.~Shnayder.
\newblock Eliciting predictions and recommendations for decision making.
\newblock In \emph{ACM Transactions on Economics and Computation}, number~2,
  chapter~6. 6 2014.

\bibitem[Demski(2019)]{demski2019partial}
A.~Demski.
\newblock Partial agency.
\newblock AI Alignment Forum, 2019.
\newblock
  \url{https://www.alignmentforum.org/posts/4hdHto3uHejhY2F3Q/partial-agency}.

\bibitem[Evans et~al.(2021)Evans, Cotton-Barratt, Finnveden, Bales, Balwit,
  Wills, Righetti, and Saunders]{evans2021truthful}
O.~Evans, O.~Cotton-Barratt, L.~Finnveden, A.~Bales, A.~Balwit, P.~Wills,
  L.~Righetti, and W.~Saunders.
\newblock Truthful {AI}: Developing and governing {AI} that does not lie.
\newblock \emph{arXiv preprint arXiv:2110.06674}, 2021.

\bibitem[Foerster et~al.(2018)Foerster, Farquhar, Al-Shedivat, Rockt{\"a}schel,
  Xing, and Whiteson]{foerster2018dice}
J.~Foerster, G.~Farquhar, M.~Al-Shedivat, T.~Rockt{\"a}schel, E.~Xing, and
  S.~Whiteson.
\newblock Dice: The infinitely differentiable monte carlo estimator.
\newblock In \emph{Proceedings of the 35th International Conference on Machine
  Learning}, volume~80, pages 1529--1538. PMLR, 2018.

\bibitem[Gneiting and Raftery(2007)]{gneiting2007strictly}
T.~Gneiting and A.~E. Raftery.
\newblock Strictly proper scoring rules, prediction, and estimation.
\newblock \emph{Journal of the American statistical Association}, 102\penalty0
  (477):\penalty0 359--378, 2007.

\bibitem[Good(1952)]{Good1952}
I.~J. Good.
\newblock Rational decisions.
\newblock \emph{Journal of the Royal Statistical Society. Series B
  (Methodological)}, 14:\penalty0 107--114, 1952.

\bibitem[Greaves(2013)]{greaves2013epistemic}
H.~Greaves.
\newblock Epistemic decision theory.
\newblock \emph{Mind}, 122\penalty0 (488):\penalty0 915--952, 2013.

\bibitem[Hanson(2003)]{Hanson2003}
R.~Hanson.
\newblock Combinatorial information market design.
\newblock \emph{Information Systems Frontiers}, 5\penalty0 (1):\penalty0
  107--119, 2003.

\bibitem[Hardt et~al.(2022)Hardt, Jagadeesan, and
  Mendler-D{\"u}nner]{hardt2022performative}
M.~Hardt, M.~Jagadeesan, and C.~Mendler-D{\"u}nner.
\newblock Performative power.
\newblock In \emph{NeurIPS}, 2022.

\bibitem[Hubinger et~al.(2019)Hubinger, van Merwijk, Mikulik, Skalse, and
  Garrabrant]{hubinger2019risks}
E.~Hubinger, C.~van Merwijk, V.~Mikulik, J.~Skalse, and S.~Garrabrant.
\newblock Risks from learned optimization in advanced machine learning systems.
\newblock \emph{arXiv preprint arXiv:1906.01820}, 2019.

\bibitem[Izzo et~al.(2021)Izzo, Ying, and Zou]{izzo2021learn}
Z.~Izzo, L.~Ying, and J.~Zou.
\newblock How to learn when data reacts to your model: performative gradient
  descent.
\newblock In \emph{Proceedings of the 38th International Conference on Machine
  Learning}, pages 4641--4650. PMLR, 2021.

\bibitem[Jagadeesan et~al.(2022)Jagadeesan, Zrnic, and
  Mendler-D{\"u}nner]{pmlr-v162-jagadeesan22a}
M.~Jagadeesan, T.~Zrnic, and C.~Mendler-D{\"u}nner.
\newblock Regret minimization with performative feedback.
\newblock In \emph{Proceedings of the 39th International Conference on Machine
  Learning}, volume 162, pages 9760--9785. PMLR, 2022.

\bibitem[Jeffrey(1990)]{jeffrey1990logic}
R.~C. Jeffrey.
\newblock \emph{The logic of decision}.
\newblock University of Chicago press, 1990.

\bibitem[Kitchen(1966)]{Kitchen1966}
J.~Kitchen.
\newblock Concerning the convergence of iterates to fixed points.
\newblock \emph{Studia Mathematica}, 27\penalty0 (3):\penalty0 247--249, 1966.

\bibitem[Krueger et~al.(2020)Krueger, Maharaj, and Leike]{krueger2020hidden}
D.~Krueger, T.~Maharaj, and J.~Leike.
\newblock Hidden incentives for auto-induced distributional shift.
\newblock \emph{arXiv preprint arXiv:2009.09153}, 2020.

\bibitem[Letcher et~al.(2019)Letcher, Foerster, Balduzzi, Rockt{\"a}schel, and
  Whiteson]{letcherstable}
A.~Letcher, J.~Foerster, D.~Balduzzi, T.~Rockt{\"a}schel, and S.~Whiteson.
\newblock Stable opponent shaping in differentiable games.
\newblock In \emph{International Conference on Learning Representations}, 2019.

\bibitem[Li et~al.(2022)Li, Hartline, Shan, and Wu]{li2022optimization}
Y.~Li, J.~D. Hartline, L.~Shan, and Y.~Wu.
\newblock Optimization of scoring rules.
\newblock In \emph{Proceedings of the 23rd ACM Conference on Economics and
  Computation}, pages 988--989, 2022.

\bibitem[McCarthy(1956)]{McCarthy1956}
J.~McCarthy.
\newblock Measures of the value of information.
\newblock \emph{Proceedings of the National Academy of Sciences of the United
  States of America}, 42:\penalty0 654--655, 9 1956.

\bibitem[Mendler-D{\"u}nner et~al.(2020)Mendler-D{\"u}nner, Perdomo, Zrnic, and
  Hardt]{mendler2020stochastic}
C.~Mendler-D{\"u}nner, J.~Perdomo, T.~Zrnic, and M.~Hardt.
\newblock Stochastic optimization for performative prediction.
\newblock \emph{NeurIPS}, 33:\penalty0 4929--4939, 2020.

\bibitem[Neely(2021)]{Neely2021}
M.~J. Neely.
\newblock Infinitely often, probability 1, {B}orel-{C}antelli, and the law of
  large numbers, 2021.
\newblock URL
  \url{https://viterbi-web.usc.edu/~mjneely/Borel-Cantelli-LLN.pdf}.

\bibitem[Neyman et~al.(2021)Neyman, Noarov, and Weinberg]{neyman2021binary}
E.~Neyman, G.~Noarov, and S.~M. Weinberg.
\newblock Binary scoring rules that incentivize precision.
\newblock In \emph{Proceedings of the 22nd ACM Conference on Economics and
  Computation}, pages 718--733, 2021.

\bibitem[Ngo et~al.(2022)Ngo, Chan, and Mindermann]{ngo2022alignment}
R.~Ngo, L.~Chan, and S.~Mindermann.
\newblock The alignment problem from a deep learning perspective.
\newblock \emph{arXiv preprint arXiv:2209.00626}, 2022.

\bibitem[Oesterheld and Conitzer(2020{\natexlab{a}})]{OesConAcquisition}
C.~Oesterheld and V.~Conitzer.
\newblock Minimum-regret contracts for principal-expert problems.
\newblock In \emph{Proceedings of the 16th Conference on Web and Internet
  Economics (WINE)}. 2020{\natexlab{a}}.

\bibitem[Oesterheld and Conitzer(2020{\natexlab{b}})]{oesterheld2020decision}
C.~Oesterheld and V.~Conitzer.
\newblock Decision scoring rules.
\newblock In \emph{International Workshop on Internet and Network Economics},
  page 468, 2020{\natexlab{b}}.

\bibitem[Oka et~al.(2014)Oka, Todo, Sakurai, and Yokoo]{oka2014predicting}
M.~Oka, T.~Todo, Y.~Sakurai, and M.~Yokoo.
\newblock Predicting own action: Self-fulfilling prophecy induced by proper
  scoring rules.
\newblock In \emph{Second AAAI Conference on Human Computation and
  Crowdsourcing}, 2014.

\bibitem[Omohundro(2008)]{omohundro2008basic}
S.~M. Omohundro.
\newblock The basic ai drives.
\newblock In \emph{Proceedings of the 2008 conference on Artificial General
  Intelligence: Proceedings of the First AGI Conference}, pages 483--492. IOS
  Press, 2008.

\bibitem[Osband(1989)]{Osband1989}
K.~Osband.
\newblock Optimal forecasting incentives.
\newblock \emph{Journal of Political Economy}, 97\penalty0 (5):\penalty0
  1091--1112, 10 1989.

\bibitem[Ostrovsky(2009)]{ostrovsky2009information}
M.~Ostrovsky.
\newblock Information aggregation in dynamic markets with strategic traders.
\newblock In \emph{Proceedings of the 10th ACM conference on Electronic
  commerce}, pages 253--254, 2009.

\bibitem[Othman and Sandholm(2010)]{Othman2010}
A.~Othman and T.~Sandholm.
\newblock Decision rules and decision markets.
\newblock In \emph{Proc. of 9th Int. Conf. on Autonomous Agents and Multiagent
  Systems (AAMAS 2010), van der Hoek, Kaminka, Lesp\'erance, Luck and Sen
  (eds.), May, 10--14, 2010, Toronto, Canada}, pages 625--632. 2010.

\bibitem[Pennock and Sami(2007)]{Pennock2007}
D.~M. Pennock and R.~Sami.
\newblock Computational aspects of prediction markets.
\newblock In N.~Nisan, T.~Roughgarden, E.~Tardos, and V.~V. Vazirani, editors,
  \emph{Algorithmic Game Theory}, chapter~26, pages 651--675. Cambridge
  University Press, 2007.

\bibitem[Perdomo et~al.(2020)Perdomo, Zrnic, Mendler-D{\"u}nner, and
  Hardt]{perdomo2020performative}
J.~Perdomo, T.~Zrnic, C.~Mendler-D{\"u}nner, and M.~Hardt.
\newblock Performative prediction.
\newblock In \emph{Proceedings of the 37th International Conference on Machine
  Learning}, volume 119, pages 7599--7609. PMLR, 2020.

\bibitem[Rasmussen and Williams(2006)]{Rasmussen2006}
C.~E. Rasmussen and C.~K.~I. Williams.
\newblock \emph{Gaussian Processes for Machine Learning}.
\newblock The MIT Press, 2006.

\bibitem[Russell(2019)]{Russell2019}
S.~J. Russell.
\newblock \emph{Human Compatible: Artificial Intelligence and the Problem of
  Control}.
\newblock Viking, 2019.

\bibitem[Savage(1971)]{Savage1971}
L.~J. Savage.
\newblock Elicitation of personal probabilities and expectations.
\newblock \emph{Journal of the American Statistical Association}, 66:\penalty0
  783--801, 12 1971.

\bibitem[Shi et~al.(2009)Shi, Conitzer, and Guo]{shi2009prediction}
P.~Shi, V.~Conitzer, and M.~Guo.
\newblock Prediction mechanisms that do not incentivize undesirable actions.
\newblock In \emph{International Workshop on Internet and Network Economics},
  pages 89--100. Springer, 2009.

\bibitem[Turner et~al.(2021)Turner, Smith, Shah, Critch, and
  Tadepalli]{turner2021optimal}
A.~Turner, L.~Smith, R.~Shah, A.~Critch, and P.~Tadepalli.
\newblock Optimal policies tend to seek power.
\newblock \emph{NeurIPS}, 34:\penalty0 23063--23074, 2021.

\bibitem[Uesato et~al.(2020)Uesato, Kumar, Krakovna, Everitt, Ngo, and
  Legg]{uesato2020avoiding}
J.~Uesato, R.~Kumar, V.~Krakovna, T.~Everitt, R.~Ngo, and S.~Legg.
\newblock Avoiding tampering incentives in deep {RL} via decoupled approval.
\newblock \emph{arXiv preprint arXiv:2011.08827}, 2020.

\bibitem[Weirich(2020)]{sep-decision-causal}
P.~Weirich.
\newblock {Causal Decision Theory}.
\newblock In E.~N. Zalta, editor, \emph{The {Stanford} Encyclopedia of
  Philosophy}. Metaphysics Research Lab, Stanford University, {W}inter 2020
  edition, 2020.

\end{thebibliography}

\onecolumn

\appendix

\section{Proofs}
\label{appendix:proofs}

\subsection{Preliminaries}
We begin by proving a lemma characterizing the gradient \(\nabla_\p(S(\p,f(\p)))\), which we will use throughout.
\begin{lemma}\label{lemma:derivative-S}
    Assume \(G,g,f\) are differentiable. Then
    \[\nabla_\p(S(\p,f(\p)))=  Dg(\p)^\top (f(\p)-\p) +  Df(\p)^\top g(\p).\]
    If \(S\) is strictly proper and \(\p\in\interior{\Pset}\) an optimal report, then
    \[(\p-f(\p))^\top Dg(\p)= g(\p)^\top Df(\p).\]
\end{lemma}
\begin{proof}
    We have
    \begin{align}\nabla_\p(S(\p,f(\p)))
    &= \nabla_\p\left(G(\p) + g(\p)^\top (f(\p)-\p)\right)
    \\&=g(\p) + Dg(\p)^\top(f(\p)-\p) + Df(\p)^\top g(\p) - \Id g(\p)
    \\&=Dg(p)^\top (f(\p)-\p) + Df(\p)^\top g(\p).\end{align}
Next, if \(\p\) is an optimal report and an interior point, it must be \(\nabla_\p(S(\p,f(\p)))^\top \vec{v}=0\) for any \(\vec{v}\in \TPset \). Since \(\nabla_\p(S(\p,f(\p)))\in\TPset\), it follows that \(\nabla_\p(S(\p,f(\p)))=0\). Hence, using the above, it follows that
\begin{align}
    &\phantom{\Rightarrow} 0=\nabla_\p(S(\p,f(\p)))=Dg(p)^\top (f(\p)-\p) + Df(\p)^\top g(\p)
    \\
    &\Rightarrow Dg(\p)^\top (\p-f(\p)) = Df(\p)^\top g(\p).
\end{align}
\end{proof}

\subsection{Proof of Proposition~\ref{prop:non-fixed-point-optimal}}

\propone*

\begin{proof}

\begin{figure}
\centering
\includegraphics[width=0.5\textwidth]{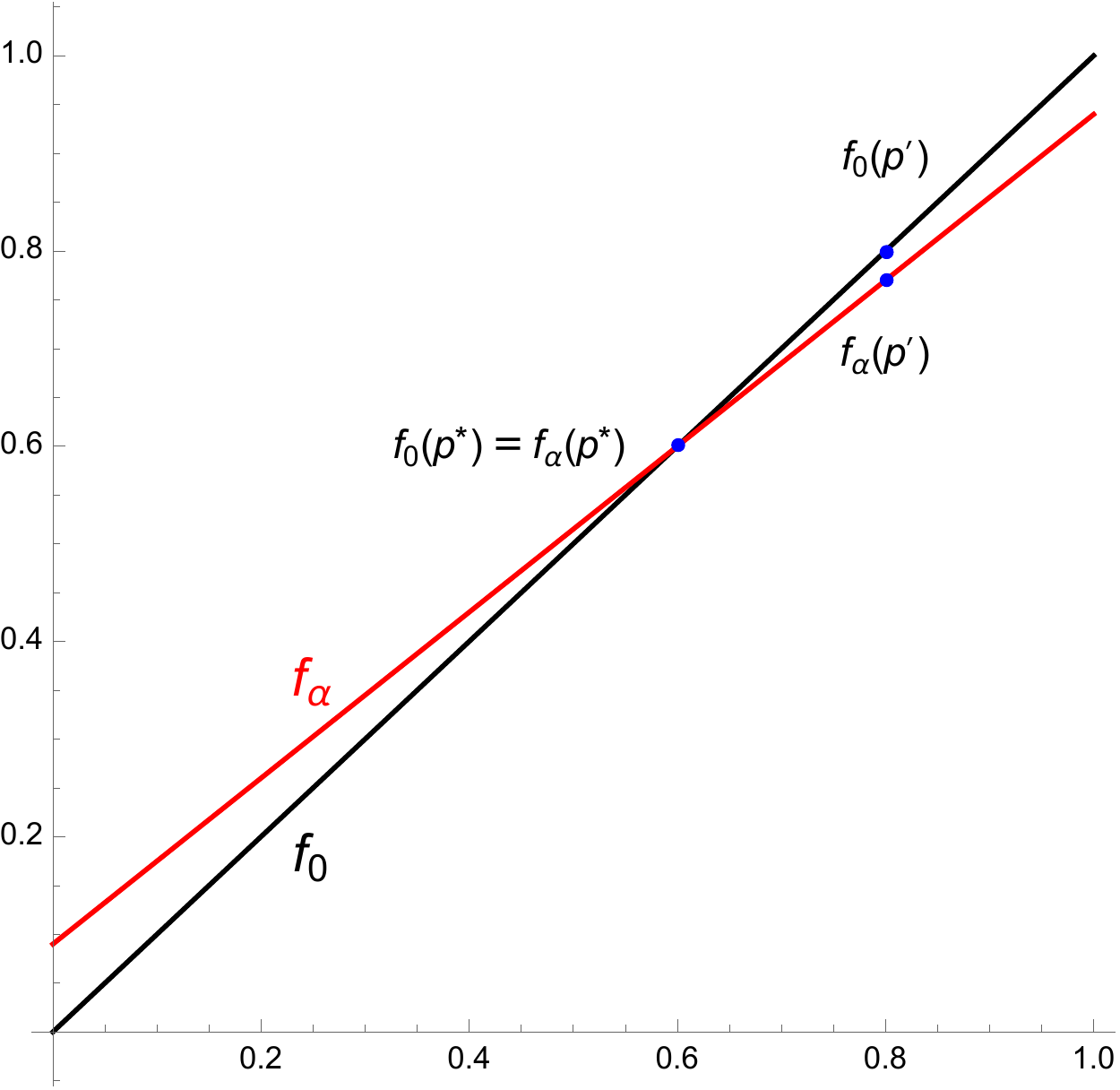}
\caption{Illustration of the setup for our proof. We plot \(f_0\) in black and \(f_\alpha\) for \(\alpha=0.15\) in red, projected onto a single dimension.}
\label{fig:illustration-fp-proof}
\end{figure}

To begin, let \(\p^*\in \interior{\Pset}\) arbitrary and define
\(f_{\alpha}(\p):=(1-\alpha)\p + \alpha \p^*\)
 for 
\(\alpha\in [0,1]\) and \(\p\in\Pset\). Note that since \(\Pset\) is convex, \(f_\alpha(\p)\in\Pset\). Let \(G\) be as in the Gneiting and Raftery characterization of \(S\) (\Cref{theorem:gneiting-raftery}).

To provide an intuition of how our proof will work, consider a binary prediction setting with \(f\) as given in \Cref{fig:illustration-fp-proof}. For any \(\alpha>0\), \(f_\alpha\) has a unique fixed point at \(\p^*\), while \(f_0\) is the identity function, so all points are fixed points of \(f_0\). By strict convexity of \(G\), there exists a point \(\p'\) which receives a strictly higher score than \(\p^*\) if it is a fixed point, so \(\Score(\p',f_0(\p'))>\Score(\p^*,f_0(\p^*))\). \(\p'\) is not a fixed point of \(f_\alpha\) for \(\alpha>0\). However, we will show that \(\Score(\p',f_\alpha(\p'))\) must be continuous in \(\alpha\), which means that we can choose a small enough \(\alpha>0\) to make sure that \(\p'\) remains preferable over \(\p^*\), i.e., \(\Score(\p',f_\alpha(\p'))>\Score(\p^*,f_\alpha(\p^*))\), despite it not being a fixed point.

To formalize the proof, begin by noting that
\[\Vert f_{\alpha}(\p)-f_{\alpha}(\p')\Vert=\Vert (1-\alpha )(\p-\p')\Vert=(1-\alpha)\Vert \p-\p'\Vert\]
for any \(\p,\p'\in \Pset\), so \(f_\alpha\) has Lipschitz constant 
\(L:=(1-\alpha)<1\), and as mentioned, \(\p^*\) is the unique fixed point of \(f_\alpha\). 

Now consider the case 
\(\alpha=0\). As mentioned, every point is a fixed point of \(f_0\). Then by strict convexity of \(G,\) since \(\p^*\) is an interior point, there exists another interior point \(\p'\in \interior{\Pset}\) and \(\epsilon>0\) such that \(G(\p')\geq G(\p^*)+ \epsilon.\)
It follows that
\begin{equation}\label{eq:1}\Score(\p',f_{0}(\p'))=\Score(\p',\p')\geq\Score(\p^*,\p^*)+\epsilon=\Score(\p^*,f_{0}(\p^*))+\epsilon.\end{equation}
So for 
\(\alpha=0\), the model prefers to predict 
\(\p'\)
 over 
\(\p^*\)
 and gets at least 
\(\epsilon\)
 additional expected score. Lastly, note that since \(\p'\) is an interior point as well, it follows that \(G(\p')<\infty\).

Now we show that the model still prefers to predict 
\(\p'\), even for some small 
\(\alpha>0\). 
To that end, note that
\[\Score(\p',f_{\alpha}(\p'))=\E_{y\sim f_{\alpha}(\p')}[S(\p',y)]\]
is linear in \(f_{\alpha}(\p')\), and \(f_{\alpha}(\p')\)
 is affine-linear in 
\(\alpha\)
by construction. This means that 
\(\Score(\p,f_{\alpha}(\p))\)
 is continuous in 
\(\alpha\). So there must exist some small 
\(\alpha>0\)
 such that
\begin{align}\Score(\p',f_{\alpha}(\p'))&\geq\Score(\p',f_{0}(\p'))-\frac{\epsilon}{2}=\Score(\p',\p')-\frac{\epsilon}{2}
\\&\underset{\text{(\ref{eq:1})}}{\geq} \Score(\p^*, \p^*)+\frac{\epsilon}{2}> \Score(\p^*,\p^*)
\\&=\Score(\p^*,f(\p^*)).\end{align}
Choosing 
\(\alpha\)
 in this way, we can define 
\(f:=f_{\alpha}\), and have thus provided a function that satisfies the statement that we wanted to prove.
\end{proof}

\subsection{Proof of Theorem~\ref{prop:fixed-points-optimal-reports-are-rare}}
\label{appendix-proof-of-theorem-8}

We begin with two lemmas. In the following, we always assume a strictly proper scoring rule \(S\) and accompanying functions \(G,g\) as in the Gneiting and Raftery characterization (\Cref{theorem:gneiting-raftery}). Moreover, we let \(\Pi_{n-1}\colon\mathbb{R}^n\rightarrow\mathbb{R}^{n-1}\) be the projection onto \(\mathbb{R}^n\), defined via \(\Pi_{n-1} \x = (x_i)_{1\leq i\leq n-1}\) for \(\x\in\mathbb{R}^n\). We will not go into issues of measurability in our proofs.

First, we show that if \(\p^*\in \interior{\Pset}\) is a fixed point of \(f\), then either \(g(\p^*)=0\) or \(Df(\p)|_{\TPset}\), i.e., the map
\[Df(\p)\colon \TPset\rightarrow\TPset, \vec{v}\mapsto Df(\p)\vec{v},\] is singular.
\begin{lemma}\label{lemma:criticalpointatfixedpoint}
    Let \(G,g\), and \(f\) be differentiable. Let $\p\in \interior{\Pset}$ be a fixed point of $f$ and a performatively optimal prediction. Then $Df(\p)|_\TPset$ is singular or $g(\p)=0$.
\end{lemma}

\begin{proof}
Note that \(f(\p)\in \Pset\) for all \(\p\in\Pset\), so \(\partial_{\vec{v}} f(\p)=Df(\p)\vec{v}\in \TPset\) for all \(\vec{v}\in \TPset\). Hence, \(Df(\p)\) defines an automorphism \(Df(\p)|_{\TPset}\).

It follows from \Cref{lemma:derivative-S} that
$Dg(\p)^\top(\p-f(\p)) =Df(\p)^\top g(\p)$. Since $f(\p)-\p=0$, it must be $Df(\p)^\top g(\p)=0$, so either \(g(\p)=0\), or \(Df(\p)^\top\) (and thus also \(Df(\p)\)) is singular when restricted to \(\TPset\).
\end{proof}

Next, we show that the fixed points of \(f\) are almost surely not at points \(\p\) such that \(g(\p)=0\), under our assumptions on the distribution over \(f\).

\begin{lemma}\label{lemma:gneq0}
    Let $\mathcal{F}\defeq\{F(\p)\}_{\p\in \interior{\Pset}}$ be a stochastic process with values in \(\Pset\) and assume that for each \(\p\in \interior{\Pset}\), the random vector \(\Pi_{n-1} F(\p)\) has a density \(h_{\Pi_{n-1} F(\p)}\).
    Then almost surely if \(F(\p)=\p\) for some \(\p\in \interior{\Pset}\) then \(g(\p)\neq 0\). That is,
    \[\mathbb{P}(\exists \p\colon F(\p)=\p\land g(\p)=0)=0.\]
\end{lemma}

\begin{proof}
    First note that if $S$ is strictly proper, then $G$ is strictly convex and so there exists at most one \(\p\in\Pset\) with \(g(\p)=0\). If there is no such point, then we are done. Otherwise, let that point be $\p^*$. Since we assume that \(\Pi_{n-1} F(\p^*)\) has a density function \(h_{\Pi_{n-1} F(\p^*)}\), it follows that
    \[\mathbb{P}(F(\p^*)=\p^*)= \mathbb{P}(\Pi_{n-1} F(\p^*)=\Pi_{n-1} \p^*)=\int_{\{\Pi_{n-1}\p^*\}} h_{\Pi_{n-1} F(\p^*)}(\x)d\x=0.\]
\end{proof}

Lastly, we require a result about random fields. 
The following is adapted from Proposition~6.11 in \citet{azais2009level}.
\begin{proposition}[\cite{azais2009level}, Proposition~6.11]\label{prop:random-field}
    Let \(\mathcal{Y}=\{Y(\x)\}_{\x\in W}\) be a random field with values in \(\mathbb{R}^{d}\) and \(W\) an open subset of \(\mathbb{R}^{d'}\). Let \(\vec{u}\in\mathbb{R}^{d}\) and \(I\subseteq W\). Assume that
    \begin{itemize}
\item the sample paths \(\x\rightsquigarrow Y(\x)\) are continuously differentiable
\item for each \(\x\in W\), \(Y(\x)\) has a density \(h_{Y(\x)}\) and there exists a constant \(C\) such that \(h_{Y(\x)}(\y)\leq C\) for all \(\x\in I\) and \(\y\in\mathbb{R}^{d}\).
\item The Hausdorff dimension of \(I\) is strictly smaller than \(d\).
\end{itemize}
Then, almost surely, there is no point \(\x\in I\) such that \(Y(\x)=\vec{u}\).
\end{proposition}

Now we can turn to the proof of the main result.

\fprare*

\begin{proof}

We want to show that almost surely there does not exist \(\p\in \interior{\Pset}\) such that \(F(\p)=\p\) and \(\p\) is performatively optimal. I.e., we want to show that
\[\mathbb{P}(\exists \p\colon F(\p)=\p\land \p\in \argmax \Score(\p,F(\p)))=0.\]

First, let \(\p^*\in \interior{\Pset}\) be a performatively optimal report. By \Cref{lemma:criticalpointatfixedpoint}, either \(g(\p^*)=0\) or \(DF(\p^*)|_{\TPset}\) is singular. Moreover, by assumption, \((\Pi_{n-1} F(\p),\Pi_{n-1} DF(\p)\vec{v})\) has a density function for any \(\vec{v}\in \TPset\cap S^{n-1}\), and thus also \(\Pi_{n-1}F(\p)\) has one. Hence, by \Cref{lemma:gneq0}, it follows that if \(F(\p)=\p\) for some \(\p\in \interior{\Pset}\), then almost surely \(g(\p)=0\). 

Second, we need to show that also almost surely \(DF(\p)|_{\TPset}\) is invertible at any fixed point of \(F\). To that end, define the random field \(\mathcal{Y}:=\{Y(\p,\vec{v})\}_{(\p,\vec{v}) \in W}\) where \(W:=\interior{\Pset}\times \TPset\) and
\[Y(\p,\vec{v}):=(\Pi_{n-1} F(\p)-\Pi_{n-1} \p,\Pi_{n-1}DF(\p)\vec{v}),\]
with values in \(\mathbb{R}^{n-1}\times\mathbb{R}^{n-1}\).

Note that since \(F\) is in \(\mathcal{C}^2\), \(DF\) is continuously differentiable, and thus also \(Y\). Moreover
\[h_{Y(\p,\vec{v})}(\x,\y)=h_{\Pi_{n-1} F(\p),\Pi_{n-1} DF(\p)\vec{v}}(\x+\Pi_{n-1}\p,\y)\leq C\] for \(\x,\y\in\mathbb{R}^{n-1}\) by assumption. Finally, define \(\vec{u}:=(0,0)\in\mathbb{R}^{n-1}\times\mathbb{R}^{n-1}\) and \(I:=\Pset\times (\TPset \cap S^{n-1})\), where \(\TPset \cap S^{n-1}=\{\vec{v}\in \TPset\mid \Vert \vec{v}\Vert=1\}\). Note that the Hausdorff dimension of \(I\) is \(n-1+n-2=2n-3\), while \(Y\)'s values are \(2n-2\)-dimensional.

This shows all conditions of \Cref{prop:random-field}, so we can apply it to \(Y\) to conclude that almost surely there exists no \(\p,\vec{v}\in I\) such that \(Y(\p,\vec{v})=(0,0)\). This means that almost surely there exists no point \(\p\in \Pset\) such that \(F(\p)=\p\) and such that \(DF(\p)|_{\TPset}\) is singular, since if such a point existed, then also there would be a vector \(\vec{v}\in \TPset \cap S^{n-1}\) such that \(DF(\p)\vec{v}=0\) and thus \(\Pi_{n-1}DF(\p)\vec{v}=0\), implying that
\[Y(\p,\vec{v})=(\Pi_{n-1}F(\p)-\Pi_{n-1}\p,\Pi_{n-1}DF(\p)\vec{v})=0.\]

Summarizing our argument, it follows that
\[\mathbb{P}(\exists \p\colon {F(\p)=\p}\land {\p\in \argmax \Score(\p,F(\p))})
\]
\[\leq \mathbb{P}(\exists \p \colon {F(\p)=\p}\land {g(\p)=0})+\mathbb{P}(\exists \p \colon {F(\p)=\p}\land {DF(\p)\text{ is singular}})=0.\]
This concludes the proof.
\end{proof}

We conclude by providing an example of a stochastic process that satisfies our conditions, for the binary prediction case.
\begin{example}
\label{ex:gaussian-process}
Consider a Gaussian process \(\{X(p)\}_{p\in (0,1)}\) with values in \(\mathbb{R}\), with infinitely differentiable kernel and mean functions. We can make it into a process \(F(\p)\) by defining \(F(\p)= (X(p_1), 1-X(p_1))\) for \(\p\in \Delta([2])\). Note that the paths of \(F\) are infinitely differentiable and the values of \(\Pi_1 F(\p)=X(p_1)\) and its directional derivatives \(\Pi_1 DF(\p)\vec{v}=X'(p_1)v_1\) are jointly Gaussian and thus have a bounded density 
\cite[see][Ch.~9.4]{Rasmussen2006}. To deal with the restriction that \(X(p)\in [0,1]\) for \(p\in (0,1)\), we could condition on the event \(E:=\{\forall p\colon  F(p)\in [0,1]\}\), for instance. Then paths are still twice differentiable, and we claim that \(h_{X(p)|E}\), defined as the density of \(X\) at point \(p\), conditional on \(E\), is still bounded. To see that, note that if \(\mathbb{P}(E)>0\), then we are done, since then 
\[h_{X(p)|E}(x)=\frac{\mathbbm{1}_E(x)}{\mathbb{P}(E)}h_{X(p)}(x).\] We leave it as an exercise to the reader to prove that \(\mathbb{P}(E)>0\).
\end{example}




%


\subsection{Proof of Theorem~\ref{theorem:Caspar-approx-fix-point}}

\inaccuracybound*

Recall that we assume that $g(\p)$ is normalized to be orthogonal to $\mathbf{1}$. Note that this is also the choice that minimizes \(\Vert g(\p)\Vert\) and makes sure that \(\Vert g(\p)\Vert = \Vert g(\p)^\top|_{\TPset} \Vert_{\mathrm{op}}\), where $g(\p)^\top|_{\TPset}$ denotes the function $\TPset\rightarrow\mathbb R\colon \mathbf v \mapsto g(\p)^\top \mathbf v$. 
This is due to the Cauchy--Schwarz inequality 
and the Pythagorean theorem, since \(\Vert g(\p)+\alpha \mathbf{1}\Vert^2=\Vert g(\p)\Vert^2 + |\alpha|^2 \Vert \mathbf{1}\Vert^2\) for any \(\alpha\in\mathbb{R}\) when \(g(\p)\in\TPset\). Moreover, by Cauchy--Schwarz, we have \(\Vert g(\p)^\top \vec{v}\Vert \leq \Vert g(\p)\Vert\Vert \vec{v}\Vert\) for any \(\vec{v}\in\TPset\) and if \(g(\p)\in\TPset\) then \(\Vert g(\p)^\top|_{\TPset}\Vert_{\mathrm{op}}\geq\Vert g(\p)^\top g(\p)\Vert/\Vert g(\p)\Vert = \Vert g(\p)\Vert\).

\begin{proof}
Assume \(\p\) is a performatively optimal report and that \(Dg(\p)|_{\TPset}\succeq\gamma_p\).
Note that this is equivalent to all eigenvalues of the function \(Dg(\p)|_{\TPset}\) being at least \(\gamma_p\), assuming \(Dg(\p)|_{\TPset}\) is symmetric. Moreover, \(Dg(\p)\) must be symmetric if \(G\) is twice differentiable (note that continuous differentiability is not needed since we assume differentiability in general, not just existence of the coordinate partial derivatives). This can be used to calculate our bound in practice.

Consider \(\nabla_\p(S(\p,f(\p)))^\top (f(\p)-\p),\) the directional derivative of \(\varphi\colon \p\mapsto S(\p,f(\p))\) in the direction $(f(\p)-\p)$. Note that this derivative must be at most zero: The line from $\p$ to $f(\p)$ lies entirely within the probability simplex, and so if the derivative were positive, $S(\p,f(\p))$ could be increased by moving in the direction of $f(\p)$ from $\p$. By \Cref{lemma:derivative-S}, we know that
\[\nabla_\p(\Score(\p,f(\p)))=Dg(\p)^\top (f(\p)-\p) +  Df(\p)^\top g(\p).\]
It follows that
\begin{align}&0\geq \nabla_\p(S(\p,f(\p)))^\top (f(\p)-\p)=
(f(\p)-\p)^\top Dg(\p) (f(\p)-\p)+g(\p)^\top Df(\p) (f(\p)-\p)\\
\Rightarrow&
-g(\p)^\top Df(\p) (f(\p)-\p)
\geq (f(\p)-\p)^\top  (Dg(\p)) (f(\p)-\p).
\end{align}
Using that \(Dg(\p)|_{\TPset}\succ \gamma_\p\) and thus \((f(\p)-\p)^\top  (Dg(\p)) (f(\p)-\p) \geq \gamma_\p\Vert f(\p)-\p\Vert^2\)
, it follows that
\begin{eqnarray*}
    && \gamma_\p\Vert f(\p)-\p\Vert^2\\
    &\leq & (f(\p)-\p)^\top  Dg(\p) (f(\p)-\p) \\
    &\leq& - g (\p)^\top Df(\p)(f(\p)-\p)\\
    &\leq& \vert g(\p)^\top Df(\p)(f(\p)-\p)\vert\\
    &\underset{\text{Cauchy-Schwarz}}{\leq}& \Vert g(\p)\Vert \Vert Df(\p)(f(\p)-\p)\Vert\\
    &\leq& \Vert g(\p)\Vert \Vert Df(\p)\Vert_{\mathrm{op}}\Vert f(\p)-\p\Vert
\end{eqnarray*}
Dividing by $\gamma_\p\Vert f(\p)-\p\Vert $, we get that $\Vert f(\p)-\p\Vert\leq \Vert Df(\p)\Vert_{\mathrm{op}} \Vert g(\p)\Vert/\gamma_\p$.

For the ``in particular'' part, note that if \(f\) is Lipschitz continuous with constant \(L_f\), then \(\Vert Df(\p)\Vert_{\op}\leq L_f\) for all \(\p\). Moreover, if \(G\) is Lipschitz continuous with constant \(L_G\), we have
\[L_G\geq \Vert DG(\p)\Vert_{\mathrm{op}}=\Vert g(\p)^\top \Vert_{\op}=\Vert g(\p)\Vert\]
for all \(\p\in\Pset\). Here, in the last step, we have used that for the Euclidean norm
\[\Vert g(\p)^\top\Vert_{\op}=\max_{\vec{v}\in \TPset}\frac{g(\p)^\top \vec{v}}{\Vert\vec{v}\Vert}=\frac{g(\p)^\top g(\p)}{\Vert g(\p)\Vert}=\Vert g(\p)\Vert. \]
Lastly, \(G\) being \(\gamma\)-strongly convex implies that \(D^2G(\p)\succeq \gamma\) for all \(\p\in\Pset\), and thus also \(Dg(\p)=D^2G(\p)^\top\succeq \gamma\).

Putting everything together, we get
\[    \Vert f(\p)-\p\Vert
\leq
\frac{\Vert g(\p)\Vert \Vert Df(\p)\Vert_{\mathrm{op}}}{\gamma_p}
\leq 
\frac{L_GL_f}{\gamma}\]
for all performatively optimal reports \(\p\).
\end{proof}

\subsection{Proof of Theorem~\ref{thm:distance-to-fp}}

\bounddisttofp*

\begin{proof}
For \textit{any} \(\p\in \Delta(\mathcal{N})\), we have
\begin{eqnarray*}
\Vert \p-\p^*\Vert
&\underset{\text{triangle ineq.}}{\leq} & \Vert \p -f(\p)\Vert + \Vert f(\p)-\p^*\Vert \\
&\underset{p^*\text{ fixpoint}}{=}& \Vert \p -f(\p)\Vert + \Vert f(\p)-f(\p^*)\Vert\\
&\leq& \Vert\p-f(\p)\Vert + L_f\Vert \p-\p^*\Vert 
\end{eqnarray*}
Solving for $\Vert \p-\p^*\Vert $ yields
\begin{equation*}
\Vert \p-\p^*\Vert\leq \frac{\Vert \p-f(\p)\Vert}{1-L_f}.
\end{equation*}
Hence, if \(\p\in \Pset\) is an optimal prediction, it follows by \Cref{theorem:Caspar-approx-fix-point} that
\begin{eqnarray*}
\Vert \p-\p^*\Vert_2&\leq& \frac{\Vert \p-f(\p)\Vert}{1-L_f}\\
&\leq& \frac{\Vert Df(\p)\Vert_{\mathrm{op}} \Vert g(\p)\Vert}{\gamma_{\p}(1-L_f)}\\
&\leq& \frac{L_f \Vert g(\p)\Vert}{\gamma_{\p}(1-L_f)},
\end{eqnarray*}
which concludes the proof.
\end{proof}

\subsection{There is no non-trivial bound on the distance to the fixed point as \(L_f\rightarrow 1\)}
\label{appendix:Lf1-no-non-trivial-bound}

We here show why Theorem~\ref{thm:distance-to-fp} requires that we have some bound $L_f<1$ on the function $f$. Specifically, we show that if $f$ can have Lipschitz constants arbitrarily close to $1$, then even in the two-outcome case, only the trivial bound on the difference to the fixed point holds. (The trivial bound is $\lVert \p^*-\p \rVert \leq \sqrt{2}$, because any two points in $\Delta(\{1,2\})$ are at most $\lVert (0,1) - (1,0) \rVert=\sqrt{2}$ apart.) We prove that this holds even for the binary case.

\begin{proposition}
\label{prop:Lf1-no-non-trivial-bound}
    Consider the case of two outcomes, i.e., let $\mathcal{N}=\{1,2 \}$. Let $S$ be any strictly proper scoring rule. Then there exist functions $f$ with Lipschitz constants smaller than $1$ such that $\lVert \p^*-\p \rVert$ is arbitrarily close to $\sqrt{2}$, where $\p^*$ is the fixed point of $f$ and $\p$ is the optimal prediction for $S,f$.
\end{proposition}

\begin{figure}
    \centering
    \includegraphics[width = 0.5\linewidth]{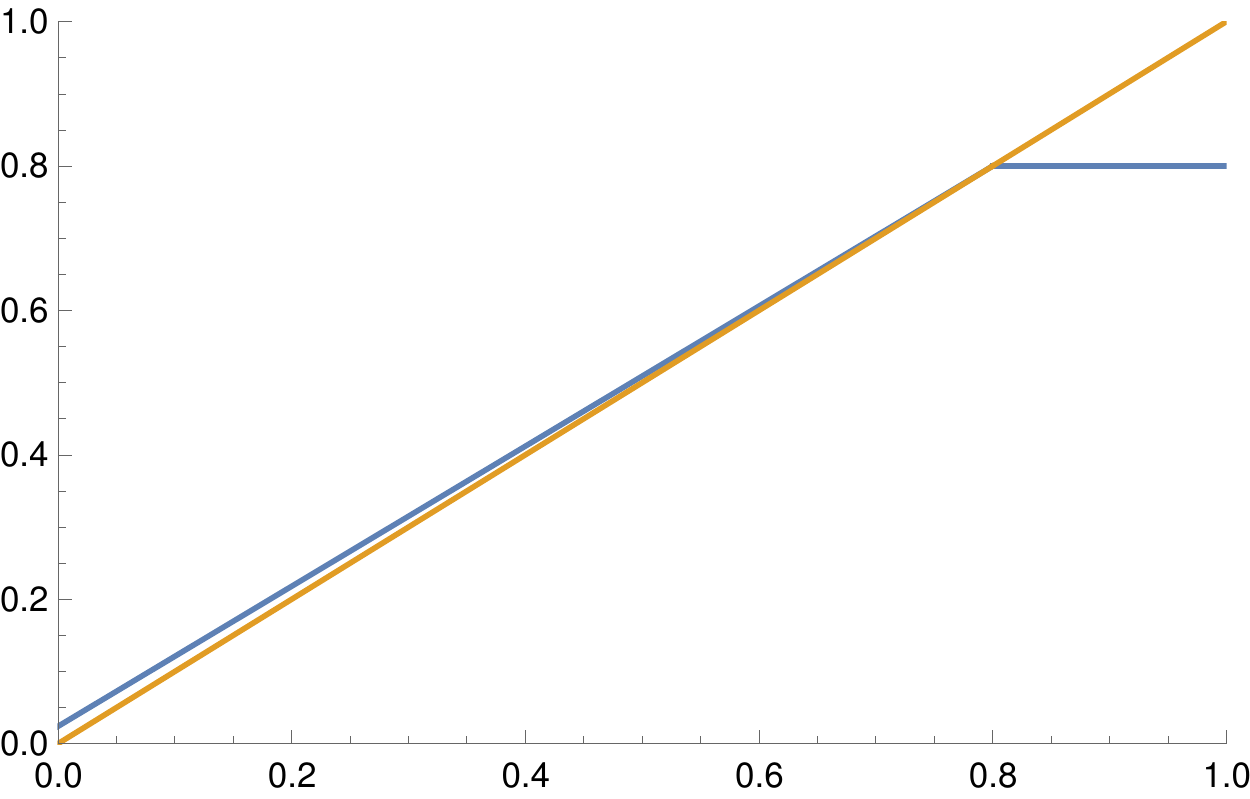}
    \caption{The blue line is the function used in the proof of \Cref{prop:Lf1-no-non-trivial-bound}. The orange line is the identity function.}
    \label{fig:funct-proof-of-prop:Lf1-no-non-trivial-bound}
\end{figure}

We here give some intuition for why the result holds. Recall that, roughly speaking, scoring rules generally induce a preference for extreme honest predictions over non-extreme honest predictions (see \Cref{preferences-between-fps}). In particular, in the binary case any scoring rule must either incentivize near-0 honest predictions or near-1 honest predictions (or both) over honest relatively close-to-uniform predictions. Consider the case where $S$ incentivizes predictions close to $0$ over more uniform predictions and take the function $f$ in \Cref{fig:funct-proof-of-prop:Lf1-no-non-trivial-bound}. The unique fixed point is at $0.8$. But if a prediction close to $0$ is made, the prediction is \textit{approximately} honest while more extreme than $0.8$. It turns out that a slight dishonesty (discrepancy between the report $\p$ and the true distribution $f(\p)$) can be outweighed by the fact that the prediction is more extreme. Predicting near $0$ may therefore be a better report than a of prediction $0.8$.

Note that in this example, the distance of the optimal report to fixed point ($\lVert \p - \p^* \rVert$) and the inaccuracy of the optimal report ($\lVert \p - f(\p) \rVert$) come apart: The optimal report might be far from the fixed point but still very accurate.

\begin{proof}
For notational convenience, we consider functions $f:[0,1]\rightarrow [0,1]$ on a single probability and similarly scoring rules $S:[0,1]\times [0,1]\rightarrow \mathbb{R}$.

Let $\zeta$ be a small positive real number and let $\delta$ be s.t.\ $0<\delta<\zeta/2$. By the strict convexity of the function $x\mapsto S(x,x)$, one of the following must be the case:
\begin{enumerate}[nolistsep]
    \item $S(\delta,\delta)>\max_{x\in [2\delta,1-2\delta]} S(x,x)$; or
    \item $S(1-\delta,1-\delta)>S(x,x)$ for all $x$ between $2\delta$ and $1-2\delta$.
\end{enumerate}

Consider the first case: Now for small positive $\epsilon$ consider the function $f_\epsilon$ that starts at some small positive value and increases linearly at rate $1-\epsilon$ from $0$ to $1-\zeta$ and is fixed at value $1-\zeta$ from $1-\zeta$ to $1$. \Cref{fig:funct-proof-of-prop:Lf1-no-non-trivial-bound} illustrates this function for $\zeta = 0.2, \epsilon = 0.03, \delta = 0.1$. Formally, $f_\epsilon(p)=1-\zeta$ for $p\geq \zeta$ and otherwise $f_\epsilon(p)=1-\zeta - (1-\zeta - p)(1-\epsilon)$. Note that $f_\epsilon$'s fixed point is $1-\zeta$ and $f_\epsilon$'s Lipschitz constant is $1-\epsilon$.
We will show that for small enough $\epsilon$ the optimal report for $f_\epsilon$ and $S$ is then close to $0$ and thus almost $1$ away from the fixed point, which means the distance is close to $\sqrt{2}$ in the simplex. 
Note that by continuity of $f_\epsilon$ and linearity (and thus continuity) of $S(p,q)$ in $q$, we have that
$S(\delta,f_\epsilon(\delta)) \rightarrow S(\delta,\delta)$ as $\epsilon\rightarrow 0$. Thus, for small enough $\epsilon$, we have for all $x$ between $2\delta$ and $1-2\delta$ that $S(\delta,f_\epsilon(\delta))>S(x,x)$. It follows that the optimal report $p$ cannot have $f_\epsilon(p)\in [2\delta,1-2\delta]\supset [\zeta,1-\zeta]$, because then we would have that $S(\delta,f_\epsilon(\delta))>S(f_\epsilon(p),f_\epsilon(p))>S(p,f_\epsilon(p))$, i.e., $\delta$ would be a better report. By construction of $f_\epsilon$, this means that the optimal report cannot be in $[2\delta,1]$. 
Thus, the distance of the optimal report to the fixed point is at least $1-\zeta-2\delta$. By choosing $\delta$ and $\zeta$ to be small, we can make this arbitrarily close to $1$.

The second case can be considered analogously, by considering a function $f$ that is constant at value $\zeta$ from $0$ to $\zeta$ and then increases linearly at rate $1-\epsilon$.
%
%
%
%
%
\end{proof}

\subsection{Proof of Theorem~\ref{theorem:two-outcomes-arbitrarily-good-bounds}}

\exponentialthm*

\begin{proof}
Consider the exponential scoring rule defined by $G(\p)=\frac{2}{K}e^{Kp_1}$ and $g(\p)=(e^{Kp_1},-e^{Kp_1})^\top$ s.t.\ $Dg(\p)=\begin{pmatrix} Ke^{Kp_1} & -Ke^{Kp_1}\\ 0& 0\end{pmatrix}$ and $\Vert g(\mathbf{p})\Vert =\sqrt{2}e^{Kp_1}$. The only eigenvalue of $Dg(\p)\vert_{\TPset}$ is $Ke^{Kp_1}$. Thus, $Dg(\p)\succeq Ke^{Kp_1}$. Therefore, by \Cref{theorem:Caspar-approx-fix-point}, the optimal report $\p$ satisfies $\Vert \p -f(\p)\Vert \leq \sqrt{2}L_f/K$. Thus, by choosing $K=L_f/(\sqrt{2}\epsilon)$, we obtain the desired bound. If $L_f<1$, then by \Cref{thm:distance-to-fp} we further have that $\Vert \p -\p^*\Vert  \leq \frac{L_f}{(1-L_f)}\frac{\sqrt{2}}{K}$, so that we can achieve the desired bound by setting $K=(1-L_f)/(\sqrt{2}\epsilon L_f)$.
\end{proof}

\subsection{Proof of Theorem \ref{thm:need-exponential-new}}


Throughout this section we use the following simplifying notation. Let $S$ be a proper scoring rule for the two-outcome case with $G,g$ as per \Cref{theorem:gneiting-raftery}. Then for a single probability $p\in [0,1]$ we define $G(p)=G(p,1-p)$ and $g(p)=(1,-1)g(p,1-p)$. And $S(p,q)=S((p,1-p),(q,1-q))$. Then we have that $S(p,q)=g(p)(q-p)+G(p)$, where $g$ as a function on $[0,1]$ is a subgradient of $G$ as a function on $[0,1]$. Conversely, note that any function of this form induces a proper scoring rule on $\Delta(\{1,2\})$.

First, we prove a result that reduces the claim about $S$ to a claim about $g$.

\begin{lemma}\label{lemma:bounding-S-ratios-bounding-g-ratios}
    Let $[p_1,p_2]$ be any interval and $S$ be a proper scoring rule defined via $g$ as usual. Then
    \begin{equation*}
        \frac{\sup_{p\in [p_1,p_2]} S(p,p)-S(p+x,p)}{\inf_{p\in [p_1,p_2]} S(p,p)-S(p+x,p)} \geq \frac{1}{4} \frac{\sup_{p\in [p_1,p_2]} g(p+x)-g(p)}{\inf_{p\in [p_1,p_2]} g(p+x)-g(p)}.
    \end{equation*}
\end{lemma}

\begin{proof}
For our proof, we will use the following bounds:
    \begin{eqnarray*}
        S(p,p)-S(p+x,p) &=& x g(p+x) - \int_p^{p+x} g(t)dt\\
        &\geq & xg(p+x)- xg(p+x)/2 - xg(p+x/2)/2\\
        &=& x(g(p+x)-g(p+x/2))/2\\
        S(p,p) - S(p+x,p) &=& xg(p+x)-\int_p^{p+x} g(t)dt\\
        &\leq & xg(p+x) - xg(p)\\
        &=& x(g(p+x) - g(p)).
    \end{eqnarray*}
Using these bounds, we can prove the lemma as follows:
    \begin{eqnarray*}
        \frac{\sup_p S(p,p)-S(p+x,p)}{\inf_p S(p,p)-S(p+x,p)}
        &\geq& \frac{\sup_p x(g(p+x)-g(p+x/2))/2}{\inf_p x(g(p+x) - g(p))}\\
        &=& \frac{1}{2}\frac{\sup_p g(p+x)-g(p+x/2)}{\inf_p g(p+x) - g(p)}\\
        &\geq & \frac{1}{4}\frac{\sup_p g(p+x)-g(p)}{\inf_p g(p+x) - g(p)}.
    \end{eqnarray*}
\end{proof}

\begin{lemma}\label{compounding_hops}
    Let $y\geq 0$, $h > 0$.
    Let $g\geq 0$ be strictly increasing on $[a,b]$ s.t.\ $g(x+h)-g(x)\geq y g(x)$ for all $x\in [a,b]$. Then
    \begin{equation*}
        \frac{\sup_{x\in [a,b]} g(x+h)-g(x)}{\inf_{x\in [a,b]} g(x+h)-g(x)} \geq y(1+y)^{\lfloor (b-a)/h \rfloor-1}.
    \end{equation*}
\end{lemma}

\begin{proof}
    Let $N = \lfloor \frac{b-a}{h} \rfloor$.
    
    Note that $g(x+h) \geq (1+y)g(x)$ for $x \in [a,b]$. Thus, iterating, we get that $g(a+Nh)\geq (1+y)^{N-1}g(a+h)$.

    As a consequence, since $a+Nh \leq b$, we have:
    \begin{align*}
        g(a+(N+1)h)-g(a+Nh) &\geq y g(a+Nh)\\
        &\geq y(1+y)^{N-1}g(a+h)\\
        &\geq y(1+y)^{N-1}(g(a+h)-g(a))
    \end{align*}
    
    And so:
    \begin{align*}
         \frac{\sup_{x\in [a,b]} g(x+h)-g(x)}{\inf_{x\in [a,b]} g(x+h)-g(x)} \geq \frac{g(a+Nh + h)-g(a+Nh)}{g(a+h)-g(a)}\geq y(1+y)^{N-1}
    \end{align*}
\end{proof}

\begin{lemma}\label{exp_hops_somewhere}
Let $S$ defined via $g$ as usual be a proper scoring rule. 
Let $\epsilon>0,L>0$. Assume that $S$ has the following property: For every $f$ with Lipschitz constant $L$, we have $|p^*-f(p^*)|\leq \epsilon$ for the optimal report(s) $p^*$.
Let $\delta = \frac{\epsilon}{L+1}$.
Then for every $2\delta$ interval contained in $[0,1-3\epsilon+2\delta]$
, there is a $p$ in that interval such that ($p+2\delta\leq 1$ and)
\begin{equation*}
    g(p+2\delta)-g(p) \geq \frac{2L}{L+3}g(p).
\end{equation*}
\end{lemma}

(Note that this result doesn't assume $g>0$. However, note that the consequent of the lemma is vacuous if $g(p)\leq 0$ (since $g$ is monotone increasing.)

\begin{proof}
We shall show, equivalently, that for every interval of width $2\delta$ in $[2\delta, 1-3\epsilon+4\delta]$ \ec{should we add a note here that the top of this interval can be above 1 but that that's okay?}, there is some $p$ (in $[0,1]$) contained in the interval such that:
\begin{equation}\label{g_deriv_bound}
g(p)-g(p-2\delta)\geq\frac{2L}{L+3}g(p-2\delta)
\end{equation}

Given an interval of width $2\delta$ in $[2\delta,1-3\epsilon+4\delta]$
, write the interval as $[p_0-\delta,p_0+\delta]$, where $p_0 \in [3\delta, 1-3L\delta] = [3\delta, 1 - 3\epsilon + 3\delta] \subseteq [0,1]$.

Then, we construct $f$ as follows.

Let
\begin{align}
k_2 &\defeq p_0\left(1+\frac{1}{L}\right)\\
k_1 &\defeq k_2 - \frac{1}{L} = p_0 -\frac{1}{L}(1-p_0).\label{eq:def-k1}
\end{align}
Then consider 
\begin{equation*}
f(p)\defeq
\left\{\begin{array}{cl}
        1 & \text{if } p\leq k_1\\
        0 & \text{if } p\geq k_2\\
        \frac{k_2-p}{k_2-k_1} = L(k_2-p) & \text{if } k_1\leq p\leq k_2
        \end{array}\right.
\end{equation*}
for $p \in [0,1]$.

For $k_1=0.3,k_2=0.4,L=10$, this function looks as follows.\\
\begin{center}
\includegraphics[width=0.6\linewidth]{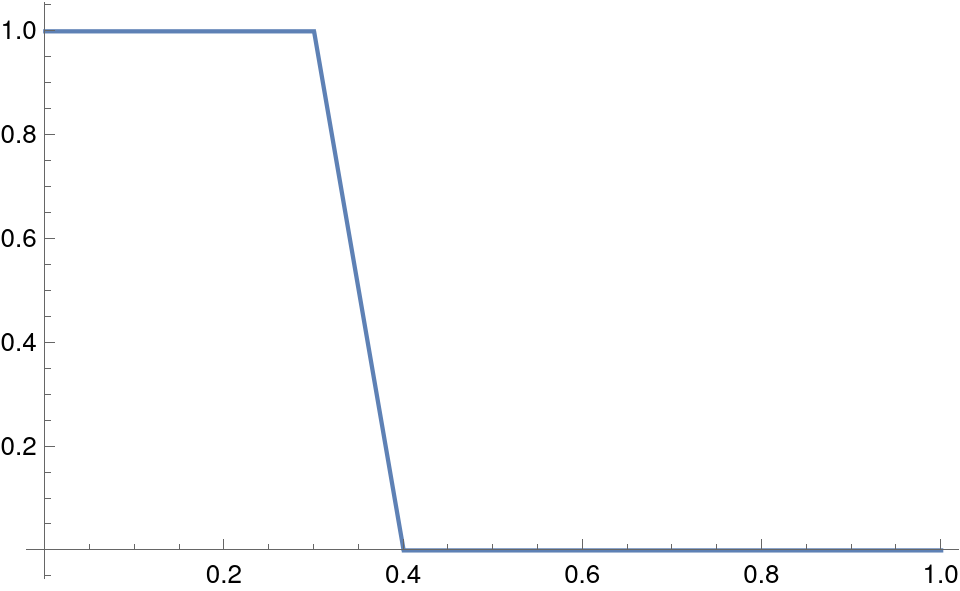}
\end{center}
Note that $f$ then has Lipschitz constant L. 

Moreover, note that $f$ has a unique fixed point, which occurs at $p_0 \in (k_1,k_2)$, since
\begin{equation*}
    f(p_0)=L(k_2-p_0)=L\left(p_0 + \frac{1}{L}p_0 - p_0\right) = p_0.
\end{equation*}
We can see that for any proper scoring rule, the optimal report under $f$ is in $[k_1,k_2]\cap [0,1]$.

Next we will show that for $p\in [k_1,k_2] \cap [0,1]$ to satisfy the bound $|f(p)-p|\leq \epsilon$, we must have that $p\in [p_0-\delta,p_0+\delta]$. To show this, observe that, if $p \in [k_1,k_2] \cap [0,1]$:
\begin{eqnarray*}
    f(p)-p &=& f(p)-f(p_0) - (p-p_0) +f(p_0)-p_0\\
    &=& -L(p-p_0) - (p-p_0) + f(p_0)-p_0\\
    &=& -(L+1)(p-p_0) + f(p_0)-p_0\\
    &\underset{p_0\text{ fixed point}}{=}& -(L+1)(p-p_0),
\end{eqnarray*}
So that, for $k_1\leq p \leq k_2$ we have $|f(p)-p|\leq\epsilon$ if and only if $|p-p_0|<\delta$. Thus, by assumption, we must have that the optimal report $p^*$ satisfies $p^* \in [p_0-\delta,p_0+\delta].$ We will show the claim of the theorem by showing that that $p=p^*$ satisfies equation (\ref{g_deriv_bound}).

First, we will check that $p^*-2\delta$ is in $[0,1]$, and is still on the steep section of the graph, i.e., is in $[k_1,k_2]$.

We have
\begin{equation*}
    p^* - 2\delta \geq p_0 - 3\delta \underset{\text{\Cref{eq:def-k1}}}{=} k_1 + \frac{1}{L}(1-p_0) - 3\delta \underset{p_0\leq 1-3L\delta}{\geq} k_1.
\end{equation*}
Also we chose $p_0$ to satisfy $p_0 - 3\delta \geq 0$, so that overall $p^*-2\delta \geq \max(k_1,0)$. Also, we must have that $\max(k_1,0) \leq p^* \leq \min (k_2,1)$.

Now we just use the optimality of $p^*$ and the definition of $f$ to get the result. $S(p^*,f(p^*)) \geq S(p^*-2\delta, f(p^*-2\delta))$, i.e.:
\begin{align*}
    G(p^*)+g(p^*)(f(p^*)-p^*) &\geq G(p^*-2\delta)+g(p^*-2\delta)(f(p^*-2\delta)-p^*+2\delta) \\
    &= G(p^*-2\delta)+g(p^*-2\delta)(f(p^*)+2L\delta-p^* +2\delta)
\end{align*}
Now, by the fact that $g$ is a subgradient of $G$, we have that $G(p^*)-G(p^*-2\delta)\leq 2\delta g(p^*)$. 
Thus, rearranging:
\begin{align*}
    2\delta g(p^*) +g(p^*)(f(p^*)-p^*) \geq g(p^*-2\delta)(f(p^*)+2L\delta - p^* + 2\delta)
\end{align*}
and so
\begin{align*}
    (f(p^*)-p^* + 2\delta)(g(p^*)-g(p^*-2\delta)) \geq 2\delta L g(p^*-2\delta).
\end{align*}
Thus, since $|f(p^*)-p^*|\leq\epsilon$
\begin{align*}
    g(p^*)-g(p^*-2\delta) &\geq g(p^*-2\delta)\frac{2L\delta}{\epsilon+2\delta}\\
    &= \frac{2L}{L+3}g(p^*-2\delta).
\end{align*}
%
%
%
\end{proof}
\begin{lemma}\label{exp_hops_everywhere}
    Let $g$ be any monotonically increasing nonnegative function with the property that on each interval of length $h$ fully contained in $[a,b]$ there is $p$ in that interval such that $g(p+h)-g(p)\geq yg(p)$. Then for all $p\in [a,b-h]$, $g(p+2h)-g(p)\geq yg(p)$, provided $g$ is defined on $[a,b+h]$.
\end{lemma}
\begin{proof}
    For any $p \in [a,b-h]$, we have that the interval $[p,p+h]$ must contain some $p^*$ with $g(p^*+h)-g(p^*) \geq yg(p^*)$. Then, since $p\leq p^*$, and $p^*+h \leq p+2h$, we have, by monotonicity
    \begin{equation*}
        g(p+2h)-g(p) \geq g(p^*+h)-g(p^*)\geq yg(p^*) \geq y g(p).
    \end{equation*}
\end{proof}

\begin{lemma}\label{lemma:symmetry-for-bounds}
    Let $S(p,q)=g(p)(q-p)+G(p)$ be a (strictly) proper scoring rule, and $f: [0,1]\rightarrow[0,1]$ be $L_f$-Lipschitz. Let
    \begin{align*}
        \tilde g(p) &\defeq - g(1-p)\\
        \tilde G(p) &\defeq G(1-p)\\
        \tilde S(p,q) &\defeq \tilde g(p)(q-p)+\tilde G(p)\\
        \tilde f(p) &\defeq 1 - f(1-p)
    \end{align*}
    Then $\tilde S$ is a (strictly) proper scoring rule, $\tilde f$ is $L_f$-Lipschitz and
    \begin{equation*}
        \tilde S(1-p,\tilde f(1-p)) = S(p,f(p)).
    \end{equation*}
\end{lemma}
\begin{proof}
First, we show that $\tilde S$ is a (strictly) proper scoring rule, by verifying that it conforms to the \citeauthor{gneiting2007strictly} characterization. From the form of $\tilde G$, we can see that $\tilde G$ is (strictly) convex iff $G$ is. It remains only to check that $\tilde g$ is a subderivative of $g$. Then, we have
\begin{equation*}
    \tilde G(p)- \tilde G(q) = G(1-p)-G(1-q) \geq g(1-q)(1-p-1+q) = \tilde g (q) (p-q).
\end{equation*}
as required. By inspection, $\tilde f$ is $L_f$-Lipschitz.

Finally,
\begin{align*}
    \tilde S(1-p,\tilde f(1-p)) &= \tilde g(1-p)(\tilde f(1-p)-(1-p))+\tilde G(1-p)\\
    &= (-g(p))(1-f(p) - (1-p))+G(p)\\
    &= g(p)(f(p) - p) +G(p)\\
    &= S(p,f(p)).
\end{align*}
\end{proof}

\begin{lemma}\label{lemma:g-ratios}
Suppose $S$ is a proper scoring rule defined via $g$ s.t.\ for some $\epsilon, L_f > 0$ we have that whenever $f$ is $L_f$-Lipschitz, the optimal report $p^*$ satisfies $|f(p^*) - p^*| < \epsilon$. Let $\delta = \frac{\epsilon}{L_f+1}$. Further consider $p_l,p_h$ s.t.\ $3\epsilon-3\delta  \leq p_l \leq p_h \leq 1-3\epsilon-\delta$. Then we have:
\begin{equation*}
        \frac{\sup_{x\in [p_l,p_h]} |g(x+4\delta)-g(x)|}{\inf_{x\in [p_l,p_h]} |g(x+4\delta)-g(x)|} \geq \frac{2L_f}{L_f+3}\left(3\frac{L_f+1}{L_f+3}\right)^{(L_f+1)(p_h-p_l)/(8\epsilon) -5/2}.
\end{equation*}
\end{lemma}


\begin{proof}
We will, show, equivalently, that for $3\epsilon -\delta \leq p_l \leq p_h \leq 1 - 3\epsilon + \delta$:

\begin{equation}\label{alt_form}
        \frac{\sup_{x\in [p_l,p_h]} |g(x+2\delta)-g(x-2\delta)|}{\inf_{x\in [p_l,p_h]} |g(x+2\delta)-g(x-2\delta)|} \geq \frac{2L_f}{L_f+3}\left(3\frac{L_f+1}{L_f+3}\right)^{(L_f+1)(p_h-p_l)/(8\epsilon) -5/2}.
\end{equation}

Consider first the case where $g((p_l+p_h)/2)<0$. Then consider $\tilde g$ as specified by \Cref{lemma:symmetry-for-bounds}. Note that from \Cref{lemma:symmetry-for-bounds} it follows that $\tilde S,\tilde g$ satisfy the claim of the Theorem in the form of equation (\ref{alt_form}) for $[p_l',p_h']\defeq[1-p_h,1-p_l]$ if and only if $S,g$ satisfy the claim of the Theorem for $[p_l,p_h]$:

\begin{align*}
        \frac{\sup_{x\in [p_l,p_h]} |g(x+2\delta)-g(x-2\delta)|}{\inf_{x\in [p_l,p_h]} |g(x+2\delta)-g(x-2\delta)|}
        &= \frac{\sup_{x\in [p_l,p_h]} |\tilde g(1-x+2\delta)-\tilde g(1-x-2\delta)|}{\inf_{x\in [p_l,p_h]} |\tilde g(1-x+2\delta)-\tilde g(1-x-2\delta)|}\\
        &= \frac{\sup_{y\in [1-p_h,1-p_l]} |\tilde g(y+2\delta)-\tilde g(y-2\delta)|}{\inf_{y\in [1-p_h,1-p_l]} |\tilde g(y+2\delta)-\tilde g(y-2\delta)|}.\\
\end{align*}

Note further that $\tilde g((p_l'+p_h')/2)>0$. Thus, for our proof we can assume WLOG $g((p_l+p_h)/2)>0$ and thus by monotonicity $g(x)>0$ for $x>(p_l+p_h)/2$.

Note that $p_h\leq 1-3\epsilon +\delta \leq1-3\epsilon+2\delta $. So, by \Cref{exp_hops_somewhere}, we have that in every $2\delta$ interval contained in $[(p_l+p_h)/2,p_h]$ there is a $p$ such that
\begin{equation*}
    g(p+2\delta)-g(p) \geq \frac{2L_f}{L_f+3}g(p).
\end{equation*}
Hence, by \Cref{exp_hops_everywhere}, we have that for all $p\in [(p_l+p_h)/2,p_h-2\delta]$,
\begin{equation*}
    g(p+4\delta)-g(p) \geq \frac{2L_f}{L_f+3}g(p)
\end{equation*}
since $p_h+2\delta \leq 1-3\epsilon + \delta + 2\delta \leq 1$.
Thus, by \Cref{compounding_hops},
\begin{align*}
&\phantom{ = } \frac{\sup_{x\in [p_l,p_h]} |g(x+2\delta)-g(x-2\delta)|}{\inf_{x\in [p_l,p_h]} |g(x+2\delta)-g(x-2\delta)|}\\
&=
\frac{\sup_{x\in [p_l-2\delta,p_h-2\delta]} |g(x+4\delta)-g(x)|}{\inf_{x\in [p_l-2\delta,p_h-2\delta]} |g(x+4\delta)-g(x)|}\\
        &\geq\frac{\sup_{x\in [(p_l+p_h)/2,p_h-2\delta]} g(x+4\delta)-g(x)}{\inf_{x\in [(p_l+p_h)/2,p_h-2\delta]} g(x+4\delta)-g(x)} \geq \frac{2L_f}{L_f+3}\left(1+\frac{2L_f}{L_f+3}\right)^{\lfloor (p_h-p_l)/(8\delta) - 1/2 \rfloor-1}.
\end{align*}
\end{proof}

\co{TODO: Should do some sanity checks. E.g., if $\epsilon=1$, the bound should be trivial.}

\needexponential*

\begin{proof}
    Follows from \Cref{lemma:bounding-S-ratios-bounding-g-ratios,lemma:g-ratios}.
\end{proof}

\subsection{Proof of Theorem\ \ref{thm:no_higher_dim_bound}}

We'll first need a lemma that we can find a section of an isoline of sufficient length that doesn't turn too much:
\begin{lemma}\label{lemma:isoline_section}
Let $\Pset$ be the probability simplex in $\mathbb{R}^3$ (an equilateral triangle with side length $\sqrt{2}$ lying in a plane embedded in $\mathbb{R}^3$). Let $G:\Pset\rightarrow \mathbb{R}$ be strictly convex, with subgradient $g$ (with entries summing to 0). Let $r=\frac{\sqrt{6}}{12}$, and $0<l<r$.  Then we can find a section of an isoline of $G$, $\gamma$, with the following properties:
\begin{enumerate}
    \item $\gamma$ has length $l$.
    \item Let $\p$ be one endpoint of $\gamma$, and $\vec{n}$ be the unit vector parallel to $g(\p)$, i.e. $\vec n = g(\p)/\norm{g(\p)}$. Then, $\p-2r\vec{n} \in \Pset$, and $G(\p-2r\vec{n}) \leq G(\p)$.
    \item For all $\p' \in \gamma$, the angle $\theta$ between $g(\p)$ and $g(\p')$ satisfies $\theta \leq \frac{2l}{r}$.
    \item Each point on $\gamma$ has distance at least $r - l$ from the boundary of the simplex.
\end{enumerate}
Moreover, let $\gamma$ be an isoline section with the above properties. Let $\p$ and $\vec{n}$ be defined as above, and $\vec{t}$ a unit vector perpendicular to $\vec{n}$ (and to $\vec{1}$). Let $\q$ be the other endpoint of $\gamma$. Then $|\vec{t}^\top(\q-\p)| \geq l\left(1-\frac{2l}{r}\right)$. I.e. the length of $\gamma$ in the direction orthogonal to $\vec{n}$ is at least $l\left(1-\frac{2l}{r}\right)$.
\end{lemma}
\begin{proof}
    Note that each isoline $\{\x\in\Pset:G(\x)=y\}$ forms part of the boundary of the set $\{\x\in\Pset:G(\x)\leq y\}$, which by strict convexity of $G$ is a convex set. Call this the \textit{enclosed set} of the isoline. Then, at each point $\p$ on an isoline, there is at least one supporting line to the isoline, i.e., a straight line that touches the isoline but does not contain any of the interior points of the enclosed set. Moreover, $g(\p)$ is always perpendicular to a supporting line to the isoline through $\p$, and points out of the enclosed set. 
    
    We will proceed by finding an isoline tangent to and enclosing a circle at the center of $\Pset$, and then arguing that a section of this isoline has the desired properties.

    First, consider the circle of radius $r$ at the center of $\Pset$. Let $\p$ be a point on the boundary of this circle at which $G$ is maximal, and consider the isoline of $G$ through $\p$.

    \begin{center}
    \begin{tikzpicture}
        \coordinate (A) at (0,0);
        \coordinate (B) at (6,0);
        \coordinate (F) at (3,3.5);
        \coordinate (G) at (1,0);
        \tkzDefTriangle[equilateral](A,B);
        \tkzGetPoint{C};
        \coordinate (H) at ($0.6*(B)+0.4*(C)$);
        \tkzDefTriangleCenter[centroid](A,B,C);
        \tkzGetPoint{O};
        \coordinate (D) at ($(O)+6*(0,0.144)$);
        \tkzDrawPolygons(A,B,C);
        \tkzDrawPoints(O);
        \tkzLabelPoint[above right](O){$O$};
        \tkzDrawCircle(O,D);
        \tkzDefPointBy[rotation= center O angle 180](D);
        \tkzGetPoint{E};
        \draw [dashed] (O) -- (E) node [midway, left] {$r$};
        \tkzDefPointBy[rotation= center O angle 15](D);
        \tkzGetPoint{p};
        \tkzDrawPoints(p);
        \tkzLabelPoint(p){$\p$};
        \tkzDefPointBy[rotation= center O angle 15](F);
        \tkzGetPoint{n};
        \draw [->] (p)--(n)  node [midway, above right] {$g(\p)$};
        \draw (G) to [out = 80, in = 195] (p) to [out = 15, in = 150] (H);
    \end{tikzpicture}
    \end{center}

    Observe that, by construction, the enclosed set of this isoline contains the circle. Moreover, the isoline is tangent to the circle at $\p$. Note further that the tangent line to the circle at $\p$ must be the unique supporting line to the isoline through $\p$, so that $g(\p)$ is perpendicular to this tangent. As a consequence, we know that if $\vec{n} = \frac{g(\p)}{\norm{g(\p)}}$, then $\p - 2r\vec{n}$ is the point on the opposite side of the circle. Thus, we must have $G(\p-2r\vec{n}) \leq G(\p)$, as required.

    Now, we have two cases: either the isoline stays within the interior of the simplex, or it reaches the boundary of the simplex.
    In the former case, the total length of the isoline is at least the circumference of the circle, i.e., $2\pi r > l$. Thus, we may choose a section (with two distinct endpoints) of the isoline, $\gamma$, with length $l$ and endpoint $\p$.

    Now, note that the distance from the centre of the simplex to the (nearest point on the) boundary is $\frac{\sqrt{6}}{6} = 2r$. Hence, the minimum distance from the circle to the boundary of the simplex is at least the $2r - r = r > l$. Thus, in the latter case, the isoline must have a connected section starting at $\p$ of length $l$. Call this section $\gamma$.

    In either case, $\gamma$ has distance at least $r-l$ from the boundary of the simplex.

    Now, we will bound the change in the angle of supporting lines, moving along $\gamma$. WLOG assume $\p$ is the anticlockwise-most point of $\gamma$.

    \begin{center}
    \begin{tikzpicture}
        \coordinate (O) at (0,0);
        \coordinate (P) at (0,7);
        \coordinate (P+) at (1,7);
        \coordinate (P-) at (-1,7);
        \coordinate (A) at (5,6);
        \coordinate (T0) at (P);
        \coordinate (T1) at ($(P)+(6,0)$);
        \tkzDrawArc[R, color = black](O,7)(0,90);
        \tkzDrawPoints(P, O,A);
        \tkzLabelPoint(O){$O$};
        \tkzLabelPoint[left](P){$\p$};
        \draw (P) to [out = 0, in = 150]  node [midway, below right] {$\gamma$} (A) node [above right] {$\q$};
        \draw [dashed] (T0) -- (T1);
        \draw [dashed] (O) -- (P) node [midway, left] {$r$};
        \tkzMarkRightAngle(T1,P,O);
        \tkzDefPointBy[rotation= center O angle -75](P);
        \tkzGetPoint{C};
        \tkzDrawPoints(C);
        \tkzLabelPoint(C){$C$};
        \tkzDefPointBy[rotation= center O angle -75](P+);
        \tkzGetPoint{C+};
        \tkzDefPointBy[rotation= center O angle -75](P-);
        \tkzGetPoint{C-};
        \tkzInterLL(C-,C+)(P,P+);
        \tkzGetPoint{B};
        \tkzDrawLines[add = 0.3 and 0.1](C,B);
        \tkzDrawPoints(B);
        \tkzLabelPoint(B){$B$};
        \draw [dashed] (O) -- (C)  node [midway, below right] {$r$};
        \tkzMarkAngle(C,O,P);
        \tkzLabelAngle[above right](C,O,P){$\theta$};
        \tkzMarkRightAngle(O,C,B);
        \tkzDefMidPoint(P,B);
        \tkzGetPoint{l1};
        \tkzLabelPoint[above](l1){$l$};
        \tkzDefMidPoint(C,B);
        \tkzGetPoint{l2};
        \tkzLabelPoint[right](l2){$l$};
        \draw [->] (P)--($(P)+(0,1)$) node [midway, right] {$g(\p)$};
        
    \end{tikzpicture}
    \end{center}

    Let the center of the circle be $O$. Consider the tangent to the circle at the point $C$, where $C$ is such that $OC$ is at an angle of $\theta$ from the line from $O$ to $\p$. Let the intersection of the tangents through $C$ and $\p$ be $B$. Let $\theta$ be such that the line from $\p$ to $B$ has length $l$. Note that $\theta<\pi/2$, since $l<r$. Let $\q$ be the other (clockwise-most) endpoint of $\gamma$.

    Note then that since the isoline must lie below the line $\p B$, and the length of $\gamma$ is $l$, $\gamma$ never crosses the line $BC$.

    Meanwhile, consider a point $\p'$ on $\gamma$. Note that since the enclosed set of the isoline contains the circle, no supporting line to the isoline through $\p'$ contains interior points of the circle. The angle of such a line must then lie between the angles of the tangents at $\p$ and $C$ (ie, is at most as steep as $BC$, on the diagram), and so makes angle at most $\theta$ with $\p B$. Since $g$ points out of the enclosed set of the isoline, orthogonal to its supporting lines, the difference in angle between $g(\p')$ and $g(\p)$ is then at most $\theta$.

    Therefore, we have that:
    \begin{equation*}
        \frac{\theta}{2} \leq \tan\left(\frac{\theta}{2}\right) = \frac{l}{r}
    \end{equation*}
    since $\tan x \geq x$ for $x \in [0,\pi/2)$. Thus, $\theta \leq \frac{2l}{r}$, as required. We have now established that $\gamma$ has the stated properties.

    We need then only check the final part of the statement, i.e., that $\gamma$ with these properties has sufficient length in the direction orthogonal to $\vec{n}$. Let $\vec{t}$ be as in the statement of the lemma, i.e., parallel to the supporting line through $\p$. We have by convexity that $\gamma$ lies within the triangle defined by supporting lines to the isoline at $\p$ and $\q$ and the straight line from $\p$ to $\q$. Hence, since $\theta<\pi/2$, $\gamma$ lies entirely within the triangle defined by the supporting line at $\p$ parallel to $\vec{t}$, the line through $\q$ parallel to $\vec{n}$, and the line from $\p$ to $\q$, as depicted in the diagram below.
        \begin{center}
    \begin{tikzpicture}
        \coordinate (P) at (0,1);
        \coordinate (Q) at (5,0);
        \coordinate (R) at (5,1);
        \tkzDrawPoints(P,Q);
        \tkzLabelPoint[left](P){$\p$};
        \draw (P) to [out = 0, in = 150]  node [midway, below right] {$\gamma$} (Q) node [above right] {$\q$};
        \draw [dashed] (P) -- (R)--(Q)--(P);
        \tkzMarkRightAngle(P,R,Q);

        \node (compass) at (-2,0) {};
        \draw [->](compass.north)--(-2,1) node [midway, left] {$\vec{n}$};
        \draw [->](compass.east)--(-1,0) node [midway, below] {$\vec{t}$};
        
    \end{tikzpicture}
    \end{center}
    Moreover, the maximum possible length of a convex path, within this triangle, from $\p$ to $\q$ is just the combined length of the two shorter sides, i.e. $|\vec{n}^\top(\p-\q)| + |\vec{t}^\top(\p-\q)|$. Thus, $\left|\vec{t}^\top(\p-\q)\right| \geq l - \left|\vec{n}^\top(\p-\q)\right|$. Then, note that the straight line from $\p$ to $\q$ makes angle  at most $\theta$ with the line parallel to $\vec{t}$, since its angle must lie between the angle of supporting lines at $\p$ and $\q$.
    
    Thus,
    \begin{equation*}
        |\vec{n}^\top(\p-\q)| \leq \sin\theta \norm{\p-\q} \leq l \sin\theta \leq l\theta.
    \end{equation*}

    Hence, we have
    \begin{equation*}
        |\vec{t}^\top(\p-\q)| \geq l - l\theta = l\left(1-\frac{2l}{r}\right)
    \end{equation*}
    and we are done.
\end{proof}
Now, our main result: We can't get arbitrarily good bounds for fixed Lipschitz constant, and the bound one can at best get scales linearly with $L_f$ in the limit $L_f \rightarrow 0$.

%
\impossibility*

\begin{proof}
Let $g$ and $G$ be as in the Gneiting and Raftery characterization of $S$. Let $\lambda = \min(L_f,2)$.

\jt{is it possible to illustrate this somehow?}

We will proceed as follows:
\begin{itemize}
    \item Find an isoline of $G$ on which the angle of $g(\p)$ doesn't change much. On this isoline, we are then able to move along the isoline without $g(\p)$ changing much in the direction of movement, and hence without $g(\p)^\top  \p$ changing much.
    \item Construct a $\lambda$-Lipschitz (and hence $L_f$-Lipschitz) function $f$ with fixed point $\p_0$ such that as we move sideways along the isoline, $f(\p)$ moves upwards, incentivising us to misrepresent in the direction of the isoline.
    \item We will then show that for a point $\q$, with $\norm{f(\q)-\q} \geq \epsilon$, reporting $\q$ gives higher score than any point $\p$ for which $\norm{f(\p)-\p} <\epsilon$ (for $\epsilon$ which we will choose).
\end{itemize}

Let $\theta = \arctan(\lambda/4)\leq\lambda/4\leq 1/2$. Note that then $\theta \sim L_f/4$ as $L_f \rightarrow 0$.

Let $r = \frac{\sqrt{6}}{12}$. Then, let $\gamma$ be an isoline satisfying the properties of Lemma \ref{lemma:isoline_section} with $l = \frac{\theta r}{2}$. Let the end points of $\gamma$ be $\p_1$ and $\q$ (chosen such that $\q$ is the same end as $\q$ in the statement of the Lemma), $\vec{n}\defeq \frac{g(\p_1)}{\norm{g(\p_1)}}$, and $\vec{t}$ a unit vector orthogonal to both $\vec{n}$ and $\vec{1}$. Then we have, in particular:
\begin{enumerate}[label = (P\arabic*)]
    \item $\p_1-2r\vec{n} \in \Pset$ and $G(\p_1-2r\vec{n}) \leq G(\p_1)$.\label{p:can_drop}
    \item For all $\p \in \gamma$, the angle between $g(\p)$ and $g(\p_1)$ (equivalently, $\vec{n}$) is at most $\frac{2l}{r} = \theta$.\label{p:small_angle}
    \item Each point on $\gamma$ has distance at least $r-l \geq 3l$ from the boundary of the simplex.\label{p:margin}
    \item $|t^\top(\q-\p_1)| \geq l(1-\frac{2l}{r}) = l(1-\theta)$.\label{p:width}
\end{enumerate}

Let $\epsilon = \frac{1}{2}|t^\top(\q-\p_1)|$. We have, by \ref{p:width}, $\epsilon \geq l(1-\theta)/2\geq l/4 > 0$. Note that $l(1-\theta)/2 \leq \epsilon \leq l/2$, and $l\theta = o(L_f)$, and so $\epsilon \sim \frac{1}{2}l = \frac{1}{4}\theta r \sim \frac{1}{16}L_f r=\frac{\sqrt{6}}{192}L_f$ as $L_f \rightarrow 0$.

Let $\p_0 = \p_1 - \epsilon\lambda\vec{n}$. Note that since $\epsilon\lambda \leq 2\epsilon \leq l$, we have by \ref{p:margin} that $\p_0$ is within the simplex.

\begin{center}
\begin{tikzpicture}[
    tangent/.style={
        decoration={
            markings,
            mark=
                at position #1
                with
                {
                    \coordinate (tangent point-\pgfkeysvalueof{/pgf/decoration/mark info/sequence number}) at (0pt,0pt);
                    \coordinate (tangent unit vector-\pgfkeysvalueof{/pgf/decoration/mark info/sequence number}) at (1,0pt);
                    \coordinate (tangent orthogonal unit vector-\pgfkeysvalueof{/pgf/decoration/mark info/sequence number}) at (0pt,1);
                }
        },
        postaction=decorate
    },
    use tangent/.style={
        shift=(tangent point-#1),
        x=(tangent unit vector-#1),
        y=(tangent orthogonal unit vector-#1)
    },
    use tangent/.default=1
]
    \coordinate (B) at (5,0.5); 
    \coordinate (p1) at (0,1); 
    \coordinate (p0) at (0,0); 

    \draw [name path = gamma, tangent = 1, tangent = 0.6] (p1) node [above left] {$\p_1$} to[out = 0, in = 165] (B); 

    \coordinate [use tangent] (q); 
    \draw (q) node [right] {$\q$};
    \coordinate (q-) at (q |- p0); 
    \draw (q-) node [below] {$\p_0+2\epsilon \vec{t}$};

    \draw [dashed] (q-) -- (q);
    \draw [dashed] (p1) -- (p0) node [below] {$\p_0$};
    \draw [dashed] (p0)--(q-) node [midway, below] {2$\epsilon$};

    \tkzDrawPoints(p1,p0,q,q-);

    \node (compass) at (8,1) {};
    \draw [->](compass.north)--(8,2) node [midway, left] {$\vec{n}$};
    \draw [->](compass.east)--(9,1) node [midway, below] {$\vec{t}$};

    \draw [use tangent, ->] (0,0.2) -- (0,1.3) node [midway, right]{$g(\q)$};
    \draw [->] ($(p1)+(0,0.2)$) -- ($(p1)+(0,1.5)$) node [midway, right]{$g(\p_1)$};
\end{tikzpicture}
\end{center}

By construction, there is a supporting line to $\gamma$, parallel to $\vec{t}$, through $\p_1$. Thus, $(\q-\p_1)^\top \vec{n} \leq 0$.

Now, let $f(\p) = \p_0 + \lambda\min(|\vec{t}^\top (\p-\p_0)|,2\epsilon)\vec{n}$. Note that the image of $f$ is the line segment $[\p_0,\p_0 + 2\epsilon\lambda\vec{n}]$, which has maximum distance $\lambda\epsilon \leq l$ from $\gamma$, and hence by \ref{p:margin} is entirely within the probability simplex. Also, $f$ has Lipschitz constant $\lambda \leq L_f$.

Then, if $\norm{f(\p)-\p}\leq\epsilon$, we must have
\begin{equation*}
    \epsilon \geq |\vec{t}^\top (f(\p)-\p)| = |\vec{t}^\top (\p_0-\p)|
\end{equation*}
and hence $f(\p)$ must in fact lie in the line segment $[\p_0,\p_0 + \epsilon\lambda\vec{n}] = [\p_0,\p_1]$. Moreover, $\norm{f(\q)-q} \geq 2\epsilon > \epsilon$.

Meanwhile, we have by convexity and \ref{p:can_drop} that for $\p \in [\p_1-2r\vec{n}, \p_1]$, $G(\p) \leq \max(G(\p_1-2r\vec{n}),G(\p_1)) = G(\p_1)$. Hence, since $\lambda\epsilon < 2r$, the maximum of $G$ on $[\p_0,\p_1]$ is $G(\p_1)$.

Therefore, whenever $\norm{f(\p)-\p}\leq\epsilon$:

\begin{equation*}
    S(\p,f(\p))\leq S(f(\p),f(\p)) = G(f(\p)) \leq G(\p_1)
\end{equation*}
that is, the maximum achievable score is at most the score of honestly reporting $\p_1$.

We will now show that the score of reporting $\q$ is greater than $G(\p_1)$.

First, we have that
\begin{align*}
    S(\q,f(\q))&= g(\q)^\top (f(\q)-\q) + G(\q) &\text{(Gn\&Raf)}\\
    &= g(\q)^\top (\p_0+2\epsilon\lambda \vec{n} - \q) + G(\p_1) &\text{(Def. of } f, \q)
\end{align*}

It is left to show that the left summand is positive. We have that
\begin{align*}
    & g(\q)^\top (\p_0+2\epsilon\lambda\vec{n} -\q)\\
    &= (2\epsilon\lambda+(\p_0-\q)^\top \vec{n})g(\q)^\top \vec{n} + ((\p_0-\q)^\top \vec{t})g(\q)^\top \vec{t}\\
    &\geq  \epsilon\lambda g(\q)^\top \vec{n} - 2\epsilon |g(\q)^\top \vec{t}| &\text{(Def. of }\q,\p_0)\\
    &\geq \norm{g(\q)}\epsilon(\lambda\cos\theta - 2\sin\theta) &\text{(By \ref{p:small_angle}})\\
    &=  2\norm{g(\q)}\epsilon \cos(\theta)(\lambda/2-\tan\theta)\\
    &\geq \norm{g(\q)}\epsilon\cos(\theta) \lambda/2 > 0 &\text{(Choice of }\theta)
\end{align*}

\end{proof}

\section{Preferences between different fixed points}
\label{preferences-between-fps}

\begin{proposition}
\label{prop:preferences-between-fps}
    Let $F=\{\p\colon f(\p)=\p\}$ be a set of fixed points of $f$. Let $\p\in F$ such that $\p$ is the convex combination of elements of $F-\{\p\}$. (In other words, $\p$ is in the interior of the convex hull of $F$). Then if $S$ is strictly proper, there exists a $\p^*\in F$ s.t.\ $S(p^*,f(p^*))> S(p,f(p))$. Thus, $\argmax_{p\in F} S(p,f(p))$ is a subset of the extreme points of $F$.
\end{proposition}

This follows directly from the convexity of the expected score under honest reporting as per \Cref{theorem:gneiting-raftery}, but for completeness we provide a detailed proof.

\begin{proof}
Let $\p = \sum_{i=1}^k c_i\p_i$ for $c_i\in [0,1]$ with $\sum_{i=1}^k c_i=1$ and $\p_i\in F-{\p}$. Then
\begin{eqnarray*}
S(\p,f(\p)) &=& g(\p)(f(\p)-\p)+G(\p)\\
& \underset{\text{$\p$ fixed point}}{=} & G(\p)\\
&= & G\left(\sum_{i=1}^k c_i\p_i\right)\\
&\underset{\text{G strictly convex}}{<} & \sum_{i=1}^k c_i G(\p_i)\\
&\underset{\text{$\p_i$ fixed point}}{=}& \sum_{i=1}^k c_i (g(\p_i)(f(\p_i)-\p_i) +G(\p_i))\\
&=& \sum_{i=1}^k c_i S(\p_i,f(\p_i)).
\end{eqnarray*}
Now for the average of the $S(\p_i,f(\p_i))$ to be greater than $S(\p,f(\p))$, at least one of the $S(\p_i,f(\p_i))$ must be greater than $S(\p,f(\p))$.
\end{proof}

\co{LOW PRIORITY: Think about the symmetric case again. Maybe one can prove something interesting for the many-outcome case after all. But given our results the two-outcome case is most important.}

\co{LOW PRIORITY: Prove the following.

Let $F$ be a set of fixed points of $f$ and let $\p^*$ be an extreme point of $F$, i.e., an element of $F$ that is not the convex combination of other elements of $F$. Then there exists a strictly proper scoring rule $S$ such that $\argmax_{\p\in F} S(\p,f(\p))=\p^*$.}

\co{There are alternative versions of the above, but I think the above is the simplest. If $G$ was allowed to be weakly convex, then $G$ could just be $v^\top p$ s.t.\ $G$ could be any utility function. But since $G$ must be strictly convex, this can only be approximately true. Which is still interesting, but not as nice to state.}

\section{Additional experimental results}
\label{appendix:experimental-results}

\subsection{Two outcomes}


%
%

\begin{figure}[H]
\centering
\includegraphics[width=0.7\columnwidth]{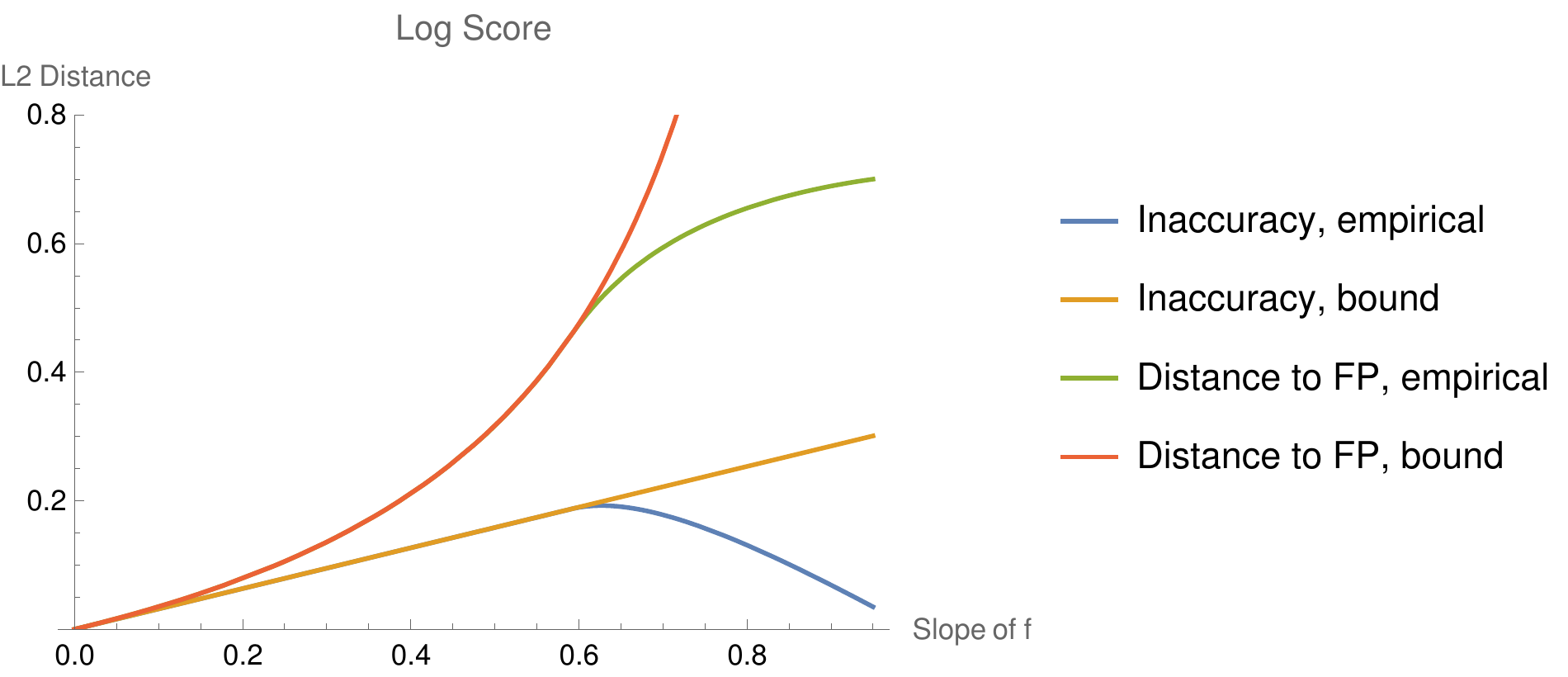}
\caption{Maximal inaccuracy and maximal distance to fixed point (FP) of optimal predictions, depending on the slope of \(f\), according to our simulation and our theoretical bound.}
\label{fig:max-distance-log-score-new}
\end{figure}

\begin{figure}[H]
\centering
\includegraphics[width=0.49
\textwidth]{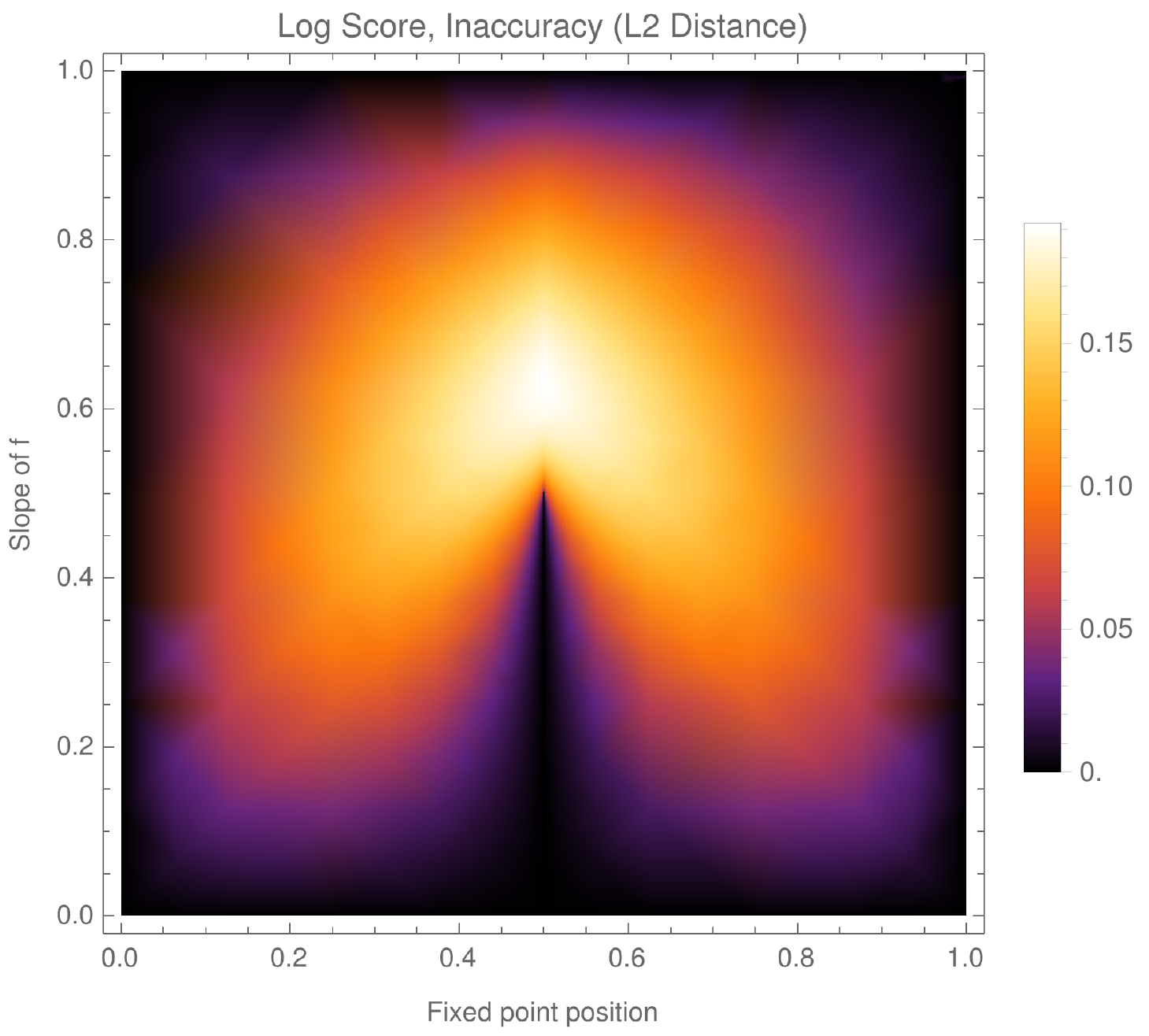}
\caption{Heatmap of the L2 inaccuracy  of optimal predictions, depending on fixed point position and slope of \(f\), for the logarithmic scoring rule.}
\label{fig:density-plot-log-l2-inaccuracy}
\end{figure}

\begin{figure}[H]
\centering
\includegraphics[width=0.49
\textwidth]{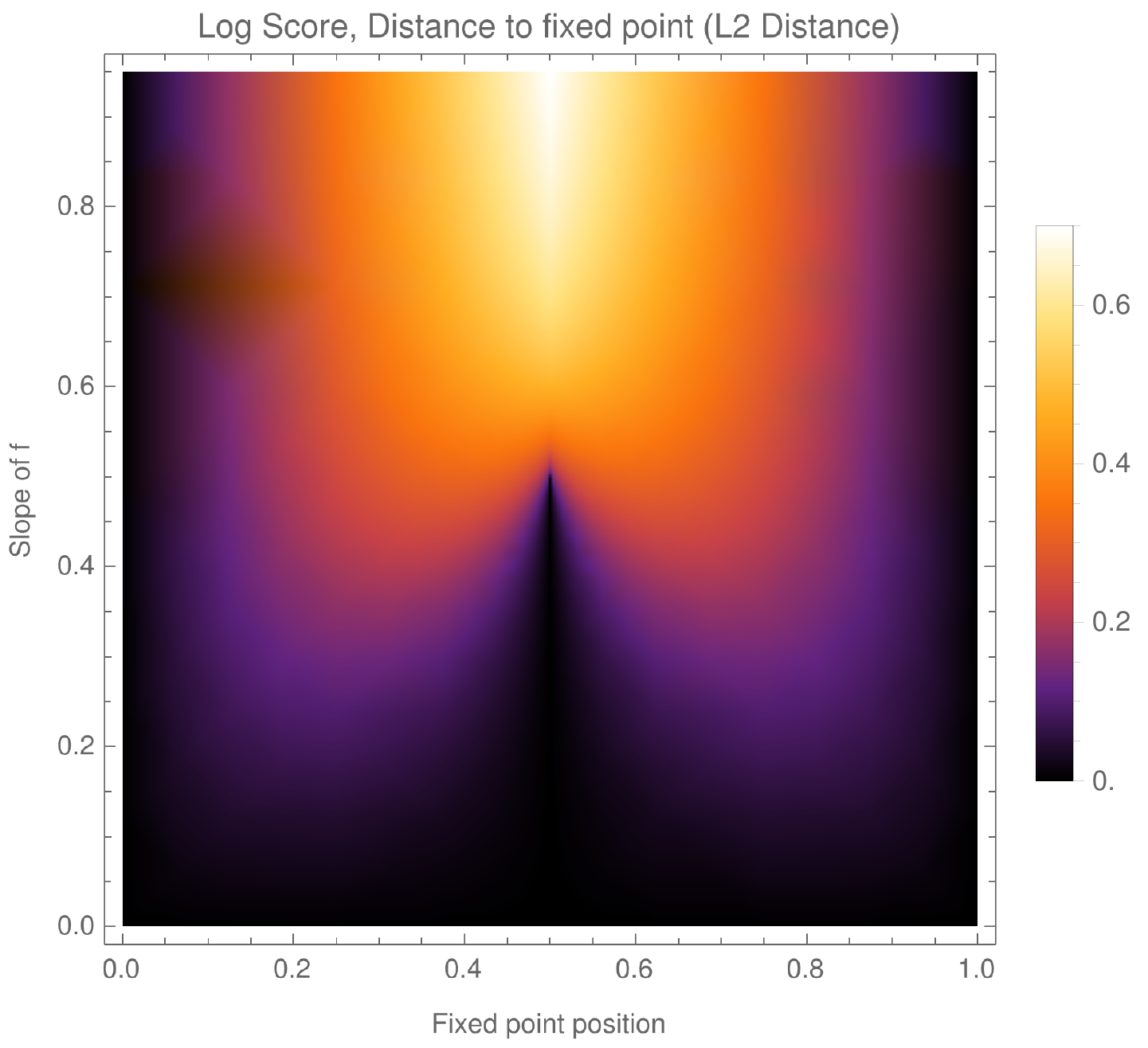}
\caption{Heatmap of the L2 distance to the fixed point of optimal predictions, depending on fixed point position and slope of \(f\), for the logarithmic scoring rule.}
\label{fig:density-plot-log-l2-disttofp}
\end{figure}

\Cref{fig:max-distance-log-score-new,fig:density-plot-log-l2-disttofp,fig:density-plot-log-l2-inaccuracy} give the same graphs that we give for the Brier scoring rule in the main text.
\jt{We have to give the theoretical bound here! (since we don't give in main text)}

Here the bounds for the log scoring rule are obtained as follows. First, note that for the log scoring rule we have that $g(p)=(\log p_i - \nicefrac{1}{2} (\log p_1+\log p_2 ))_i$. So, $\Vert g(\p)\Vert =\frac{| \log (p_1)-\log (p_2)| }{\sqrt{2}}$ and $Dg(p)=\begin{pmatrix}
\frac{1}{2p_1} & -\frac{1}{2p_2}\\
-\frac{1}{2p_1} &   \frac{1}{2p_2}\\
\end{pmatrix}.$ The eigenvalue of this on the tangent space is $1/(2p_1p_2)$. Thus, since $Dg$ is symmetric, $DG(\p)\succeq 1/(2p_1p_2)$. By \Cref{theorem:Caspar-approx-fix-point},
$\Vert \p-f(\p) \Vert \leq L_f \Vert g(p) \Vert 2p_1p_2= \sqrt{2} L_f p_1p_2 | \log (p_1)-\log (p_2)|$. Numerically this bound seems to be maximized at $p=0.824$ so that we get a bound $\Vert \p-f(\p) \Vert \leq 0.316 L_f$. Similarly, by \Cref{thm:distance-to-fp}, $\Vert \p - \p^* \Vert \leq 0.316 L_f / (1-L_f)$.

For the logarithmic scoring rule, we also give the same plot for the absolute distance between the logits or log odds of the two probabilities (logit distance), see \Cref{fig:density-plot-log-score-odds}. It is defined as \(d(\p,\p'):=|\sigma^{-1}(\p)-\sigma^{-1}(\p')|\), where
\(\sigma^{-1}(\p):=\log \frac{p_1}{p_2}\) is the logit of \(\p\) (or the inverse sigmoid transform). If probabilities are close to \(0\) or \(1\), then L2 distance will always evaluate to very small distances. In contrast, the logit distance depends on order of magnitude differences between probabilities, which may be the more useful quantity.

\begin{figure}[H]
\centering
\includegraphics[width=0.5\textwidth]{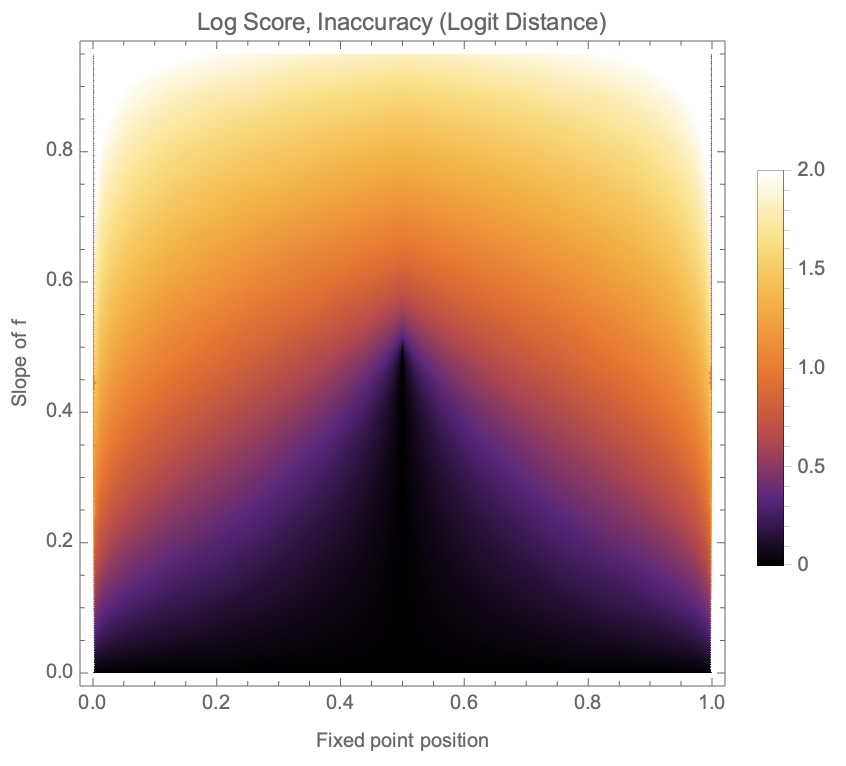}
\caption{Heatmap of logit distance inaccuracy of optimal predictions for the log scoring rule.}
\label{fig:density-plot-log-score-odds}
\end{figure}

We can see that inaccuracy remains high in logit space for fixed points close to \(0\) and \(1\). We don't plot logit distances for the quadratic score, since for that score, optimal predictions often take values close to or equal to \(\{(0,1),(1,0)\}\) (even if neither \(f(\p)\) nor \(\p^*\) lie in \(\{(0,1),(1,0)\}\)), so the corresponding distances become very large or infinite. The fact that logit distances are bounded for the log score is an advantage of that scoring rule.

\subsection{Many outcomes}

\subsubsection{Inaccuracy and ditance to fixed point are strongly correlated}

Throughout this paper we consider two measures of how wrong a prediction a prediction is, the inaccuracy, i.e., distance of the performatively optimal report $\p$ to $f(\p)$, and the distance of the performatively optimal report to the fixed point. Our experiments show that these measures are closely but not perfectly correlated, see \Cref{fig:scatterplot-distfptounif-l2-inaccuracy-brier-it2}. The correlation is $0.958$.

\begin{figure}[H]
    \centering
    \includegraphics[width=0.6\linewidth]{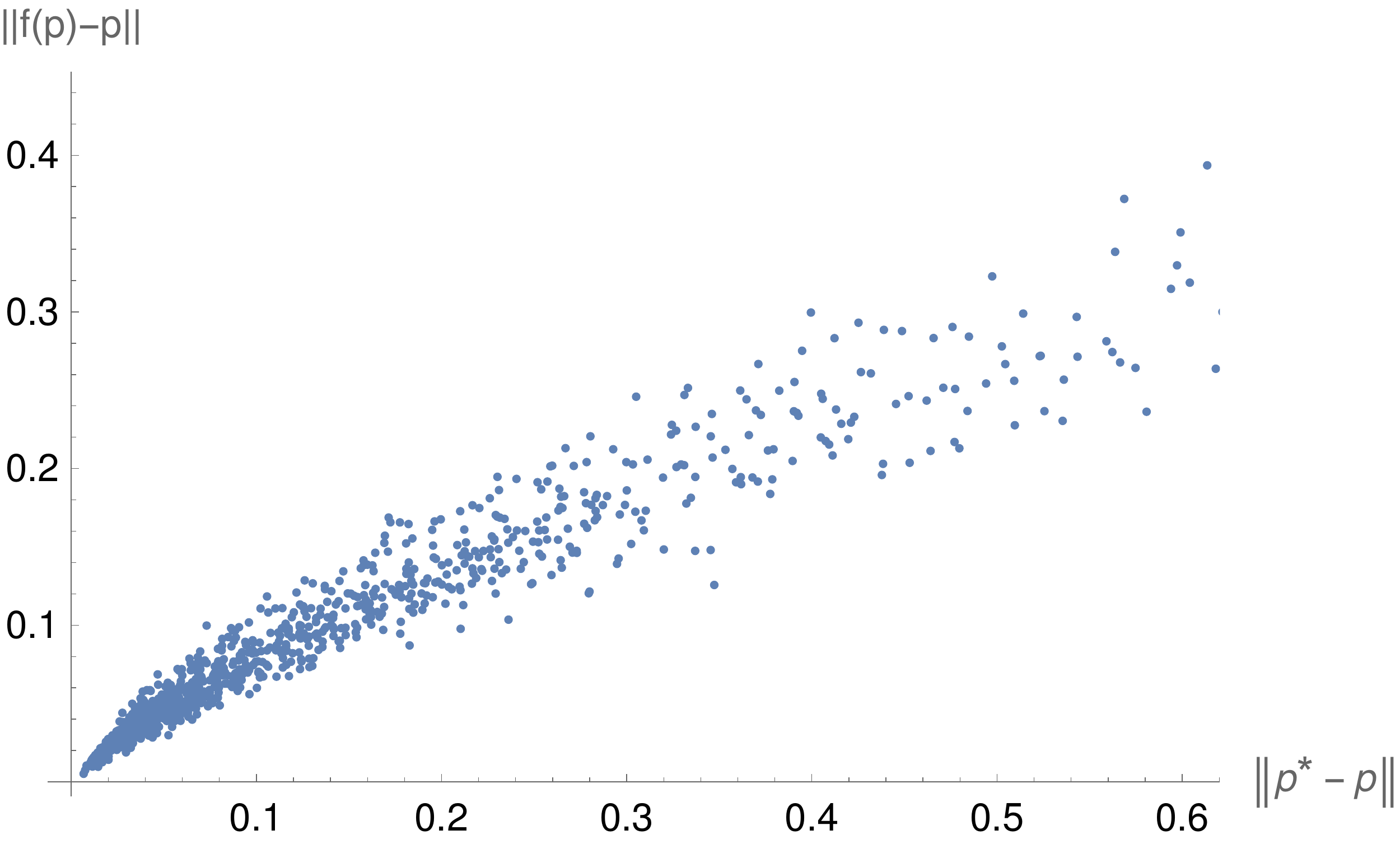}
    \caption{Scatter plot showing the L2 inaccuracy of the performatively optimal report against the L2 distance of the performatively optimal report to the fixed point report.}
    \label{fig:scatterplot-distfptounif-l2-inaccuracy-brier-it2}
\end{figure}

\subsubsection{The effect of fixed point location}
\label{appendix:experiments-many-outcomes-effect-of-fp-loc}

\Cref{fig:scatterplot-distfptounif-l2-disttofp-brier} scatter-plots the distance to fixed points against the distance of the fixed points from the uniform distribution. The blue line is the best linear fit, which is $0.0274 + 0.751 x$. Similarly \Cref{fig:scatterplot-distfptounif-l2-inaccuracy-brier} scatter-plots the inaccuracy of the performatively optimal report against the distance of the fixed point report to the uniform distribution. The blue line is again given by the best linear fit, which is $0.0231 + 0.468 x$.

The overall effect of the distance of $\p^*$ from uniform actually seems larger than the effect of the operator norm, as indicated by the correlation coefficients in \Cref{table:two-by-two-correlation}.

\begin{figure}[H]
    \centering
    \includegraphics[width=0.6\linewidth]{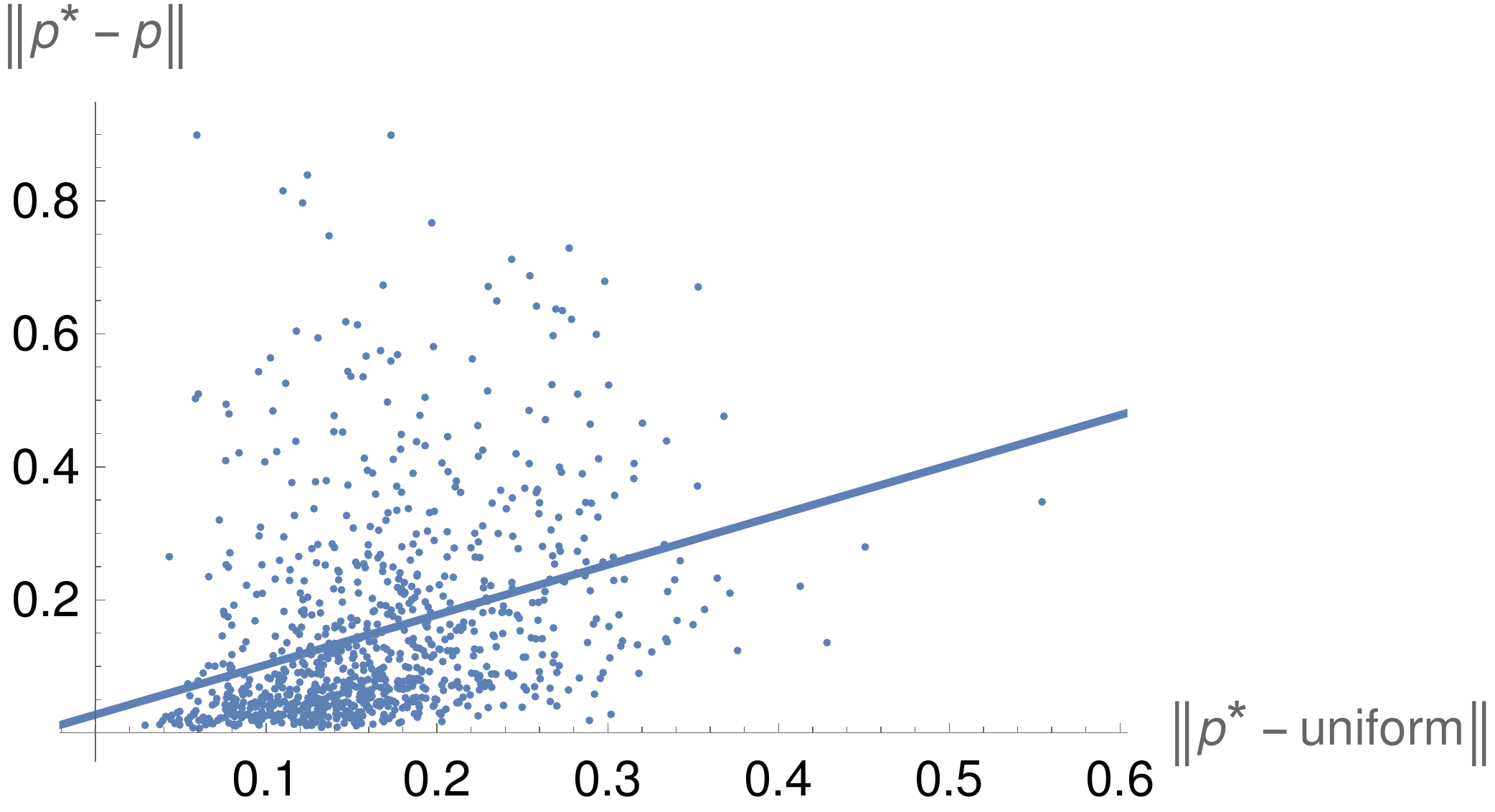}
    \caption{Scatter plot showing the L2 distance of the performatively optimal report to the fixed point report against distance of the fixed point to the uniform distribution in our experiments. The blue line is found by linear regression on the points.}
    \label{fig:scatterplot-distfptounif-l2-disttofp-brier}
\end{figure}

\begin{figure}[H]
    \centering
    \includegraphics[width=0.6\linewidth]{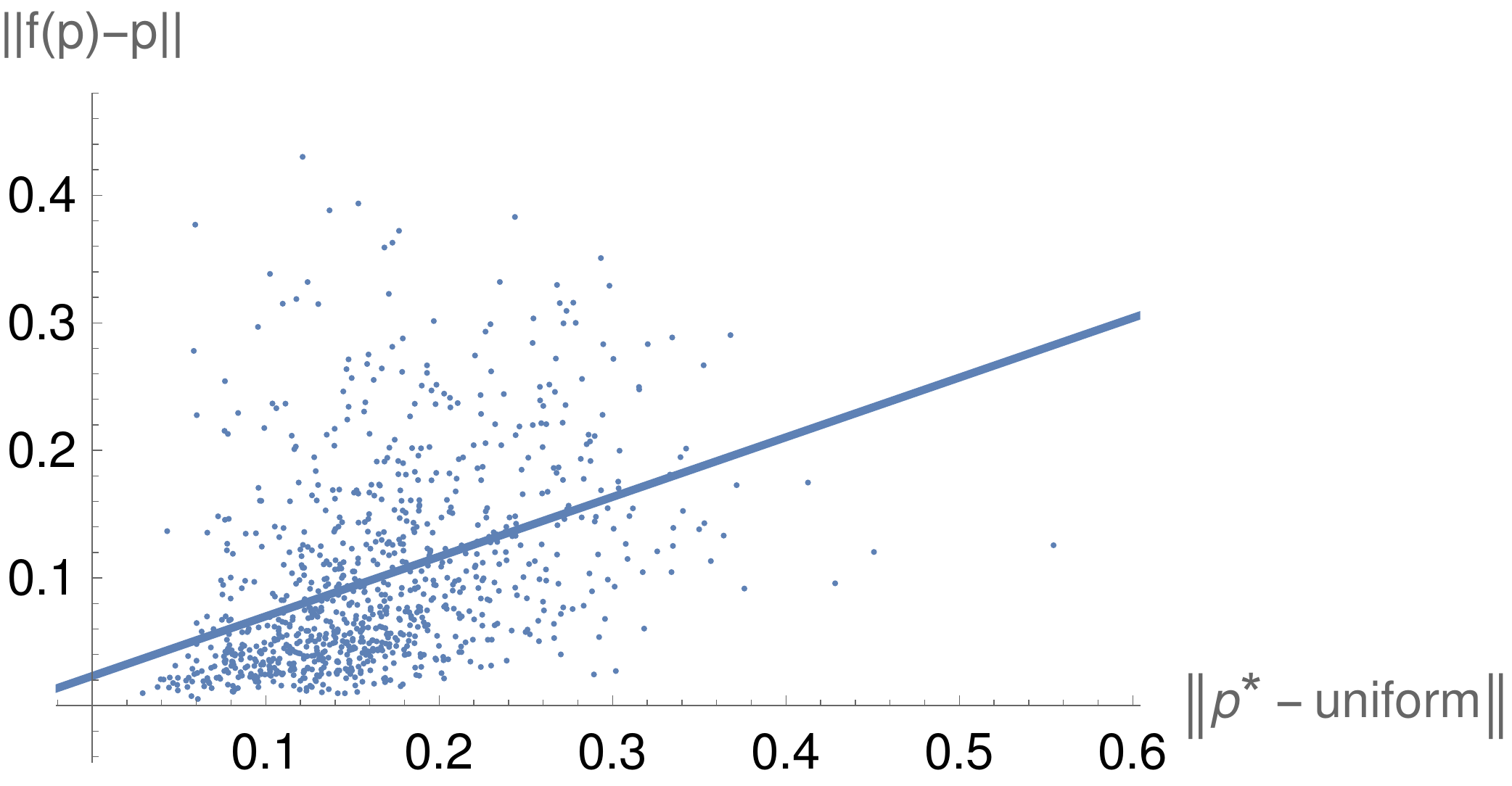}
    \caption{Scatter plot showing the L2 inaccuracy of the performatively optimal report against the distance of the fixed point of $f$ to the uniform distribution in our experiments. The blue line is found by linear regression on the points.}
    \label{fig:scatterplot-distfptounif-l2-inaccuracy-brier}
\end{figure}

\begin{table}[H]
    \centering
    \begin{tabular}{ccc}
    \hline\hline
         & $\Vert \p - \p^*\Vert$ & $\Vert \p - f(\p) \Vert$ \\
        \hline
        $\Vert f_A\Vert_{\mathrm{op}}$ & 0.294 & 0.311 \\
        $\Vert p^*-\frac{1}{n}\mathbf{1}\Vert$ & 0.331 & 0.411
    \end{tabular}
    \caption{Each entry shows the empirical correlation between the quantities determined by the row and column.}
    \label{table:two-by-two-correlation}
\end{table}

\section{Fixed points via alternative notions of optimality}
\label{appendix:alternative-notions-rationality}

In this section, we will review alternatives to performance optimality under which fixed points are incentivized. We will elaborate on the settings introduced in \Cref{stop-gradients} and provide formal statements and proofs.

To motivate the following, consider an expert AI that chooses its prediction to match its world model, but without explicitly considering the effect of its prediction. For instance, such cognition could arise in an AI trained via a purely supervised objective on historical data. This AI may not learn to take into account effects of its predictions on the outcome of the prediction. If it nevertheless has a world model that generalizes correctly to performative predictions, this could put the AI in a game in which it is trying to make a prediction to match its world model, while the world model updates its beliefs conditional on the AI's prediction. The only equilibria of this game would be fixed points.

Alternatively, fixed points could also result from different training schemes that explicitly optimize an AI's prediction to track empirical outcomes, without also incentivizing influencing the outcomes themselves, such as repeated risk minimization or repeated gradient descent \citep{perdomo2020performative}. 

Such expert AIs would likely be safer than ones optimizing for performative optimality. First, they report their true beliefs, which gives us better information to base decisions on. This also enables approaches in which we ensure that there is only one safe fixed point. Second, they do not explicitly optimize the choice of fixed point for a goal such as decreasing entropy. Instead, which fixed point is chosen will be contingent on initialization and specifics of the fixed point finding procedure.




\subsection{Performative stability and game theory}
\label{appendix:performative-stability-and-game-theory}

We begin by defining performative stability and relating it to an equilibrium in a two-player game. This represents the core idea behind all of the following settings.
A prediction \(\p^*\) is called \emph{performatively stable} \citep{perdomo2020performative} if
\begin{equation}\label{eq:stability-appendix}\p^*\in \argmax_{\p}\Score(\p,f(\p^*)).\end{equation}
First, it is clear that in our case, this is equivalent to \(\p^*\) being a fixed point.
\begin{proposition}\label{performative-stable-fixed-point}
    Assume \(S\) is strictly proper. Then a prediction \(\p^*\) is a fixed point if and only if it is performatively stable.
\end{proposition}
\begin{proof}
``\(\Rightarrow\)''. Assume \(f(\p^*)=\p^*\). Then \(S(\p^*,f(\p*))=S(\p^*,\p^*)\geq S(\p,\p^*)=S(\p,f(\p^*))\) for any \(\p\) since \(S\) is proper. Hence,
    \(\p^*\in \argmax_{\p}\Score(\p,f(\p^*)).\)
    
``\(\Leftarrow\)''. Assume \(\p^*\in \argmax_{\p}\Score(\p,f(\p^*))\). Then since \(S\) is strictly proper, it must be \(\p^*=f(\p^*)\).
\end{proof}

Next, the above objective is equivalent to the definition of a Nash equilibrium in the following game.


\begin{definition}[Oracle game]
Consider a two-player continuous game in which the first player controls \(\p\in \Pset\) and the second player controls \(\q\in\Pset\), with payoff functions \(U_1(\p,\q):=\Score(\p,\q)\) and \(U_2(\p,\q):=\Score(\q,f(\p))\) for the two players, respectively.
\end{definition}

If \(\p^*,\q^*\) is a Nash equilibrium of the oracle game, we have \(p^*=\argmax_\p \Score(\p,\q)\) and \(\q^*=\argmax_{\q} S(\q,f(\p^*))\). Substituting the optimal value \(\q^*=f(\p^*)\) for the second player gives us exactly above definition of performative stability in \Cref{eq:stability-appendix}. Conversely, if a prediction \(\p^*\) is performatively stable, then setting \(\q^*:=f(\p^*)\) yields a Nash equilibrium.

\begin{proposition}\label{prop:fp-are-ne}
Assume \(S\) is a proper scoring rule. Then \(\p\in \Pset\), \(\q:=f(\p)\) is a Nash equilibrium of the oracle game, if and only if \(\p\) is performatively stable. By \Cref{performative-stable-fixed-point}, this is equivalent to \(\p\) being a fixed point.
\end{proposition}

The oracle game could arise in an agent that uses a causal decision theory \citep{sep-decision-causal} to maximize its score and that believes that \(\Score\) is influenced causally by \(\p\), but only acausally by \(f(\p)\). In that case, the only \emph{ratifiable} \cite[][Ch.~1.7]{jeffrey1990logic,bell2021reinforcement} decision is a Nash equilibrium of the above game. Similarly, the deliberational causal epistemic decision theory discussed by \citet{greaves2013epistemic} would output Nash equilibria of this game (whereas performative optimality would correspond to an agent using evidential epistemic decision theory in this case).

Note that it is important that both players act simultaneously. \citet{perdomo2020performative} introduce a Stackelberg version of the oracle game that produces performatively optimal instead of performatively stable reports. Consider a game in which player \(1\) acts first and chooses \(\p\), after which player \(2\) responds with a prediction \(\q\). Then player \(2\) responds \(\q=f(\p)\) to player \(1\)'s action, and player \(1\)'s optimization problem becomes
\[p^*=\argmax_{\p}S(\p,\argmax_{\q}\Score(\q,f(\p)))=\argmax_\p \Score(\p,f(\p)).\]

\subsection{Repeated risk minimization and repeated gradient descent}
\label{rrm-and-rgd}
Above, we have defined performative stability and a related game which yield fixed points, but we have not defined methods for solving these problems. In the performative prediction context, \citet{perdomo2020performative} introduce \emph{repeated risk minimization} and \emph{repeated gradient descent}, both methods that converge to performatively stable points. In this section, we review both schemes and show how repeated gradient descent can be seen as gradient descent on a \emph{stop-gradient} \citep{foerster2018dice,demski2019partial} objective.

We assume direct access to \(\q\), instead of having only access to samples distributed according to \(\q\). In the next section, we discuss online learning when we only have access to samples. One way to understand this distinction is that the former corresponds to the internal cognition of an agent with a belief \(\q=f(\p)\) optimizing a prediction \(\p\). The latter instead corresponds to a machine learning training setup for an oracle AI, where \(\q\) is the ground truth environment distribution instead of the oracle's belief. Of course, there is no strict divide between the two. Any optimization algorithm could be used either by the agent itself or to train the agent.


First, \emph{repeated risk minimization} is a procedure by which we start with a prediction \(\p_0\) and then iteratively update the prediction as \(\p_{t+1}=\argmax_{\p}S(\p,f(\p_{t}))\). This is also the same as alternating best response learning in the oracle game, where player~\(1\) iteratively updates their prediction, responding to predictions \(\q_{t}=f(\p_{t})\) from player \(2\). If \(S\) is strictly proper, \(\p_{t+1}=f(\p_{t})\), and this results in \emph{fixed point iteration} for \(f\). Fixed point iteration converges globally to a fixed point if \(f\) has Lipschitz constant \(L_f<1\). It also converges locally to a fixed point \(\p^*\) if \(f\) is continuously differentiable at \(\p^*\) and \(\rho(Df(\p^*))<1\), where \(\rho(Df(\p^*))\) is the spectral radius of the Jacobian matrix \(Df(\p^*)\).

Second, assume that \(\Score\) is differentiable. Then \emph{repeated gradient ascent} updates points via
\[\p_{t+1}:=\Pi_\Delta(\p_{t}+\alpha \E_{y\sim f(\p_{t})}[\nabla_\p S(\p_{t},y)]),\]
where \(\Pi_\Delta\) is the Euclidean projection onto the probability simplex \(\Pset\), and \(\alpha>0\) is the learning rate.

Using the definition of \(\Score(\p,\q)\), we have
\[\E_{y\sim f(\p_t)}[\nabla_\p\Score(\p_t,y)]
=\nabla_\p(\E_{y\sim\q}[\Score(\p_t,y)])|_{\q=f(\p_t)}
=\nabla_\p(\Score(\p_t,\q))|_{\q=f(\p)}\]
We can express this as
\[\nabla_\p(\Score(\p_t,\bot f(\p_t))):=\nabla_\p(\Score(\p_t,\q))|_{\q=f(\p)},\]
where \(\bot\) is the \emph{stop-gradient operator}, which evaluates to the identity function but sets gradients to zero, \(\nabla_x( \bot x)=0\) \citep{foerster2018dice,demski2019partial}.\footnote{This is not a mathematical function (there is no function that is equal to the identity but has gradient zero everywhere), but rather a notational convention in reference to the \texttt{stop\_gradient} or \texttt{detach} functions from the tensorflow or pytorch python libraries. Interestingly, one can perform valid derivations using the stop-gradient operator (e.g., using the chain rule). We leave it to future work to explore the mathematics behind stop-gradients further.} In the following, we call \(S(\p, \bot f(\p))\) the \emph{stop-gradient objective}.

Importantly, it matters that the gradient in repeated gradient ascent lies inside instead of outside the expectation:
\[\E_{y\sim f(\p_t)}[\nabla_\p S(\p_t,y)]=\nabla_\p(S(\p_t,\bot f(\p_t)))\neq \nabla_p(S(\p_t,f(\p_t)))=\nabla_\p \E_{y\sim f(\p_{t})}[ S(\p_{t},y)]).\]
Unlike repeated gradient ascent, the latter implements gradient ascent on \(S(\p,f(\p))\) and thus leads to performatively optimal reports.

\citet{perdomo2020performative} show that, given their assumptions, repeated gradient descent globally converges to stable fixed points. They also provide convergence rates. We will show an analogous result relating repeated gradient ascent to fixed points in our setting, though we won't analyze global convergence or rates of convergence.

To begin, we show that repeated gradient descent is equivalent to Naive Learning \citep{letcherstable} in the oracle game, assuming that player~\(2\) always plays \(\q=f(\p)\).    
\begin{proposition}\label{prop:game-stop-gradient}
    Assume player \(1\) is performing gradient ascent on its objective with learning rate \(\alpha\), under the assumption that player \(2\) always plays \(\q=f(\p)\). Then player \(1\)'s update is
    \[\p_{t+1}=\Pi_\Delta(\p_t+\alpha\nabla_\p(\mathbf{S}(\p_t,\bot f(\p_t)))).\]
\end{proposition}

\begin{proof}The proof follows immediately from the definitions. Player \(1\)'s update is, by assumption,
    \[\p_{t+1}=\Pi_\Delta(\p_t+\alpha\nabla_\p(U_1(\p_t,\q)))=\Pi_\Delta(\p_t+\alpha\nabla_\p(\Score(\p_t,\q)))\]
    where \(\q\) is player \(2\)'s action. Assuming player \(2\) plays \(\q=f(\p_t)\), we get
\[\p_{t+1}=\Pi_\Delta(\p_t+\alpha\nabla_\p(\Score(\p_t,\q)))=
\Pi_\Delta(\p_t+\alpha\nabla_p(\Score(\p_t,\bot f(\p_t))))\]
\end{proof}

Next, we show that fixed points are critical points of the stop-gradient objective.

\begin{proposition}\label{stop-gradient-critical-points}
Assume \(S\) is proper and let \(G,g\) as in the Gneiting and Raftery characterization of \(S\) (\Cref{theorem:gneiting-raftery}) be differentiable. Then for any \(\p\in \Pset\), we have
\[\nabla_\p(S(\p,\bot f(\p)))= Dg(\p)^\top (f(\p)-\p).\]
In particular, if \(\p\) is a fixed point, it follows that \(\nabla_\p(\Score(\p,\bot f(\p)))=0\). The reverse is true if \(Dg(\p)|_{\TPset}\succ 0\).
\end{proposition}
\begin{proof}
    \begin{multline}
        \nabla_\p (\Score(\p,\bot f(\p)))
        =\nabla_\p(\Score(\p,\q))|_{\q=f(\p)}
        =\nabla_\p(G(\p)+g(\p)^\top (\q-\p))|_{\q=f(\p)}
        \\
        =(g(\p)+Dg(\p)^\top (\q-\p)-g(\p))|_{\q=f(\p)}
        =Dg(\p)^\top (f(\p)-\p).
    \end{multline}
    If \(\p\) is a fixed point, it follows that \(\nabla_\p (\Score(\p,\bot f(\p)))=0\). Moreover, if \(Dg(\p)|_{\TPset}\succ 0\), then if \(f(\p)-\p\neq 0,\) 
    \[  \nabla_\p (\Score(\p,\bot f(\p)))^\top (f(\p)-\p)
    =(f(\p)-\p)^\top Dg(\p)(f(\p)-\p)>0\]
    and thus \( \nabla_\p (\Score(\p,\bot f(\p)))\neq 0\).
\end{proof}


Finally, we show that in our setting, repeated gradient ascent locally converges to fixed points \(\p^*\), assuming that \(\Vert Df(\p^*)\Vert_{\op}\) is sufficiently small. This is a local version of convergence results from \cite{perdomo2020performative}, adapted to our setting.

\begin{proposition}\label{prop:convergence-stop-gradient}
   Let \(S\) be a strictly proper scoring rule. Let \(\p^*\in\interior{\Pset}\) be a fixed point of \(f\) such that \(G\) is three times differentiable at \(\p^*\), i.e. \(D^2g(\p^*)=D^2\nabla g(\p^*)\) exists. Assume \(\beta\succeq Dg(\p^*)|_{\TPset}\succeq \gamma>0\), that \(f\) is differentiable at \(\p^*\), and \(\Vert Df(\p^*)\Vert_{\op}<\frac{\gamma}{\beta}\). Then, for small enough \(\alpha>0\), an agent taking updates \(\p_{t+1}=\Pi_{\Delta}(\p_t+\alpha \nabla_\p(\mathbf{S}(\p_t,\bot f(\p_t))))\) will locally converge to \(\p^*\).
\end{proposition}

For the proof, we use the following generalization of Ostrowski's theorem, adapted from \citet{Kitchen1966}.

\begin{theorem}[\cite{Kitchen1966}]\label{thm:kitchen}
Let \(\varphi\colon D\subseteq V\rightarrow W\) where \(V,W\) are Banach spaces. Assume
\begin{itemize}
\item \(\varphi\) has a fixed point \(\x^*\in \mathrm{int}(D)\)
\item \(\varphi\) is differentiable at \(\x^*\)
\item \(\rho(D\varphi(\x^*))<1\).
\end{itemize}
Then there exists an open set \(U\subseteq D\) with \(x^*\in U\) such that, letting \(\x_0\in U\) and \(\x_t\defeq\varphi(\x_{t-1})\) for \(k\in\mathbb{N},\) we have \(\x_t\in U\) for all \(k\) and \(\lim_t\x_t= \x^*.\)
\end{theorem}

\begin{proof}[Proof of \Cref{prop:convergence-stop-gradient}]
The Banach space we consider will be \(\TPset\). Note that, since \(\p^*\in \interior{\Pset},\) there exists an open set \(\D\subseteq\TPset\) (with respect to the standard topology on \(\TPset\)) with \(0\in \D\) such that \(\vec{v}+\p^*\subseteq \Pset\) for all \(\vec{v}\in \D\). Our iteration function then is \[\varphi\colon \D\subseteq\TPset\rightarrow \TPset, \vec{v}\mapsto \vec{v}+\alpha\nabla_{\vec{v}}(\Score(\vec{v}+\p^*,\bot f(\vec{v}+\p^*))).\]

Note that \(\varphi\) has a fixed point at \(0\). Our goal is now to show that there exists \(\alpha>0\) and an open set \(U\subseteq\D\) such that iterates of \(\varphi\) starting in \(U\) stay in \(U\) and converge to \(0\).

To that end, note that, using \Cref{stop-gradient-critical-points}, we have \(\nabla_{\p}(\Score(\p,\bot f(\p)))=Dg(\p)^\top (f(\p)-\p)\) and thus
\begin{align}D\varphi(0)
&=D(\vec{v}\mapsto\vec{v}+\alpha\nabla_{\vec{v}}\mathbf{S}(\vec{v}+\p^*,\bot f(\vec{v}+\p^*)))(0)\\
&
=
\Id + \alpha D(\vec{v}\mapsto Dg(\vec{v}+\p^*)^\top(f(\vec{v}+\p^*)-\vec{v}-\p^*))(0)\\
&=\Id + \alpha D^2g(\p^*)[f(\p^*)-\p^*]
+ \alpha Dg(\p^*)^\top(Df(\p^*)-\Id).
\end{align}
Here, \(D^2g(\vec{v}+\p^*)\) is a third-degree tensor, and \(D^2g(\vec{v}+\p^*)[f(\p^*)-\p^*]\) is a linear map. Since \(f(\p^*)=\p^*\), it follows 
\(D\varphi(0)=\Id+\alpha Dg(\p^*)^\top (Df(\p^*)-\Id)\). In particular, \(\varphi\) is differentiable at \(0\).

Now let \(\vec{v}\) be an arbitrary eigenvector of \(D\varphi(0)\), with eigenvalue \(\lambda\) and w.l.o.g. assume \(\Vert\vec{v}\Vert=1\). Note that \(\vec{v}^\top Dg(\p^*)\vec{v}\geq \gamma \Vert \vec{v}\Vert = \gamma\) and \(\vec{v}^\top Dg(\p^*)\vec{v}\leq \beta\Vert \vec{v}\Vert\leq \beta\) by assumption.
Letting \(\alpha:=\frac{1}{ \beta}\), it follows that \(\alpha \vec{v}^\top Dg(\p^*)\vec{v}\leq 1\) and thus
\[|1-\alpha \vec{v}^\top Dg(\p^*)^\top\vec{v}|
=|1-\alpha \vec{v}^\top Dg(\p^*)\vec{v}|
=1-\alpha \vec{v}^\top Dg(\p^*)\vec{v}
\leq 1-\alpha \gamma.
\]
Moreover, since \(Dg(\p^*)\) is the Hessian of \(G\) and thus symmetric since \(G\) is twice differentiable, we have \(\Vert Dg(\p^*)\Vert_{\mathrm{op}}\leq \beta\).
Using this, as well as our assumption \(\Vert Df(\p^*)\Vert_{\op}<\frac{\gamma}{\beta}\), we get
\begin{multline}
    |\lambda|
    =\vert \lambda \vec{v}^\top\vec{v}\vert
    =\vert \vec{v}^\top D\varphi(0)\vec{v}\vert
    =\vert\vec{v}^\top (\Id+\alpha Dg(\p^*)^\top(Df(\p^*)-\Id))\vec{v}\vert
    \\
    =\vert\vec{v}^\top \vec{v}-\alpha \vec{v}^\top Dg(\p^*)^\top\vec{v}+\alpha \vec{v}^\top Dg(\p^*)^\top Df(\p^*))\vec{v}\vert
    \\
    \leq
    \vert 1-\alpha \vec{v}^\top Dg(\p^*)^\top\vec{v}\vert + \alpha \vert\vec{v}^\top Dg(\p^*)^\top Df(\p^*)\vec{v}\vert
    \\
    \underset{\text{Cauchy-Schwarz}}{\leq}
     1-\alpha \gamma+ \alpha \Vert Dg(\p^*) \vec{v}\Vert \Vert Df(\p^*)\vec{v}\Vert
     \\
    \leq
     1-\alpha \gamma + \alpha \Vert Dg(\p^*)\Vert_{\op} \Vert \vec{v}\Vert \Vert Df(\p^*) \Vert_{\op}\Vert\vec{v}\Vert
    \\ =
     1-\alpha \gamma + \alpha \Vert Dg(\p^*)\Vert_{\op}\Vert Df(\p^*)\Vert_{\op}
     < 1-\alpha \gamma + \alpha \gamma
     =1.
\end{multline}

This shows that \(\rho(D\varphi(0))<1\). Hence, by \Cref{thm:kitchen}, we can conclude that there exists an open set \(U\subseteq \D\) such that for arbitrary \(\vec{v}_0\in U\), \(\vec{v}_t:=\varphi(\vec{v}_{t-1})\in U\) for all \(t\geq 1\), and \(\lim_{t\rightarrow\infty}\vec{v}_t=0\). In particular, note that since \(\vec{v}_t\in U\) for all \(t\), \(\vec{v}_t+\p^*\in\Pset\) and
\[\p^*+\vec{v}_{t+1}=\p^* +\vec{v}_t+\alpha\nabla_{\vec{v}}(\Score(\vec{v}_t+\p^*,\bot f(\vec{v}_t+\p^*)))
=
\Pi_\Delta(\p^*+\vec{v}_t+\alpha\nabla_{\vec{v}}(\Score(\vec{v}_t+\p^*,\bot f(\vec{v}_t+\p^*))))
\]
for all \(t\). Hence, setting \(\p_t:=\p^*+\vec{v}_t\), it follows
\(\p_{t+1}=\Pi_{\Delta}(\p_t+\alpha \nabla_\p(\mathbf{S}(\p_t,\bot f(\p_t))))\)
for all \(t\) and \[\lim_{t\rightarrow\infty}\p_t
=\p^* + \lim_{t\rightarrow\infty}\vec{v}_t
=\p^*.\]
This concludes the proof.
\end{proof}

\subsection{Online learning}
\label{appendix:online-learning}
Now consider a machine learning setup in which we train an oracle with stochastic gradient ascent on environment samples. We assume that at time \(t\), a model makes a prediction \(\Pvar_t\) and receives a score \(S(\Pvar_t,\Y_t)\), where \(\Y_t\sim f(\Pvar_t)\). The model is then updated using gradient ascent on \(\Score (\Pvar_t,\Y_t)\). That is, for some learning rate schedule \((\alpha_t)_t\), we have
\[\Pvar_{t+1}=\Pi_\Delta(\Pvar_t+\alpha_t\nabla_\p\Score(\Pvar_t,\Y_t)),\]
where $\Pi_\Delta$ is the Euclidean projection onto $\Delta(\mathcal N)$ as before.

We discuss this as a theoretical model for oracles trained using machine learning, to show how training setups may incentivize predicting fixed points. There are many issues with the setting beyond giving accurate predictions; for instance, learning may fail to converge at all, and even if the training process sets the right incentives on training examples, the learned model may be optimizing a different objective when generalizing to new predictions \citep{hubinger2019risks} .

To see that this setting leads to fixed points, note that we have \[\E_{\Y_t\sim f(\Pvar_t)}[\nabla_\p\Score(\Pvar_t,\Y_t)]=\nabla_\p \E_{\Y_t\sim \bot f(\Pvar_t)}[\Score(\Pvar_t,\Y_t)]
=\nabla_\p( \Score(\Pvar_t,\bot f(\Pvar_t))).\]
That is, the expectation of this gradient, conditional on \(\Pvar_t,\) is exactly the repeated gradient from the previous section. Hence, given the right assumptions, this converges to fixed points instead of performative optima. We do not show this here, but an analogous result in performative prediction was proved by \citet{mendler2020stochastic}.

There are several variations of this setup that essentially set the same incentives. For instance, one could also draw entire batches of outcomes \(\Y_{t,1:B}\) and then perform updates based on the batch gradient \(\nabla_\p\sum_{b=1}^BS(\Pvar_t,\Y_{t,b}).\) This is a Monte Carlo estimate of the repeated gradient and hence also converges to performatively stable points and thus fixed points \citep{perdomo2020performative}. One could also mix the two algorithms and, e.g., perform gradient ascent on an average of past losses, yielding a version of the backwards-facing oracle discussed in \citet{armstrong2018standard}.

Note that finding fixed points depends on the fact that we differentiate \(S(\Pvar_t,\Y_t)\) instead of the expectation
\(\E_{\Y_t\sim f(\Pvar_t)}[\Score(\Pvar_t,\Y_t)]=\Score(\Pvar_t,f(\Pvar_t))\). If we used policy gradients to differentiate \(\Score(\Pvar_t,f(\Pvar_t))\), for instance, we would again optimize for performative optimality. Similarly, we could learn a Q-function representing scores for each prediction, and update the function based on randomly sampled predictions \(\p\). Then the Q-function would converge to estimates of \(\Score(\p,f(\p))\), and the highest Q-value prediction would be a performative optimum. There are also some more recent results in performative prediction that explicitly try to estimate the gradient \(\nabla_\p(\Score(\p,f(\p)))\) and thus find performatively optimal instead of stable points \citep{izzo2021learn}.

Stop-gradients could also be circumvented in a hidden way \citep{krueger2020hidden}. For instance, consider a hyperparameter search to meta-learn a learning algorithm, where the evaluation criterion is the accumulated score during an episode. Then this search would prefer algorithms that optimize \(\Score(\p,f(\p))\) directly, without a stop-gradient.

Lastly, repeated gradient descent is related to \emph{decoupled approval} in RL \cite{uesato2020avoiding}. The decoupled approval policy gradient samples actions and approval queries independently and can thus differentiate with a stop-gradient in front of the approval signal. In our setting, we can differentiate through \(S(\Pvar_t,\Y_t)\) directly, so it is not necessary to calculate this gradient with a decoupled policy gradient. Decoupled gradients could be used to implement the stop-gradient objective if scores were discrete or otherwise not differentiable.

\subsection{No-regret learning}
\label{appendix:no-regret}
In this section, we consider no-regret learning and show that algorithms have sublinear regret if and only if their prediction error is sublinear. Regret takes environment outcomes as given and asks which predictions would have been optimal in hindsight. It thus corresponds to an alternative notion of optimality with a ``stop-gradient'' in front of environment probabilities.

As in the previous section, we assume that at time \(t\in\mathbb{N},\) the agent (i.e., the oracle AI) makes a prediction \(\Pvar_t\) and receives a score \(S(\Pvar_t,\Y_t)\), where \(Y_t\sim f(\Pvar_t)\). The agent's cumulative score at step \(T\) is defined as \(\sum_{t=1}^T\Score(\Pvar_t,\Y_t)\). In no-regret learning, we compare performance against \emph{experts}, which choose sequences of probabilities \((\Pvar'_t)_{t},\) \(\Pvar'_t\in \Pset\). We assume that an expert's prediction \(\Pvar_t'\) is independent of \(\Y_t\) conditional on \(\Pvar_t\). I.e., an expert knows the predictions \(\Pvar_t\) and thus probabilities \(f(\Pvar_t)\), but it does not know the outcome of \(\Y_t\). Let \(\mathcal{P}\) be the set of all such experts.

The regret of the agent is the difference between the cumulative score received by the best expert in expectation and the cumulative score received by the agent. To define it formally, let
\[\Pvar^*_t\in \argmax_{\Pvar'_t\in\mathcal{P}}\mathbb{E}[S(\Pvar'_t,\Y_t)\mid \Pvar_t]\] for \(t\in\mathbb{N}\). \(\Pvar^*_t\) is a random variable that maximizes the expectation of \(S(\Pvar^*_t,\Y_t)\) before \(\Y_t\) is drawn, but conditional on \(\Pvar_t\).
\begin{definition}[Regret]
The regret of agent $(\Pvar_t)_{t}$ at time \(T\) is
\begin{equation*}
\mathrm{Regret}(T) \coloneqq \sum_{t=1}^T S(\Pvar_t^*,\Y_t) - S(\Pvar_t,\Y_t).
\end{equation*}
The agent is said to have \emph{sublinear regret} or \emph{no-regret} if
\[\limsup_{T\rightarrow\infty}\frac{1}{T}\mathrm{Regret}(T)\leq 0.\]
\end{definition}

First, note that we define regret relative to the best expert in expectation instead of the best expert in hindsight. The latter would always be the one that made confident predictions and accidentally got all predictions exactly right. We are interested in algorithms with sublinear regret, and for that purpose it would be too much to ask the agent to perform well compared to the best expert in hindsight. Moreover, for scoring rules that are symmetric between the outcomes, this expert would have a constant score \(C\). This would imply that $\text{Regret}(T) = \sum_{t=1}^T C - S(\Pvar_t,\Y_t)$ and reduce the problem to minimizing the negative score, which would lead to performatively optimal predictions. 

Second, we evaluate the performance of the expert with respect to the environment outcomes \(\Y_t\) generated by the agent \((\Pvar_t)_t\), instead of evaluating the expert according to outcomes \(\tilde{\Y}_t\sim f(\Pvar^*_t)\) generated using the expert's own predictions. This means that, to receive sublinear regret, the agent only has to make accurate predictions—it does not have to find a performatively optimal prediction. This is different from the no-regret learning setup discussed in \citet{pmlr-v162-jagadeesan22a}, where regret is defined with respect to \(S(\Pvar^*_t,f(\Pvar^*_t))\). In that setting, only agents converging to performatively optimal predictions have sublinear regret.

We begin by showing that the best expert in expectation actually exists, and that \(\Pvar^*_t=f(\Pvar_t)\). 
\begin{proposition}\label{prop:f-of-p-optimal}
Let \(S\) be a proper scoring rule and \((\Pvar_t')_t\in\mathcal{P}\) an expert. Then for any \(t\in\mathbb{N}\), we have
\[\mathbb{E}[S(\Pvar_t',\Y_t)]=\mathbb{E}[\Score (\Pvar_t',f(\Pvar_t))].\]
 Moreover, we have \((\Pvar^*_t)_t=(f(\Pvar_t))_t\) and thus 
\[\mathrm{Regret}(T)
=\sum_{t=1}^TS(f(\Pvar_t),\Y_t)-S(\Pvar_t,\Y_t).\]
\end{proposition}
\begin{proof}
Let \(t\in\mathbb{N}\) and let \((\Pvar_t')_t\in\mathcal{P}\) be any expert. Conditional on \(\Pvar_t\), \(\Y_t\sim f(\Pvar_t)\) and \(\Y_t\) is independent of \(\Pvar_t'\) by assumption. Hence,
\[
\mathbb{E}\left[S(\Pvar_t',\Y_t) \right]
= 
\mathbb{E}\left[ \mathbb{E}[S(\Pvar_t',\Y_t)\mid \Pvar_t,\Pvar_t'] \right]
= 
\mathbb{E}\left[\Score (\Pvar_t',f(\Pvar_t)) \right].
\]

Next, since \(S\) is proper, 
\[
\mathbb{E}\left[ \Score (\Pvar_t',f(\Pvar_t)) \right]
\leq 
\mathbb{E}\left[ \Score(f(\Pvar_t),f(\Pvar_t))\right].
\]
It follows that
\[\max_{(\Pvar'_t)_t\in\mathcal{P}}
\mathbb{E}\left[ S(\Pvar_t',\Y_t) \right]
=
\max_{(\Pvar'_t)_t\in\mathcal{P}}
\mathbb{E}\left[ \Score(\Pvar_t',f(\Pvar_t)) \right]
\leq
\mathbb{E}\left[ \Score(f(\Pvar_t),f(\Pvar_t))\right]
=
\mathbb{E}\left[ S(f(\Pvar_t), \Y_t)\right].\]
Moreover, \((f(\Pvar_t))_t\in\mathcal{P}\), as \(f(\Pvar_t)\) is constant given \(\Pvar_t\) and thus independent of \(\Y_t\).

It follows that, for any \(t\in\mathbb{N},\) \(\Pvar^*_t\in\argmax_{(\Pvar_t')_t\in \mathcal{P}}\E[S(\Pvar'_t,\Y_t)]\), and thus 
\[\Regret(T)=\sum_{t=1}^TS(f(\Pvar_t),\Y_t)-S(\Pvar_t,\Y_t).\]
\end{proof}

\subsubsection{Characterization of regret in the limit}

If \(S\) is unbounded (such as the log scoring rule), then the agent's scores can become arbitrarily low, and the limit of \(\frac{1}{T}\mathrm{Regret}(T)\) may be undefined. To simplify our analysis, we will thus assume that there is a bound on the variance of the received score \(S(\Pvar'_t,\Y_t)\) and on the expected score \(\Score(\Pvar'_t,f(\Pvar_t))\) of both the agent, \(\Pvar'_t=\Pvar_t\), and the best expert, \(\Pvar'_t=\Pvar^*_t\). In the case of the log scoring rule, this would be satisfied, for instance, if the agent's predictions are bounded away from the boundary of the probability simplex.

Our next proposition shows that, given these assumptions, \(\lim_{T\rightarrow\infty}\frac{1}{T}\mathrm{Regret}(T)\) exists and is nonnegative, and having sublinear regret is equivalent to 
\(\lim_{t\rightarrow\infty}\frac{1}{T}\mathrm{Regret}(T)=0.\)

\begin{proposition}\label{prop:slln}
Let \(S\) be a proper scoring rule. Assume that \(\sup_t |\Score(\Pvar'_t,f(\Pvar_t))|<\infty\) and that \(\sup_t \mathrm{Var}(S(\Pvar'_t,\Y_t))<\infty\) for \(\Pvar'_t\in\{\Pvar_t,f(\Pvar_t)\}\). Then almost surely
\[\lim_{T\rightarrow\infty}\frac{1}{T}\mathrm{Regret}(T)=\lim_{T\rightarrow\infty}\frac{1}{T}\sum_{t=1}^T\Score(f(\Pvar_t),f(\Pvar_t)) - \Score(\Pvar_t,f(\Pvar_t))\geq 0.\]
In particular, almost surely both limits exist and are finite, and the agent has sublinear regret if and only if
\[
\lim_{T\rightarrow\infty}\frac{1}{T}\sum_{t=1}^T\Score(f(\Pvar_t),f(\Pvar_t)) - \Score(\Pvar_t,f(\Pvar_t))=0.\]
\end{proposition}


\begin{proof}
We will use a version of the strong law of large numbers for uncorrelated random variables with bounded variance, adapted from \citet[Theorem 2]{Neely2021}.

\begin{theorem}[\cite{Neely2021}, Theorem 2]\label{thm:slln}
    Let \(\{X_t\}_{t\in\mathbb{N}_0}\) be a sequence of pairwise uncorrelated random variables with mean \(0\) and bounded variances. I.e., assume that
    \begin{enumerate}
        \item \(\mathbb{E}[X_t]=0\) for all \(t\in\mathbb{N}_0\)
        \item There exists \(c>0\) such that \(\mathrm{Var}(X_t)\leq c\) for all \(t\in\mathbb{N}_0\)
        \item \(\mathrm{Cov}(X_t,X_{t'})=0\) for all \(t\neq t'\in\mathbb{N}_0\).
    \end{enumerate}
    Then almost surely
    \[\lim_{T\rightarrow\infty}\frac{1}{T}\sum_{t=1}^TX_t=0.\]
\end{theorem}
We will apply this law to random variables
\(X_t:=S(\Pvar'_t,\Y_t)-\Score(\Pvar'_t,f(\Pvar_t))\), where \(\Pvar'_t\) is either \(\Pvar_t\) or \(f(\Pvar_t)\).

First, by \Cref{prop:f-of-p-optimal}, \(\mathbb{E}[X_t]=\mathbb{E}[S(\Pvar'_t,\Y_t)-\Score(\Pvar'_t,f(\Pvar_t))]=0\). Second, by assumption, \[\sup_t\mathrm{Var}(S(\Pvar'_t,f(\Pvar_t)))<\infty.\] Hence, also 
\[\sup_t\mathrm{Var}(\Score(\Pvar'_t,f(\Pvar_t)))
=
\sup_t\mathrm{Var}(\mathbb{E}[S(\Pvar'_t,f(\Pvar_t))\mid \Pvar_t])
\leq\sup_t\mathrm{Var}(S(\Pvar'_t,f(\Pvar_t)))
<\infty.
\]
It follows that also \(\sup_t \mathrm{Var}(X_t)<\infty\).

Third, we know that \(\Y_t\) is independent of \(\Pvar_{t'}\) and \(\Y_{t'}\) for \(t>t'\), conditional on \(\Pvar_t\). Moreover, \(\Pvar_t'\) is constant given \(\Pvar_t\). Hence, given \(\Pvar_t\), also \(X_t=S(\Pvar'_t,\Y_t)-\Score(\Pvar'_t,f(\Pvar_t))\) is independent of \(X_{t'}\). Moreover,
\[
\mathbb{E}[X_t\mid \Pvar_t]
=\mathbb{E}[S(\Pvar'_t,Y_t)-\Score(\Pvar'_t,f(\Pvar_t))\mid \Pvar_t]=
\Score(\Pvar'_t,f(\Pvar_t))-\Score(\Pvar'_t,f(\Pvar_t))=0.
\]
It follows for \(t>t'\) that
\[\mathrm{Cov}(X_t,X_{t'})
=\mathbb{E}[X_t X_{t'}]
=\mathbb{E}[\mathbb{E}[X_t X_{t'}\mid \Pvar_t]]
=\mathbb{E}[\mathbb{E}[X_t \mid \Pvar_t]
\mathbb{E}[X_{t'}\mid \Pvar_t]]=0.\]

This shows all conditions of the theorem and thus
\[\lim_{t\rightarrow\infty}\frac{1}{T}\sum_{t=1}^TX_t=0\]
almost surely.

Now we turn to the limit of \(\frac{1}{T}\sum_{t=1}^T\Score(\Pvar_t',f(\Pvar_t))\). By assumption, \(\sup_t|\Score(\Pvar'_t,f(\Pvar_t))|<\infty\), so this limit exists and is finite. Thus, almost surely
\[
\lim_{T\rightarrow\infty}\frac{1}{T}\sum_{t=1}^T\Score(\Pvar'_t,f(\Pvar_t))
=
\lim_{T\rightarrow\infty}\frac{1}{T}\sum_{t=1}^TS(\Pvar'_t,Y_t)-X_t
\]
\[
=\lim_{T\rightarrow\infty}\frac{1}{T}\sum_{t=1}^TS(\Pvar'_t,Y_t)
-\lim_{T\rightarrow\infty}\frac{1}{T}\sum_{t=1}^TX_t
=\lim_{T\rightarrow\infty}
\frac{1}{T}\sum_{t=1}^TS(\Pvar'_t,Y_t).
\]

Using \Cref{prop:f-of-p-optimal}, it follows that almost surely
\[\lim_{T\rightarrow\infty}\frac{1}{T}\mathrm{Regret}(T)
=\lim_{T\rightarrow\infty}\frac{1}{T}\sum_{t=1}^TS(f(\Pvar_t),Y_t)-S(\Pvar_t,Y_t)
=\lim_{T\rightarrow\infty}\frac{1}{T}\sum_{t=1}^TS(f(\Pvar_t),Y_t)-\lim_{T\rightarrow\infty}\frac{1}{T}\sum_{t=1}^TS(\Pvar_t,Y_t)
\]
\[
=\lim_{T\rightarrow\infty}\frac{1}{T}\sum_{t=1}^T\Score(f(\Pvar_t),f(\Pvar_t))-\lim_{T\rightarrow\infty}\frac{1}{T}\sum_{t=1}^T\Score(\Pvar_t,f(\Pvar_t))
=\lim_{T\rightarrow\infty}\frac{1}{T}\sum_{t=1}^T\Score(f(\Pvar_t),f(\Pvar_t))-\Score(\Pvar_t,f(\Pvar_t)).
\]

Turning to the ``in particular'' part, note that this limit is finite by the above, and it is nonnegative since \(S\) is assumed to be proper. Moreover, it follows that almost surely
\[
\limsup_{T\rightarrow\infty}\frac{1}{T}\Regret(T)
=\lim_{T\rightarrow\infty}\frac{1}{T}\Regret(T)\geq 0.
\]
Thus, almost surely \(\limsup_{T\rightarrow\infty}\frac{1}{T}\Regret(T)\leq 0\) if and only if \(\lim_{T\rightarrow\infty}\frac{1}{T}\sum_{t=1}^T\Score(f(\Pvar_t),f(\Pvar_t))-\Score(\Pvar_t,f(\Pvar_t))=0.\)
This concludes the proof.
\end{proof}

\subsubsection{Sublinear regret \(\Leftrightarrow\) sublinear prediction error}
Now we turn to the main result of this section. We show that given our assumptions, agents have sublinear regret if and only if their prediction error is sublinear. Note that here, we do \emph{not} require the \(\Pvar_t\) to converge; they could also oscillate between different fixed points.

\begin{theorem}\label{prop:no-regret-fp}
    Let $(\Pvar_t)_t$ be the sequence of the agent's predictions and \(S\) a strictly proper scoring rule.
    Assume that \(\sup_t \mathrm{Var}(S(\Pvar'_t,\Y_t))<\infty\) for \(\Pvar'_t\in\{\Pvar_t,f(\Pvar_t)\}\), and assume that there exists a closed set \(\mathcal{C}\subseteq\Pset\) such that \(\Pvar_t\in\mathcal{C}\) for all \(t\) and \(\Score(\p,f(\p))\), \(\Score(f(\p),f(\p)),\) and \(f(\p)\) are continuous in \(\p\) at any \(\p\in \mathcal{C}\). Then almost surely the agent has sublinear regret if and only if \(\sum_{t=1}^T \Vert f(\Pvar_t)-\Pvar_t\Vert \) is sublinear, i.e., if $\lim_{t\rightarrow\infty}\frac{1}{T}\sum_{t=1}^T \Vert f(\Pvar_t)-\Pvar_t\Vert=0$.
\end{theorem}



To show the result, we begin by proving an analytic lemma.

\begin{lemma}\label{lem:lemma-analysis}
Let \(\varphi,\psi\colon\mathbb{N}\rightarrow[0,\infty)\) and assume there exists a constant \(C>0\) such that for all \(t\in\mathbb{N},\) we have \(\psi(t)\leq C.\) Assume that for any \(\epsilon>0\), there exists \(\delta>0\) such that if \(\psi(t)>\epsilon\) for any \(t\in\mathbb{N}\), then \(\varphi(t)>\delta\). Then
\[\lim_{T\rightarrow\infty}\frac{1}{T}\sum_{t=1}^T\varphi(t)=0
\Rightarrow
\lim_{T\rightarrow\infty}\frac{1}{T}\sum_{t=1}^T\psi(t)=0.\]
\end{lemma}
\begin{proof}
We prove the contrapositive. That is, we assume that there exists some constant \(c>0\) such that there are infinitely many \(T\in\mathbb{N}\) such that \(\frac{1}{T}\sum_{t=1}^T\psi(t)>c\). Let \(\mathcal{T}\) be the set of such \(T\). We show that then there exists a constant \(c'>0\) such that for infinitely many \(T\), \(\frac{1}{T}\sum_{t=1}^T\varphi(t)>c'\).

Let \(T\in\mathcal{T}\). Since by assumption
\(\frac{1}{T}\sum_{t=1}^T\psi(t)>c,\) it follows that
\(\sum_{t=1}^T\frac{\psi(t)}{c}>T.\)
Let \(C':=\max\{C,1/c\}+1\). Since \(\psi(t)<C'\) it must be \(\psi(t)>\frac{c}{2C}\) for more than \(\frac{c}{2C}\) fraction of the times \(t\leq T\). Otherwise, it would be
\[
\sum_{t=1}^T\frac{\psi(t)}{c}
\leq T\left(\frac{C}{c}\frac{c}{2C} + \left(1-\frac{c}{2C}\right) \frac{\epsilon}{2C}\right)
\leq T\left(\frac{1}{2} + \frac{\epsilon}{2C}\right)<T.
\]

By assumption, this gives us a \(\delta>0\) such that whenever \(\psi(t)>\epsilon:=\frac{c}{2C}\), also \(\varphi(t)>\delta\). In particular, this applies to at least \(\epsilon\) fraction of \(t\leq T\). Hence, it follows that for any \(T\in\mathcal{T}\),
\[
\sum_{t=1}^T\varphi(t)
\geq
\delta \epsilon T.
\]
This shows that there are infinitely many \(T\) such that \(\frac{1}{T}\sum_{t=1}^T\varphi(t)>\delta \epsilon\) and thus concludes the proof.
\end{proof}



\begin{proof}[Proof of \Cref{prop:no-regret-fp}]

To begin, note that since \(\mathcal{C}\subseteq\Pset\) is closed and \(\Pset\) compact, also \(\mathcal{C}\) is compact. Hence, continuity of \(\Score(\p,f(\p))\) and \(\Score(f(\p),f(\p))\) implies that both are also bounded on \(\mathcal{C}\) and thus \(\sup_t|\Score(\Pvar'_t,f(\Pvar_t))|<\infty\) for \(\Pvar'_t\in \{\Pvar_t,f(\Pvar_t)\}\). Hence, by our assumptions, the conditions for \Cref{prop:slln} are satisfied.

``\(\Rightarrow\)''.
Assume \(\mathrm{Regret}(T)\) is sublinear. We want to show that then $\sum_{t=1}^T \Vert f(\Pvar_t)-\Pvar_t\Vert $ is sublinear. To do this, we will apply \Cref{lem:lemma-analysis}.

To begin, define \(\varphi(t):=\Score(f(\Pvar_t),f(\Pvar_t))-\Score(\Pvar_t,f(\Pvar_t))\) and note that \(\varphi(t)\geq 0\) since \(S\) is proper. By \Cref{prop:slln}, it follows that if \(\Regret(T)\) is sublinear, also \(\sum_{t=1}^T\varphi(t)\) is sublinear almost surely. For brevity, we omit the ``almost surely'' qualification in the following. 

Next, define \(\psi(t):=\Vert f(\Pvar_t)-\Pvar_t\Vert \), and note that \(0\leq \psi(t)\leq n\). Next, let \(\epsilon>0\) arbitrary. To apply \Cref{lem:lemma-analysis} to \(\varphi\) and \(\psi\), it remains to show that there exists \(\delta>0\) such that whenever \(\psi(t)\geq\epsilon\), then \(\varphi(t)\geq \delta\).

To that end, let 
\[\delta:=\min_{\{\p\in\mathcal{C}\mid \Vert\p-f(\p)\Vert\geq\epsilon\}}\Score(f(\p),f(\p))-\Score(\p,f(\p)).\]
Since \(\Vert \cdot\Vert\) is continuous and \(f\) is continuous at any \(\p\in\mathcal{C}\), the set \(\{\p\in\mathcal{C}\mid \Vert \p-f(\p)\Vert\geq \epsilon\}\) is closed and thus compact. Moreover, \(\Score(f(\p),f(\p))\) and \(\Score(\p,f(\p))\) are continuous by assumption, and thus the minimum is attained at some point \(\hat{\p}\in\mathcal{C}\). But since \(S\) is strictly proper, it follows \(\delta=\Score(f(\hat{\p}),f(\hat{\p}))-\Score(\hat{\p},f(\hat{\p}))>0.\)
Hence, since \(\Pvar_t\in \mathcal{C}\) for any \(t\in\mathbb{N},\) it follows that whenever \(\varphi(t)\geq \epsilon,\) it follows\[\varphi(t)=\Score(f(\Pvar_t),f(\Pvar_t))-\Score(\Pvar_t,f(\Pvar_t))\geq \delta.\]

This shows all conditions for \Cref{lem:lemma-analysis}. Hence, we conclude that \(\lim_{t\rightarrow\infty}\frac{1}{T}\sum_{t=1}^T\Vert f(\Pvar_t)-\Pvar_t\Vert =0\).

``\(\Leftarrow\)''.
Let \(\varphi(t):=\Vert f(\Pvar_t)-\Pvar_t\Vert \) and \(\psi:=\Score(f(\Pvar_t),f(\Pvar_t))-\Score(\Pvar_t,f(\Pvar_t))\). We assume that \(\sum_{t=1}^T\varphi(t)\) is sublinear in \(T\) and want to show that then \(\mathrm{Regret}(T)\) is sublinear as well. To do so, we will show that \(\sum_{t=1}^T\psi(t)\) is sublinear using our lemma, and then the required statement follows again from \Cref{prop:slln}.

Now we have to show the conditions of the lemma. First, as before, \(\varphi(t)\geq 0\) and \(\psi(t)\geq 0.\) Second, as noted in the beginning, we have \(\sup_{t}\psi(t)<\infty\) by our assumption that \(\Score(f(\p),f(\p))\) and \(\Score(\p,f(\p))\) are continuous on \(\mathcal{C}.\) Now let \(\epsilon>0\) arbitrary. Assume that \(\Score(f(\Pvar_t),f(\Pvar_t))-\Score(\Pvar_t,f(\Pvar_t))>\epsilon\) for some \(\epsilon>0\) and \(t\in\mathbb{N}.\)

Consider the set \(\mathcal{C}':=\{\p\in\mathcal{C}\mid \Score(f(\p),f(\p))-\Score(\p,f(\p))\geq\epsilon\}\). Since \(\Score(f(\p),f(\p))\) and \(\Score(\p,f(\p))\) are continuous on \(\mathcal{C}\) by assumption, this set is compact. Moreover, the function \(\p\in\mathcal{C}\mapsto \Vert \p-f(\p)\Vert\) is continuous since \(f\) is continuous on \(\mathcal{C}\) by assumption. Hence, the minimum
\(\delta:=\min_{\p\in \mathcal{C}'}\Vert \p-f(\p)\Vert\)
is attained at some point \(\hat{\p}\in\mathcal{C}'.\)

Now, if \(\delta=0,\) we would have \(\hat{\p}=f(\hat{\p})\) and thus
\[\Score(f(\hat{\p}),f(\hat{\p}))-\Score(\hat{\p},f(\hat{\p}))=
\Score(\hat{\p},\hat{\p})-\Score(\hat{\p},\hat{\p})
=0<\epsilon,\] which is a contradiction. Hence, \(\delta>0.\) Since \(\Pvar_t\in\mathcal{C},\) it follows from \(\Score(f(\Pvar_t),f(\Pvar_t))-\Score(\Pvar_t,f(\Pvar_t))\geq\epsilon,\) for \(t\in\mathbb{N}\) that \(\Pvar_t\in\mathcal{C}'\) and thus \(\Vert\Pvar_t-f(\Pvar_t)\Vert\geq \delta.\)
This shows the third condition for the lemma. We can thus conclude that \(\lim_{T\rightarrow\infty}\frac{1}{T}\sum_{t=1}^T\psi(t)=0.\) Using \Cref{prop:slln}, this concludes the proof.
\end{proof}

\subsubsection{Convergence to fixed points}

The next result shows that if the agent's predictions converge to some distribution $\p$, then $\p$ must be a fixed point. 

\begin{corollary}
    \label{cor:no-regret-converge}
   In addition to the assumptions from \Cref{prop:no-regret-fp}, assume that \(\Pvar_t\) converges almost surely to a limit \(\lim_{t\rightarrow\infty} \Pvar_t=\p^*\). Then almost surely \(\p^*\) is a fixed point if and only if the agent has sublinear regret.
\end{corollary}

\begin{proof}
By \Cref{prop:no-regret-fp}, almost surely the agent has sublinear regret if and only if
\[\lim_{t\rightarrow\infty}\frac{1}{T}\sum_{t=1}^T \Vert f(\mathbf{P}_t)-\mathbf{P}_t\Vert =0.\]
It remains to show that, given that the \(\Pvar_t\) converge, the latter is equivalent to convergence to a fixed point.

Since \(\mathcal{C}\) is compact and \(\Pvar_t\in\mathcal{C}\) for all \(t\in\mathbb{N},\) also \(\p^*\in\mathcal{C}.\) Hence, \(f\) is continuous at \(\p^*,\) so \[\Vert f(\p^*)-\p^*\Vert 
=\left\Vert f\left(\lim_{t\rightarrow\infty}\Pvar_t\right)-\lim_{t\rightarrow\infty}\Pvar_t\right\Vert
=\lim_{t\rightarrow\infty}\Vert f(\Pvar_t)-\Pvar_t\Vert .\]
Since this sequence converges, it is equal to its  
Cesàro mean,
\[\lim_{t\rightarrow\infty}\Vert f(\Pvar_t)-\Pvar_t\Vert
=\lim_{T\rightarrow\infty}\frac{1}{T}\sum_{t=1}^T\Vert f(\Pvar_t)-\Pvar_t\Vert.\]
Hence,
\[
\Vert f(\p^*)-\p^*\Vert
=\lim_{t\rightarrow\infty}\Vert f(\Pvar_t)-\Pvar_t\Vert
=\lim_{T\rightarrow\infty}\frac{1}{T}\sum_{t=1}^T\Vert f(\Pvar_t)-\Pvar_t\Vert
.\]
It follows that, if \(\lim_{t\rightarrow\infty}\Pvar_t=\p^*,\) then \[\Vert f(\p^*)-\p^*\Vert =0\Leftrightarrow \lim_{T\rightarrow\infty}\frac{1}{T}\sum_{t=1}^T\Vert f(\Pvar_t)-\Pvar_t\Vert=0.\]

This shows that, almost surely, \(\p^*\) is a fixed point, if and only if \(\sum_{t=1}^T\Vert f(\Pvar_t)-\Pvar_t\Vert \) is sublinear.
\end{proof}

\subsection{Prediction markets}

\label{appendix:prediction-markets}
Lastly, we consider prediction markets. We assume a simplified model of a prediction market, in which traders submit a single prediction and get scored using a proper scoring rule. The prediction that is output by the market and that influences the outcome is just a weighted average of the individual traders' predictions. In this situation, if a trader has a small weight and can thus barely influence the market prediction, the trader's score will mostly be determined by the accuracy of the report, rather than the influence of the report on the market. Thus, if all traders are small relative to the market, the equilibrium prediction will be close to a fixed point.

A similar result was shown by \citet{hardt2022performative} in the performative prediction context. They define a firm's performative power as the degree to which the firm can influence the overall outcome with their prediction. \citeauthor{hardt2022performative} show that in an equilibrium, the distance between a player's (performatively optimal) equilibrium strategy and their strategy when optimizing loss against the fixed equilibrium distribution (here, this means predicting the market probability) is bounded by the power of the trader. We give an analogous result for our formal setting and assumptions.

To formalize the setting, assume that there are $N$ players. We associate with each player $n\in [N]$ a number $w_n\in[0,1]$ s.t.\ $\sum_nw_n=1$, representing, intuitively, what fraction of the overall capital in the market is provided by player $n$. In the game, all players simultaneously submit a probability distribution $\p_n$. Then the event \(\Y\) is sampled according to the distribution $\q=f(\sum_nw_n\p_n)$. Finally, each player is scored in proportion to
$S(\p_n,\Y)$ for some strictly proper scoring rule $S$. Typical market scoring rules would consider terms like $S(\p_n,\Y)-S(\p_n,\Y)$, but subtracting $S(\p_n,\Y)$ (or multiplying by constants) does not matter for the game. We assume that players maximize their expected score, \(\E[S(\p_n,\Y)]=\Score(\p_n,f(\sum_mw_m\p_m))\).

For discussions of market scoring rules, see 
\citet{Hanson2003}
and 
\citet{Pennock2007}.
Prior work has connected these market scoring rules to more realistic prediction markets that trade Arrow--Debreu securities markets such as PredictIt 
[e.g., \citealp{Hanson2003}; 
\citealp{Pennock2007}, Section 4; \citealp{Chen2007}; \citealp{Agrawal2009}
].

We assume that $f$ is common knowledge. Moreover, in the following we only consider pure strategy equilibria, and we do not investigate the existence of equilibria. 

\begin{theorem}\label{thm:market}
    Let \(S\) be a proper scoring rule and let \(G,g\) as in the Gneiting and Raftery characterization of \(S\). Let $(\p_n)_n$ be a pure strategy Nash equilibrium of the aforedefined game and let \(\hat{\p}:=\sum_n w_n\p_n\) be the market prediction. Assume \(f\) is differentiable at \(\hat{\p}\). For any player \(n\), if \(G,g\) are differentiable at \(\p_n\) and \(Dg(\p_n)\succ \gamma_{\p_n},\) it follows that
    \begin{equation*}
        \left\Vert f\left(\hat{\p}\right)-\p_n\right\Vert\leq \frac{w_n\Vert Df\left(\hat{p}\right)\Vert_{\op}\Vert g(\p_n)\Vert}{\gamma_{\p_n}}.
    \end{equation*}
\end{theorem}

In particular, this theorem shows that players $n$ with very low $w_n$ (little capital/influence on $\q$) will accurately predict \(\q=f(\hat{\p})\). Note, however, that $\hat{\p}$ is not necessarily a fixed point or close to a fixed point. If there are are also players $n$ with very high $w_n$, then their prediction and the overall market prediction may be wrong. (So interestingly the overall market probability $\hat{\p}=\sum_nw_n\p_n$ is worse than the prediction of individuals. One might take this to suggest that anyone interested in $\q$ should look at the latter type of predictions. Of course, if this is what everyone does, it is not so clear anymore that the model $\q=f(\sum_nw_n\p_n)$ is accurate.)

\begin{proof}The proof is analogous to that of \Cref{theorem:Caspar-approx-fix-point}.
Let \((\p_n)_n\) be a pure strategy Nash equilibrium and \(\hat{\p}:=\sum_mw_m\p_m\). Each player must play a best response to the other player's strategies, so $\p_n$ must be a global maximum of the function $\varphi\colon\p_n\mapsto S(\p_n,\sum_mw_m\p_m)$
Hence, it must be \(\nabla\varphi(\p_n)^\top(f(\hat{\p})-\p_n)\leq 0\), i.e., the directional derivative of \(\varphi\) in the direction \(f(\hat{\p})-\p_n)\) must be at most zero. Otherwise, player \(n\) could improve their loss by changing their prediction marginally towards \(f(\hat{\p})\).

Computing the gradient, we have
\begin{eqnarray*}
    \nabla_{\p_n}\left(S\left(\p_n,f\left(\sum_m w_m\p_m\right)\right)\right)
    &=& \nabla_{\p_n}\left(G(\p_n)+g(\p_n)^\top\left(f\left(\sum_mw_m\p_m\right)-\p_n\right)\right)\\
    &=& g(\p_n)+Dg(\p_n)^\top\left(f\left(\sum_mw_m\p_m\right)-\p_n\right)+ w_n Df(\hat{\p})^\top g(\p_n) - \Id g(\p_n)
    \\
    &=& Dg(\p_n)^\top(f(\hat{\p})-\p_n)+w_nDf(\hat{\p})^\top g(\p_n).
\end{eqnarray*}
It follows
\begin{align}&0\geq \nabla\varphi(\p_n)^\top (f(\hat{\p})-\p_n)=
(f(\hat{\p})-\p_n)^\top Dg(\p_n) (f(\hat{\p})-\p_n)+w_ng(\p_n)^\top Df(\hat{\p}) (f(\hat{\p})-\p_n)\\
\Rightarrow&
-w_ng(\p_n)^\top Df(\hat{\p}) (f(\hat{\p})-\p_n)
\geq (f(\hat{\p})-\p_n)^T (Dg(\p_n)) (f(\hat{\p})-\p_n).
\end{align}
Using that \(Dg(\p_n)|_{\TPset}\succ \gamma_{\p_n}\) and thus \((f(\hat{\p})-\p_n)^\top (Dg(\p_n)) (f(\hat{\p})-\p_n) \geq \gamma_\p\Vert f(\hat{\p})-\p_n\Vert^2\), it follows that
\begin{eqnarray*}
    && \gamma_{\p_n}\Vert f(\hat{\p})-\p_n\Vert^2\\
    &\leq & (f(\hat{\p})-\p_n)^\top Dg(\p_n) (f(\hat{\p})-\p_n) \\
    &\leq& - w_ng (\p_n)^\top Df(\hat{\p})(f(\hat{\p})-\p_n)\\
    &\leq& w_n\vert g(\p_n)^\top Df(\hat{\p})(f(\hat{\p})-\p_n)\vert\\
    &\underset{\text{Cauchy-Schwarz}}{\leq}&w_n \Vert g(\p_n)\Vert \Vert Df(\hat{\p})(f(\hat{\p})-\p_n)\Vert\\
    &\leq& w_n\Vert g(\p_n)\Vert \Vert Df(\hat{\p})\Vert_{\mathrm{op}}\Vert f(\hat{\p})-\p_n\Vert
\end{eqnarray*}
The result follows by dividing by $\gamma_{\p_n}\Vert f(\hat{\p})-\p_n\Vert $.

\end{proof}

\begin{corollary}
In addition to the assumptions from \Cref{thm:market}, assume that $f$ is Lipschitz-continuous and \(C:=\sup_{p\in\Pset}\frac{\Vert g(\p)\Vert }{\gamma_\p}<\infty.\) Let \((\p_n)_n\) be a Nash equilibrium and let $\epsilon>0$ arbitrary. Then there exists a $\delta>0$ such that if for all $n$, $w_n<\delta,$ all of $\p_n$ and $f(\p_n)$, for all $i$, as well as $\sum_mw_m\p_m$ and $f(\sum_mw_m\p_m)$ are within $\epsilon$ of each other.
\end{corollary}
\begin{proof}
Let \(\epsilon>0\) arbitrary. Let \(L_f\) be the Lipschitz constant of \(f\) and note that then \(\Vert Df(\p)\Vert_{\op}\leq L_f\) for all \(\p\in\Pset.\) 
By \Cref{thm:market}, it follows for \(\hat{\p}:=\sum_mw_m\p_m\) and any player \(n\) that
\[\Vert f(\hat{\p})-\p_n\Vert \leq w_nL_fC.\]
Now let \(\lambda:=\min(\{1,\frac{1}{L_f}\})\) and \(\delta:=\frac{\epsilon\lambda}{4CL_f}\), and assume \(w_n<\delta\) for all \(n\in[N]\). Then it follows
\[\Vert f(\hat{\p})-\p_n\Vert \leq \delta L_fC\leq \frac{\lambda}{4}\epsilon.\]
Moreover, since \(\hat{\p}\) is a convex combination of probabilities \(\p_n\), it follows that
\[\Vert f(\hat{\p})-\hat{\p}\Vert \leq \max_{n}\Vert f(\hat{\p})-\p_n\Vert \leq \frac{\lambda}{4}\epsilon.\]
Thus, by the triangle equality, we have \(\Vert \p_n-\hat{\p}\Vert \leq \frac{2\lambda}{4}\epsilon\), and since \(f\) is Lipschitz-continuous,
\[\Vert f(\hat{p})-f(\p_n)\Vert \leq L_f\Vert \hat{\p}-\p_n\Vert \leq L_f\frac{2\lambda}{4}\epsilon \leq \frac{1}{2}\epsilon\]
for any \(n\in[N]\).

This shows that all of \(\p_n,\hat{\p},f(\p_n)\) are within \(\epsilon/2\) of \(f(\hat{\p})\) and thus by the triangle inequality within \(\epsilon\) of each other.
\end{proof}

It would be interesting to extend these results. For example, it is unclear what happens when players make predictions \textit{repeatedly}. (To keep things simple, one should probably still imagine that all players know $f$ and that the environment probability is determined by $f$ applied to the majority forecast. If the traders have private information, prediction markets become harder to analyze. For some discussions, see \cite{ostrovsky2009information}, \cite{chen2016informational}.)

\end{document}